%% file: main.tex
\pdfoutput=1
\documentclass[11pt]{article}
\usepackage{times}
\usepackage[letterpaper, left=1in, right=1in, top=1in, bottom=1in]{geometry}

\usepackage{hyperref}
\usepackage{comment}
\usepackage{color}
\usepackage{graphics}
\usepackage{bm}
\usepackage{amsmath, amsthm, amssymb, amsfonts, mathtools, graphicx, enumerate}
\usepackage{algorithm,algorithmic}
\usepackage{xspace}
\usepackage{url}
\usepackage{soul}
\usepackage{enumitem}

\usepackage{mkolar_definitions}

\newtheorem*{theorem*}{Theorem}
\newtheorem*{lemma*}{Lemma}
\newtheorem*{proposition*}{Proposition}

\newcommand{\etamin}{\eta_{\mathrm{min}}}
\newcommand{\defeq}{:=}
\newcommand{\wh}{\widehat}
\newcommand{\wt}{\widetilde}
\newcommand{\ol}{\overline}
\newcommand{\veps}{\varepsilon}

\newcommand{\dd}{\mathrm{d}}
\newcommand{\oppro}{\opp^\pi}
\newcommand{\clip}{{\mathrm{clip}}}

\newcommand{\rank}{\mathrm{rank}}

\newcommand{\numact}{K}

\newcommand{\pseudo}{\text{pseudo}}

\newcommand{\initdist}{d_0}

\newcommand{\thetaup}{\theta^{\mathrm{up}}}
\newcommand{\thetadown}{\theta^{\mathrm{down}}}

\newcommand{\dlinpi}{\wt{d}^\pi}
\newcommand{\dlin}{\wt{d}}
\newcommand{\thetalin}{\wt\theta}

\newcommand{\opp}{\mathbf{P}}
\newcommand{\opex}{\mathbf{E}}
\newcommand{\opexp}{\opex^\pi}

\newcommand{\dlr}{\mathsf{d}}
\newcommand{\dspanner}{\mathsf{d}}

\newcommand{\hatwpi}{\wh{w}^{\pi}}
\newcommand{\hatd}{\wh{d}^{D}}
\newcommand{\hatdpi}{\wh{d}^{\pi}}
\newcommand{\hatdnext}{\wh{d}^{\, D,\dagger}}
\newcommand{\dnext}{d^{D,\dagger}}

\newcommand{\cliphatdpi}[2]{\wh d^{\pi}_{#1} \wedge \Bx_{#1}#2_{#1}}
\newcommand{\clipbardpi}[2]{\ol d^{\pi}_{#1}\wedge \Bx_{#1}#2_{#1}}
\newcommand{\clipbarddpi}[2]{\ol d^{\pi}_{#1} \wedge \dspanner#2_{#1}}

\newcommand{\emle}{\veps_{\mathrm{mle}}}
\newcommand{\ereg}[1]{\veps_{\mathrm{reg,#1}}}

\newcommand{\mle}{\mathrm{mle}}
\newcommand{\reg}{\mathrm{reg}}
\newcommand{\nmle}{n_\mathrm{mle}}
\newcommand{\nreg}{n_\mathrm{reg}}

\newcommand{\ev}{\veps_{\mathrm{mle}}}
\newcommand{\evnext}{\veps_{\mathrm{mle}}}

\newcommand{\unif}{\mathrm{unif}}
\newcommand{\Piexpl}{\Pi^{\mathrm{expl}}}

\newcommand{\Ba}{C^{\mathbf{a}}}
\newcommand{\Bx}{C^{\mathbf{x}}}

\newcommand{\Wclip}{\Wcal}

\newcommand{\Pdim}{\mathrm{Pdim}}
\newcommand{\VCdim}{\mathrm{VCdim}}

\newcommand{\mutrue}{\mu^*}

\newcommand{\munorm}{B^{\mathbf{\mu}}}
\newcommand{\ccoef}{\mathrm{CC}}
\newcommand{\covered}{\mathrm{covered}}

\newcommand{\clipd}{recursively clipped occupancy\xspace}

\newcommand{\offalg}{\textsc{Forc}\xspace}
\newcommand{\onalglong}{\textbf{\textsc FORC}-guided 
\textbf{E}xploration\xspace}
\newcommand{\onalg}{\textsc{Force}\xspace}

\newcommand{\offrepralg}{\textsc{ForcRl}\xspace}
\newcommand{\onrepralglong}{\textbf{\textsc FORCRL}-guided 
\textbf{E}xploration\xspace}
\newcommand{\onrepralg}{\textsc{ForcRlE}\xspace}

\usepackage[capitalise,nameinlink]{cleveref}
\Crefname{equation}{Eq.}{Eqs.}
\Crefname{assumption}{Assumption}{Assumptions}
\Crefname{condition}{Condition}{Conditions}

\newcommand{\para}[1]{\paragraph{#1}}

\allowdisplaybreaks

\hypersetup{
  colorlinks   = true, 
  urlcolor     = blue, 
  linkcolor    = blue, 
  citecolor   = blue  
}

\newcount\Comments  
\definecolor{darkred}{rgb}{0.7,0,0}
\definecolor{darkgreen}{rgb}{0,0.5,0}
\definecolor{orange}{rgb}{0.7,0.4,0}
\definecolor{purple}{rgb}{0.8,0.0,0.8}

\usepackage{graphicx}

\usepackage{natbib}
\bibpunct{(}{)}{;}{a}{,}{,}

\title{Reinforcement Learning  in Low-Rank MDPs \\ with Density Features}
\date{}

\author{%
Audrey Huang \!\!\thanks{The two authors contributed equally to this work.} \thanks{Department of Computer Science, University of Illinois Urbana-Champaign. Email: \texttt{audreyh5@illinois.edu, jinglinc@illinois.edu, nanjiang@illinois.edu.}}
\and
Jinglin Chen \!\!\footnotemark[1]\,\,\hspace{-.08em}\footnotemark[2]
\and
Nan Jiang\footnotemark[2]
}

\begin{document}

\maketitle

\begin{abstract}
MDPs with low-rank transitions---that is, the transition matrix can be factored into the product of two matrices, left and right---is a highly representative structure that enables tractable learning. The left matrix enables expressive function approximation for value-based learning and has been studied extensively. In this work, we instead investigate sample-efficient learning with density features, i.e., the right matrix, which induce powerful models for state-occupancy distributions. This setting not only sheds light on leveraging unsupervised learning in RL, but also enables plug-in solutions for convex RL. In the offline setting, we propose an algorithm for off-policy estimation of occupancies that can handle non-exploratory data. Using this as a subroutine, we further devise an online algorithm that constructs exploratory data distributions in a level-by-level manner. As a central technical challenge, the additive error of occupancy estimation is incompatible with the multiplicative definition of data coverage. In the absence of strong assumptions like reachability, this incompatibility easily leads to exponential error blow-up, which we overcome via novel technical tools. Our results also readily extend to the representation learning setting, when the density features are unknown and must be learned from an exponentially large candidate set.
\end{abstract}

\section{Introduction}
The theory of reinforcement learning (RL) in  large state spaces has seen fast development. In the model-free regime, how to use powerful function approximation   to learn \textit{value functions} has been extensively studied in both the online and the offline settings \citep{jiang2017contextual, jin2020provably,  jin2020pessimism, xie2021bellman}, which also builds the theoretical foundations that connect RL with (discriminative) supervised learning. On the other hand, generative models for unsupervised/self-supervised learning---which define a sampling distribution explicitly or implicitly---are becoming increasingly powerful \citep{devlin2018bert, goodfellow2020generative}, yet how to leverage them to address the key challenges in RL remains under-investigated. While prior works on RL with unsupervised-learning oracles exist \citep{du2019provably, feng2020provably}, they often consider models such as block MDPs, which are more restrictive than typical model structures considered in the value-based setting such as low-rank MDPs. 

In this paper, we study model-free RL in low-rank MDPs with density features for state occupancy 
estimation. In a low-rank MDP, the transition matrix can be factored into the product of two matrices, and the left matrix is known to serve as powerful features for value-based learning 
\citep{jin2020provably}, as it can be used to approximate the Bellman backup of any function. On the other hand, the \textit{right} matrix can be used to represent the policies' state-occupancy distributions, yet how to leverage such \textit{density features} (without the knowledge of the left matrix) in offline or online RL is unknown. To this end, our main research question is: 
\begin{center}
{\it Is sample-efficient offline/online RL  with density features possible in low-rank MDPs?} 
\end{center}
We answer this question in the positive, and below is a summary of our   contributions:
\begin{enumerate}[leftmargin=*]
\item \textbf{Offline:}~ \pref{sec:offline} provides an algorithm for  off-policy occupancy estimation. 
It bears similarity to existing algorithms for estimating \textit{importance weights} \citep{hallak2017consistent,gelada2019off}, but our setting gives rise to 
a number of novel challenges. 
Most importantly, our algorithm  enjoys guarantees under  \textit{arbitrary} offline data distributions, when the standard notion of importance weights are not even well-defined. We introduce a novel notion of \textit{\clipd} and show that it can be learned in a sample-efficient manner. The \clipd always lower bounds the true occupancy, and the two notions coincide when the data is exploratory. Such a guarantee immediately enables an offline policy learning result that only requires ``single-policy concentrability'', which is comparable to the most recent advances in value-based offline RL \citep{jin2020pessimism, xie2021bellman}.
%
\item \textbf{Online:} Using the offline algorithm as a subroutine, in \pref{sec:online}, we design an \textit{online} algorithm that builds an exploratory data distribution (or ``policy cover''  \citep{du2019provably}) from scratch in a level-by-level manner. At each level, we estimate each policy's state-occupancy distribution and construct an approximate cover by choosing the \textit{barycentric spanner} of such distributions. A critical challenge here is that the additive $\ell_1$ error in occupancy estimation destroys the multiplicative coverage guarantee of the barycentric spanner, so the constructed distribution is never perfectly exploratory. Worse still, standard algorithm designs and analyses for handling such a mismatch easily lead to \emph{an exponential error blow-up}. We overcome this by a novel 
technique, where two inductive error terms are maintained and analyzed in parallel, with delicate interdependence that still allows for a polynomial error accumulation (\Cref{fig:double-chain}).  
%
\item \textbf{Representation learning:} We also extend our offline and online results to the representation learning setting \citep{agarwal2020flambe}, where the true density features are not given but must also be learned from an exponentially large candidate feature set. 
%
\item \textbf{Implications:} Our online algorithm is automatically reward-free \citep{jin2020reward,chen2022statistical} and deployment-efficient \citep{huang2022towards}. Further, since we can 
accurately estimate the occupancy distribution for all candidate policies, our results enable plug-in solutions for settings such as convex RL \citep{mutti2022challenging,zahavy2021reward}, where the objectives and/or constraints are functions over the entire state distributions (see \pref{app:convex}). 
\end{enumerate}

\section{Preliminaries}
\label{sec:prelim}

\para{Markov Decision Processes (MDPs)}
We consider a finite-horizon episodic MDP (without reward) defined as $\Mcal = (\Xcal, \Acal, P, H)$, where $\Xcal$ is the state space, $\Acal$ is the action space, $P = (P_0,\ldots,P_{H-1})$ with $P_h: \Xcal \times \Acal \rightarrow \Delta(\Xcal)$ is the transition dynamics, $H$ is the horizon, and $\initdist \in \Delta(\Xcal)$ is the known initial state distribution.\footnote{We assume the known initial state distribution for simplicity. Our results easily extend to the unknown version.} 
We assume that $\Xcal$ is a measurable space with possibly infinite number of elements and $\Acal$ is finite with  cardinality $\numact$.
Each episode is a trajectory $\tau = \rbr{x_0,a_0,x_1,\ldots, x_{H-1}, a_{H-1}, x_H}$, where $x_0\sim\initdist$, the agent takes a sequence of actions $a_0,\ldots,a_{H-1}$, and $x_{h+1} \sim P_h(\cdot\mid x_h,a_h)$. 
We use $\pi= (\pi_0,\ldots, \pi_{H-1}) \in (\Xcal\to\Delta(\Acal))^H$ to denote a (non-stationary) $H$-step Markov policy, which chooses $a_h \sim \pi_{h}(\cdot|x_h)$. (We will also omit the subscript $h$ and write $\pi(\cdot|x_h)$ when it is clear from context.) We use $\rho$ to refer to non-Markov policies that can choose $a_h$ based on the history $x_{0:h}, a_{0:h-1}$, which often arises from the probability mixture of Markov policies at the beginning of an trajectory. 
Once a policy $\pi$ is fixed, the MDP becomes an Markov chain, with $d_h^\pi(x_h)$ being its $h$-th step distribution. 
As a shorthand, we use the notation $[H]$ to denote $\cbr{0,1,\ldots, H-1}$.

\para{Low-rank MDPs}
We consider learning in a low-rank MDP, defined as: \begin{assum}[Low-rank MDP]
\label{assum:lowrank}
$\Mcal$ is a low-rank MDP with dimension $\dlr$, that is, $\forall h \in [H]$, there exist  $\phi_h^*: \Xcal \times \Acal \rightarrow \RR^{\dlr}$ and $\mutrue_h: \Xcal \rightarrow \RR^{\dlr}$ such that $\forall x_h$,$x_{h+1} \in \Xcal , a_h \in \Acal: P_h(x_{h+1}| x_h,a_h) = \langle \phi_h^*(x_h,a_h),\mu_h^*(x_{h+1})\rangle$. Further, $\int \|\mutrue_h(x)\|_1 (\dd x) \le \munorm$ and $\|\phi_h^*(\cdot)\|_\infty \le 1$.\footnote{This is w.l.o.g.~as the norm of $\phi_h^*$ can be absorbed into $\munorm$. In a natural special case of low-rank MDPs with ``simplex features'' \citep[Example 2.2]{jin2020provably}, \pref{assum:lowrank} holds with $\munorm = \dlr$. Our sample complexities only have polylogarithmic dependence on $\munorm$ which will be suppressed by $\wt{O}$.}
\end{assum}

\para{Notation} We use the convention $\frac{0}{0}=0$ when 
we define the ratio between two functions. Define 
$a\wedge b=\min(a,b)$, and we treat $\wedge$ as an operator with precedence between ``$\times /$'' and ``$+ -$''. When clear from the context,  $\{\square_h\} = \{\square_h\}_{h=0}^{H-1}$, and we refer to state ``occupancies," ``distributions," and ``densities" interchangeably. Finally, 
letter ``d'' has a few different versions (with different fonts): $\dlr$ is the low-rank dimension, $d(x)$ is a density, and $(\dd x)$ is the differential used in integration. Further, while $d^\pi_h$ and $d^D_h$ refer to true densities, $d_h$ (without superscripts) is often used for optimization variables. 

\para{Learning setups} We provide algorithms and guarantees under a number of different setups (e.g., offline vs.~online). The result that connects all pieces together is the setting of online \textit{reward-free} exploration 
with known density features $\mutrue=(\mutrue_0,\ldots,\mutrue_{H-1})$  and a policy class $\Pi \subseteq (\Xcal\to\Delta(\Acal))^H$ (\pref{sec:online}). Here, the learner must explore the MDP and  form accurate estimations of $d_h^\pi$ for all $\pi \in \Pi$ and $h\in [H]$, that is, output 
$\{\wh d^\pi_h\}_{h\in[H],\pi\in\Pi}$ 
such that with probability at least $1-\delta$, $\forall \pi \in \Pi, h\in [H]$,
$
\| \wh d_h^\pi - d_h^\pi \|_1 \le \veps,
$ 
by only collecting $\mathrm{poly}(H,\numact,\dlr,\log(|\Pi|),1/\veps,\log(1/\delta))$ trajectories. 
Two remarks are in order: 
\begin{enumerate}[leftmargin=*]

\item 
 Such a guarantee immediately leads to standard guarantees for return maximization when a reward function is specified. More concretely (with proof in \pref{app:online_rf}),
\begin{proposition} \label{prop:density2return}
Given any policy $\pi$ and reward function\footnote{We assume known and deterministic rewards, and can easily handle unknown/stochastic versions (\pref{app:reward}).} $R = \{R_h\}$ with 
$R_h:\Xcal\times\Acal\rightarrow[0,1]$, 
define expected return as $v^\pi_R := \EE_\pi[\sum_{h=0}^{H-1} R_h(x_h,a_h)]=\sum_{h=0}^{H-1} \iint d_h^\pi(x_h) R_h(x_h,a_h) \pi(a_h|x_h) (\dd x_h) (\dd a_h)$. 
Then for   $\{\wh{d}_h^\pi\}$ such that $\|\wh{d}_h^\pi - d_h^\pi\|_1 \le \veps/(2H)$ for all $\pi \in \Pi$ and $h \in [H]$, we have $v^{\wh\pi_R}_R \ge \max_{\pi \in \Pi}v^\pi_R - \veps$, where $\wh\pi_R = \argmax_{\pi \in \Pi}\wh{v}_R^\pi$, and $\wh{v}_R^{\pi}$ is the expected return calculated using $\{\wh{d}_h^\pi\}$.  
\end{proposition}
Moreover, the result can be extended to more general settings, where the optimization objective is some function of the state (and action) distribution that cannot be written as cumulative expected rewards; e.g., entropy 
as in max-entropy exploration \citep{hazan2018provably}, or $\|d_h^\pi - d_h^{\pi_E}\|_2^2$, where $\pi_E$ is an expert policy, used in imitation learning \citep{abbeel2004apprenticeship}. 
A detailed discussion is deferred to \pref{app:convex}. 

\item 
The introduction of $\Pi$ and the dependence on $K=|\Acal|$ are both necessary, since low-rank MDPs can emulate general contextual bandits where the density features $\mutrue$ become useless; see \pref{app:hard} for more details.  
\end{enumerate}

To enable such a result, a key component is to estimate $d_h^\pi$ using offline data (\pref{sec:offline}). 
Later in \pref{sec:repr}, we also generalize our results to the \textit{representation-learning} setting \citep{agarwal2020flambe,modi2021model,uehara2021representation}, where $\mutrue$ is not known but must be learned from an exponentially large candidate set.

\section{Off-policy occupancy estimation} 
\label{sec:offline}
In this section, we describe our algorithm, \offalg, which estimates the occupancy distribution $d_h^\pi$ of any given policy $\pi$ using an offline dataset. Note that this section serves both as an important building block for the online algorithm in \pref{sec:online} and a standalone offline-learning result in its own right, so we will make remarks from both perspectives. 

We start by introducing our assumption on the offline data.

\begin{assum}[Offline data]
\label{assum:data}
Consider a dataset $\Dcal_{0:H-1}=\Dcal_0\bigcup\ldots\bigcup\Dcal_{H-1}$,
where $\Dcal_h=\{(x_h^{(i)},a_h^{(i)},$ $x_{h+1}^{(i)})\}_{i=1}^n$. For any fixed $h$, we assume that tuples in $\Dcal_h$  are sampled i.i.d.~from  $\rho^{h-1}\circ \pi^D_h$, where $a_0,\ldots,a_{h-1}\sim\rho^{h-1}$ is an arbitrary $(h-1)$-step (possibly non-Markov) policy\footnote{$h$ on the superscript of a policy distinguishes identities and does not refer to the $h$-th step 
component (which is indicated by the subscript), that is, $\rho^h$ and $\rho^{h'}$ for $h'\ne h$ can be completely unrelated policies.} and $a_h\sim\pi^D_h$ is a single-step Markov policy. 
Further, $\rho_{h-1},\pi^D_h$ can be a function of $\Dcal_{0:h-1}$, and $\pi^D_h$ is known to the learner. 
\end{assum}
The dataset consists of $H$ parts, where the $h$-th part consists of $(x_h, a_h, x_{h+1})$ tuples, allowing us to reason about the transition dynamics 
at level $h$. In practice (as well as in \pref{sec:online}), such tuples will be extracted from trajectory data. 
We use $d_h^D(x_h,a_h,x_{h+1}),d_h^D(x_h),\dnext_h(x_{h+1})$ to denote the joint and the marginal distributions, respectively. 
Importantly, we do \textit{not} assume that $\dnext_h(x_{h+1}) = d_{h+1}^D(x_{h+1})$, i.e., the next-state distribution of $\Dcal_h$ and the current-state distribution of $\Dcal_{h+1}$ (which are both over $\Xcal$) may not be the same, as we will need this flexibility in \pref{sec:online}. The $H$ parts can also sequentially depend on each other, though samples within each part are i.i.d. While this setup is sufficient for \pref{sec:online} and already weaker than the fully i.i.d.~setting commonly adopted in the offline RL literature \citep{chen2019information,yin2021towards}, in \pref{app:alt} we discuss how to relax it to handle more general situations in   offline learning.

\subsection{Occupancy estimation via importance weights} \label{sec:weights}
Recall that value functions satisfy the familiar Bellman equations, allowing us to learn them by approximating Bellman operators via squared-loss regression. The occupancy distributions $\{d_h^\pi\}$ also satisfy the Bellman flow equation:  let $\oppro_h$ denote the 
Bellman flow operator, where for any given $d_{h}: \Xcal \to \RR$ and policy $\pi$, $(\oppro_{h} d_{h})(x_{h+1}) := \iint P_{h}(x_{h+1}|x_{h}, a_{h})\pi(a_{h}|x_{h}) d_{h}(x_{h}) (\dd x_{h}) (\dd a_{h})$.\footnote{In this definition, we do not require $d_h$ to be a valid distribution. Even $\pi$ is allowed to be unnormalized; see the definition of \pseudo-policy in \pref{def:pseudo_policy}.} $d_h^\pi$ can be then recursively defined via the Bellman flow equation $d_{h}^\pi=\oppro_{h-1} d_{h-1}^\pi$, with the base case $d_0^\pi = d_0$. (One difference is that value functions are defined bottom-up, whereas occupancies are defined top-down.) Furthermore, in a low-rank MDP, 
$\oppro_h d_h$ is always linear in $\mu^*_{h}$ (\pref{lem:opp_linear}), just like the image of Bellman operators for value is always in the linear span of $\phi^*_h$. 

Given the similarity, one might think that we can also approximate $\oppro_{h-1}$ by regressing directly onto the occupancies, hoping to obtain $d_{h}^\pi$ via 
\begin{align}
 \label{eq:regress_density}
\argmin_{d_h} \EE_{d^D_{h-1}}\sbr{ \rbr{ d_h(x_{h})- d_{h-1}^\pi(x_{h-1})\frac{\pi_{h-1}(a_{h-1}|x_{h-1})}{\pi^D_{h-1}(a_{h-1}|x_{h-1})} }^2 },
\end{align}
where $\frac{\pi_{h-1}(a_{h-1}|x_{h-1})}{\pi^D_{h-1}(a_{h-1}|x_{h-1})}$ is the standard importance weighting to correct the mismatch on actions 
between $\pi_{h-1}$ and data policy $\pi^D_{h-1}$. 
Unfortunately, this does not work due to the ``time-reversed'' nature of flow operators \citep{liu2018breaking}. In fact, the Bayes-optimal solution of \cref{eq:regress_density} is 
\[
d_h(x_{h}) = \frac{(\opp_{h-1}^\pi(d^D_{h-1} d_{h-1}^\pi))(x_{h})}{\dnext_{h-1}(x_{h})} \neq (\opp_{h-1}^\pi d_{h-1}^\pi)(x_{h}). 
\]
However, the fractional form of the solution indicates that we may instead aim to learn a related function---the importance weight, or density ratio \citep{hallak2017consistent}. If we use $w_{h-1}^\pi = d_{h-1}^\pi / d^D_{h-1}$ to replace $d_{h-1}^\pi$ as the regression target in \cref{eq:regress_density}, the population solution would be
\[
\frac{(\opp_{h-1}^\pi d_{h-1}^\pi)(x_{h})}{\dnext_{h-1}(x_{h})} = \frac{d_{h}^\pi(x_{h})}{\dnext_{h-1}(x_{h})} =: w_{h}^\pi(x_{h}).
\]
The occupancy can then be straightforwardly extracted from the weight via elementwise multiplication, i.e., $d_{h}^\pi = w_{h}^\pi \cdot \dnext_{h-1}$, where $\dnext_{h-1}$ can be estimated via MLE from the dataset itself. 

While this is promising, 
the approach uses importance weight $w_{h}^\pi(x_{h})$ as an intermediate variable, whose very existence and boundedness rely on the assumption that the data distribution $\dnext_{h-1}$ is exploratory and provides sufficient coverage over $d_{h}^\pi$. We next consider the scenario where such an assumption does \textit{not} hold. 
Perhaps surprisingly, although we would like to construct exploratory datasets in \pref{sec:online} and feed them into the offline algorithm, being able to handle non-exploratory data turns out to be crucial to the online setting, and also yields novel offline guarantees of independent interest.

\begin{algorithm*}[t!]
\caption{\textbf{F}itted \textbf{O}ccupancy Ite\textbf{r}ation with \textbf{C}lipping (\offalg) 
\label{alg:offline_known}}
\begin{algorithmic}[1]
\REQUIRE policy $\pi$, density feature $\mu^*$, dataset $\Dcal_{0:H-1}$, sample sizes $\nmle$ and $\nreg$, 
clipping thresholds $\{\Bx_{h}\}$ and $\{\Ba_{h}\}$. 
\STATE Initialize $\wh d_0^\pi= \initdist$.
\FOR {$h=1,\ldots,H$} 
\STATE Randomly split $\Dcal_{h-1}$ to two folds $\Dcal_{h-1}^\mle$ and $\Dcal_{h-1}^\reg$ with sizes $\nmle$ and $\nreg$, respectively. \label{line:split}
\STATE Estimate marginal data distributions $\hatd_{h-1}(x_{h-1})$ and $\hatdnext_{h-1}(x_{h})$ by MLE on dataset $\Dcal_{h-1}^\mle$: \label{line:mle_off}
\begin{align}
\label{eq:obj_mle}
   \hatd_{h-1}= \argmax_{d_{h-1} \in \Fcal_{h-1}} \frac{1}{\nmle} \sum_{i=1}^{\nmle}  \log \rbr{d_{h-1}(x_{h-1}^{(i)})} \text{  and  }
    \hatdnext_{h-1} =\argmax_{d_{h}\in\Fcal_{h} } \frac{1}{\nmle} \sum_{i=1}^{\nmle}  \log\rbr{d_{h}(x_{h}^{(i)})},
\end{align}
where 
$\Fcal_h = \cbr{d_h = \langle \mutrue_{h-1}, \theta_h \rangle : d_h \in \Delta(\Xcal), \theta_h \in \RR^{\dlr},\|\theta_h\|_\infty \le 1 }.$  \hfill {\color{blue} \# $\|\theta_h\|_\infty \le 1$ guarantees $d^D_h \in \Fcal_h$}
\STATE  
Define $\Lcal_{\Dcal_{h-1}^{\reg}}(w_{h},w_{h-1},\ol\pi_{h-1}) \defeq \frac{1}{\nreg} \sum_{i=1}^{\nreg}  \rbr{ w_{h}(x_{h}^{(i)}) - w_{h-1}(x_{h-1}^{(i)}) \frac{\ol \pi_{h-1}(a_{h-1}^{(i)}|x_{h-1}^{(i)})}{\pi^D_{h-1}(a_{h-1}^{(i)}|x_{h-1}^{(i)})} }^2$, and estimate
\label{line:reg_off}
\begin{align}
\label{eq:obj_reg}
\hatwpi_{h} = \argmin_{w_{h} \in \Wclip_{h}} \Lcal_{\Dcal_{h-1}^{\reg}}\rbr{w_{h},\tfrac{\cliphatdpi{h-1}{\hatd}}{\wh d^D_{h-1}}, \pi_{h-1} \wedge \Ba_{h-1} \pi_{h-1}^D}, 
\end{align} 
where $\label{eq:wclip}
\Wclip_{h} = \cbr{w_{h} = \frac{\langle \mutrue_{h-1}, \thetaup_{h}\rangle}{\langle \mutrue_{h-1}, \thetadown_{h}\rangle} :\nbr{w_{h}}_\infty \le  \Bx_{h-1}\Ba_{h-1},\thetaup_{h},\thetadown_{h} \in \RR^{\dlr}}. $
\STATE Set the estimate $\wh d_{h}^\pi= \hatwpi_{h} \, \hatdnext_{h-1} $. \label{line:multiply}
\ENDFOR
\ENSURE estimated state occupancies $\{\wh d_{h}^\pi\}_{h\in[H]}$.
\end{algorithmic}
\end{algorithm*}

\subsection{Handling insufficient data coverage}

Because we make no assumptions about data coverage, 
the true occupancy $d_h^\pi$ may be completely unsupported by data, in which case there is no hope to estimate it well. What kind of learning guarantees can we still obtain?

To answer this question, we introduce one of our main conceptual contributions, a  novel learning target for occupancy estimation under arbitrary data distributions.

\begin{definition}[Pseudo-policy and \clipd]
\label{def:clipd} \label{def:pseudo_policy}
Given a Markov policy $\pi$, data distributions $\{d_h^D\}$, and  state and action clipping thresholds $\{\Bx_h\}$, $\{\Ba_h\}$, the \clipd, $\{\ol d_h^\pi\}$, is defined as follows. Let 
$\ol{d}_0^\pi \defeq d_0^\pi = \initdist$. Define $\ol \pi_h(a_h|x_h) \defeq \pi_h(a_h|x_h)\wedge \Ba_{h} \pi_h^D(a_h|x_h)$ (or $\ol \pi_h = \pi_h \wedge \Ba_{h} \pi_h^D$ for short), and for $1\le h \le H-1$, inductively set \footnote{Note that $\ol d_h^{\pi}$ depends on hyperparameters $\Bx_h$ and $\Ba_h$, which is omitted in the notation. \pref{app:clipd_threshold} discusses the relationship between $\Bx_h, \Bx_a$ and the missingness error, namely, that $\|d_h^\pi - \ol{d}_h^\pi\|_1$ is Lipschitz in, and thus insensitive to misspecifications of, the clipping thresholds.
}
\begin{align}
    \ol{d}_{h}^\pi(x_{h}) \defeq \rbr{\opp^{\ol\pi}_{h-1} ~\rbr{\clipbardpi{h-1}{d^D}}}(x_h).
    \label{eq:def_dbar}
\end{align}
We also call objects like $\ol\pi$ a \emph{\pseudo-policy}, which can yield unnormalized distributions over actions. 
\end{definition}
The above definition first clips the previous-level $\ol{d}_{h-1}^\pi$ to have at most $\Bx_{h-1}$ ratio over the data distribution $d^D_{h-1}$ and the policy $\pi$ to have at most $\Ba_{h-1}$ ratio over $\pi^D_{h-1}$, then applies the Bellman flow operator. This guarantees that $\ol{d}_{h}^\pi$ is always supported on the data distribution (unlike $d_h^\pi$), and $\ol{d}_{h}^\pi \le d_h^\pi$ because poorly-supported mass is removed from every level (and hence $\ol d_h^\pi$ is generally an unnormalized distribution). 
Further, when we do have data coverage and the original importance weights on states and actions are always bounded by $\{\Bx_h\}$ and $\{\Ba_h\}$, it is easy to see that $\ol{d}_{h}^\pi = d_{h}^\pi$,  since the clipping operations will have no effects and \pref{def:clipd} simply coincides with the Bellman flow equation for $\{d_h^\pi\}$.

As we will see below in \pref{sec:offline_alg}, $\{\ol d_h^\pi\}$ becomes a learnable target and the $\ell_1$ estimation error of our algorithm goes to $0$ when the sample size $n \to \infty$. The thresholds $\{\Bx_{h}\}$ and $\{\Ba_h\}$ reflect a bias-variance trade-off: higher thresholds ensure that less ``mass'' is clipped away (i.e., $\ol d_h^\pi$ will be closer to $d_h^\pi$), but result in a worse sample complexity as the  algorithm will need to deal with larger importance weights. Below we provide more fine-grained characterization on the bias part, i.e., how $\ol d_h^\pi$ is related to $d_h^\pi$, and the proof is deferred to \pref{app:off_occu}.

\begin{proposition}[Properties of $\ol d_h^\pi$]~ 
\label{prop:clipd}
\begin{enumerate}[leftmargin=*]
\item $\ol{d}_{h}^\pi \le d_h^\pi$.
\item $\ol{d}_{h}^\pi = d_h^\pi$ when data covers $\pi$ (i.e., $\forall h' < h$ we have $d_{h'}^\pi \le \Bx_{h'} d^D_{h'}$ and $\pi_{h'} \le \Ba_{h'} \pi^D_{h'}$). 
\item $\|\ol{d}_h^\pi - d_h^\pi\|_1 \le \|\ol{d}_{h-1}^\pi - d_{h-1}^\pi\|_1 + \|  \ol{d}_{h-1}^\pi-\clipbardpi{h-1}{d^D}\|_1  + \| \opp^{\pi}_{h-1} d_{h-1}^\pi-\opp^{\ol\pi}_{h-1} d_{h-1}^\pi\|_1.$
\end{enumerate}
\end{proposition}
The 3rd claim  shows how the bias term $\|\ol d_h^\pi  - d_h^\pi\|_1$ (i.e., how much mass $\ol d_h^\pi$ is missing from $d_h^\pi$) accumulates over the horizon: the RHS of the bound consists of 3 terms, where the first is missing mass from the previous level, and the other terms correspond to 
the mass being clipped away from states and actions, respectively, at the current level. 

\subsection{Algorithm and analyses} \label{sec:offline_alg}
We are now ready to introduce our algorithm, \offalg, with its analyses and guarantees. See pseudocode in \pref{alg:offline_known}. The overall structure of the algorithm largely follows the sketch in \pref{sec:weights}: we use squared-loss regression to iteratively learn the importance weights (\pref{line:reg_off}), and convert them to densities by multiplying with the data distributions (\pref{line:multiply}) estimated via MLE (\pref{line:mle_off}). 

The major difference is that we introduce clipping in \pref{line:reg_off} (in the same way as \pref{def:clipd}) to guarantee that the regression target is always well-behaved and bounded, and below we show that this makes $\wh d_h^\pi$ a good estimation of $\ol d_h^\pi$. 
In particular, we will bound the \textit{regression error} $\|\wh d_h^\pi - \ol d_h^\pi\|_1$ as a function of sample size $n_{\reg}$. A key lemma that enables such a guarantee is the following error propagation result:

\begin{lemma}\label{lem:regression_decomposition}
    For every $h \in [H]$, the error between estimates $\wh{d}_h^\pi$ from \pref{alg:offline_known} and the clipped target $\ol{d}_h^\pi$ is decomposed recursively as
    \begin{align*}
        \left\| \wh{d}_{h}^\pi - \ol{d}_{h}^\pi \right\|_1 \le&~ \nbr{ \wh{d}_{h-1}^\pi - \ol{d}_{h-1}^\pi }_1 
        \\
        &\quad + 2\Bx_{h-1} \nbr{ \wh{d}^D_{h-1} - d^D_{h-1} }_1 \hspace{-.2em}+ \Bx_{h-1}\Ba_{h-1} \nbr{ \hatdnext_{h-1} - \dnext_{h-1} }_1
        \\
        &\quad + \sqrt{2}\nbr{\hatwpi_{h} - \opex^{\ol\pi}_{h-1} \rbr{d^D_{h-1}  \tfrac{\cliphatdpi{h-1}{\hatd}}{\wh{d}_{h-1}^D}}}_{2,\dnext_{h-1}},
    \end{align*}
where $(\opexp_h d_h) := (\oppro_{h} d_{h}) / \dnext_h$.
\end{lemma}
The proof can be found in \pref{app:off_occu}. The bound consists of 3 parts: the first line is the error at the previous level $h-1$, showing that the regression error accumulatives \textit{linearly} over the horizon. The second line captures errors due to imperfect estimation of the data distributions, since we use the estimated  $\wh{d}^D_{h-1}$ and $\hatdnext_{h-1}$, instead of the groundtruth distributions, to set up the weight regression problem and extract the density; these errors can be reduced by simply using larger $\nmle$. The last line represents the finite-sample error in regression, which is the difference between the estimated weight $\wh{w}_h^\pi$ and the Bayes-optimal predictor. We set the constraints in the hypothesis class in a way to guarantee the Bayes-optimal predictor is in the class (see the definition of $\Wcal_{h}$ below \cref{eq:obj_reg}), so the regression is realizable.

\para{Bounding the complexities of $\Fcal_h$ and $\Wcal_h$} The last challenge is in controlling the statistical complexities of the function classes used in learning, $\Fcal_h$ and $\Wcal_h$, both of which are infinite classes. For $\Fcal_h$, we construct an optimistic covering to bound its covering number \citep{chen2022unified}. For $\Wcal_h$, however, its hypothesis takes the form of ratio between linear functions, $\frac{\langle \mutrue_{h-1}, \thetaup_{h}\rangle}{\langle \mutrue_{h-1}, \thetadown_{h}\rangle}$, where standard covering arguments, which discretize $\thetaup_{h}$ and $\thetadown_{h}$, run into sensitivity issues, as $\thetadown_{h}$ is on the denominator where small perturbations can lead to large changes in the ratio. We overcome this by recalling a technique from \citet{bartlett2006sample}: we bound the pseudo-dimension of $\Wcal_h$, which is equal to the VC-dimension of the corresponding thresholding class. Then, using \citet{goldberg1993bounding}, the VC-dimension is bounded by the syntactic complexity of the classification rule, written as a Boolean formula of polynomial inequality predicates. 
The pseudo-dimension of $\Wcal_h$ further implies $\ell_1$ covering number bounds, for which \citet{dong2019sqrt,modi2021model} provide fast-rate regression guarantees.

\paragraph{Sample complexity of \offalg} We now provide the guarantee for \offalg, with its proof deferred to \pref{app:off_occu}.

\begin{theorem}[Offline $d^\pi$ estimation]
\label{thm:offline_d_known}
Fix $\delta\in(0,1)$. Suppose \pref{assum:lowrank} and \pref{assum:data} hold, and $\mu^*$ is known. 
Then, given an evaluation policy $\pi$, by setting $\nmle = \tilde{O}(\dspanner (\sum_{h\in[H]} \Bx_h \Ba_{h})^2 \log(1/\delta)/\veps^2)$ and $\nreg = \tilde{O}(\dspanner (\sum_{h\in[H]} \Bx_h \Ba_{h})^2 \log(1/\delta)/\veps^2 )$, with probability at least $1-\delta$, \offalg (\pref{alg:offline_known}) returns state occupancy estimates $\{\wh d^\pi_h\}_{h=0}^{H-1}$ satisfying 
\[
\|\wh{d}_{h}^\pi - \ol d_h^\pi\|_1 \le \veps,\forall h\in[H].
\]
The total number of episodes required by the algorithm is 
\[
\textstyle \tilde{O}\rbr{\dspanner H \rbr{\sum_{h \in [H]} \Bx_h \Ba_{h} }^2\log(1/\delta)/\veps^2}.
\]
\end{theorem}
This result can also be used to establish a guarantee for  $\|\wh d_h^\pi - d_h^\pi\|_1$, simply by decomposing $\|\wh d_h^\pi - d_h^\pi\|_1 \le \|\wh d_h^\pi - \ol d_h^\pi\|_1 + \|\ol d_h^\pi - d_h^\pi\|_1$. The regression error in the first term is controlled by \pref{thm:offline_d_known}. The second term is a \textit{one-sided missingness} error due to insufficient coverage of data, which we have characterized in \pref{prop:clipd}. Note that we split $\|\wh d_h^\pi - d_h^\pi\|_1$ into two terms using $\ol d_h^\pi$ as an intermediate quantity and analyze how their errors accumulate over the horizon separately; alternatively, one can directly try to analyze how $\|\wh d_h^\pi - d_h^\pi\|_1$ depends on $\|\wh d_{h-1}^\pi - d_{h-1}^\pi\|_1$. In general, we find the latter can yield significantly worse bounds---in fact, \textit{exponentially worse}, as will be seen in \pref{sec:online}. 

\para{Offline policy optimization} 
\pref{thm:offline_d_known} provides learning guarantees for $\ol d_h^\pi$, which is a point-wise lower bound of $d_h^\pi$. When we consider standard return maximization with a given reward function, having access to $\wh d_h^\pi \approx \ol d_h^\pi$ immediately enables \textit{pessimistic} policy evaluation \citep{jin2020pessimism,xie2021bellman}, and we are only $\veps$-suboptimal compared to the maximal value computed over covered parts of the data, i.e., with respect to $\ol d_h^\pi$. The immediate implication is that we can compete with the best policy fully covered by data (satisfying property 2 of \pref{prop:clipd}); see \pref{app:offline_rf} for the full statement and proof.  

\begin{theorem}[Offline policy optimization]
\label{thm:offline_rf}
Fix $\delta \in (0, 1)$ and suppose \pref{assum:lowrank} and \pref{assum:data} hold, and $\mutrue$ is known. Given a policy class $\Pi$, let $\{\wh d_h^\pi\}_{h\in[H],\pi \in \Pi}$ be the output of running \pref{alg:offline_known}. Then with probability 
at least $1-\delta$, for any reward function $R$ 
and policy selected as   
$ \wh\pi_R = \argmax_{\pi\in \Pi} \wh{v}_R^\pi, $
we have
\[
    v_R^{\wh\pi_R} \ge \argmax_{\pi \in \Pi} \ol{v}_R^\pi - \veps,
\]
where $v_R^\pi$ and $\wh{v}_R^\pi$ are defined in \pref{prop:density2return}, and $\ol{v}_R$ is defined similarly for $\{\ol{d}_h^\pi\}$. 
The total number of episodes required by the algorithm is
\[
\textstyle
\tilde{O}\rbr{\dspanner H^3 \rbr{\sum_{h \in [H]} \Bx_h \Ba_{h}}^2\log(|\Pi|/\delta)/\veps^2}.
\] 
\end{theorem}

\para{Computation} We remark that our policy optimization result only enjoys statistical efficiency and does not guarantee computational efficiency, as \pref{thm:offline_rf} assumes that we can enumerate over candidate policies and run \offalg for each of them; similar comments apply to our later online algorithm as well. Since the optimization variable is a policy, the most promising approach is to come up with  off-policy policy-gradient (OPPG) algorithms to approximate the objective. However, existing model-free OPPG methods all rely on value-function approximation  
\citep{nachum2019algaedice,liu2019off}, which is not available in our setting. Studying OPPG with only density(-ratio) approximation will be a pre-requisite for investigating the computational feasibility of our problem, which we leave for future work.

\section{Online policy cover construction}
\label{sec:online}
We now consider the online setting where  the learner explores the MDP to collect its own data. The hope is that we will collect exploratory datasets that provide sufficient coverage for \textit{all} policies in $\Pi$ (so that we can estimate their occupancies accurately), which is measured by the standard definition of concentrability.

\begin{definition}[Concentrability Coefficient (CC)] 
Given a policy class $\Pi$ and any distribution $d \in \Delta(\Xcal)$, the \textit{concentrability coefficient} at level $h$ relative to $d$ is 
\[
  \textstyle  \ccoef_h(d) = \inf \cbr{c \in \RR : \max_{\pi \in \Pi} \nbr{\frac{d_h^\pi}{d}}_\infty \le c }. 
\]
\end{definition}

\newcommand{\truespan}[1]{\pi^{h, #1}_*}

To achieve this goal, we first recall the following result, which shows the existence of an exploratory data distribution that satisfies the above criterion and hints at how to construct it.  
\begin{proposition}[Adapted from \citet{chen2019information}, Prop.~10]\label{prop:bary}  
Given a policy class $\Pi$ and $h$, let $\{d^{\truespan{i}}_h\}_{i=1}^\dspanner$ be the barycentric spanner (\pref{def:bary} in  \pref{app:bary}) of $\{d^\pi_h\}_{\pi\in\Pi}$. Then, 
$\ccoef_h\rbr{\frac{1}{\dspanner} \sum_{i=1}^\dspanner d^{\truespan{i}}_h} \le \dspanner$. 
\end{proposition} 
\pref{prop:bary} shows that for each level $h$, an exploratory distribution that has $\dlr$ concentrability always exists. It is simply the mixture of $\{d_h^{\truespan{i}}\}$ for $i\in [\dlr]$, which can be identified if we have access to $d_h^{\pi}$ for all $\pi\in\Pi$. Of course, we can only estimate $d_h^\pi$ if we have exploratory data, so the estimation of $d_h^\pi$ and the identification of $\{\truespan{i}\}$ need to be interleaved to overcome this ``chicken-and-egg'' problem \citep{agarwal2020flambe,modi2021model}: suppose we have already constructed policy cover at 
$h-1$. We can construct it for the next level as follows:
\begin{enumerate}[leftmargin=*]
\item Collect a dataset $\Dcal_{h-1}$ by rolling in to level $h-1$ with the policy cover, with $\ccoef_{h-1}(d^D_{h-1}) \le \dspanner$, then taking a uniformly random action, thereby $\ccoef_h(\dnext_{h-1}) \le \dspanner K$. 
\item Use \offalg to estimate $d_{h}^\pi$ for all $\pi \in \Pi$ based on $\Dcal_{h-1}$. 
\item Choose their barycentric spanner as the policy cover for level $h$, with $\ccoef_h(d^D_h) \le \dlr$. 
\end{enumerate}
The idea is that, since we have an exploratory distribution at level $h-1$, taking a uniform action afterwards will give us an exploratory distribution at level $h$, though the degree of exploration will be diluted by a factor of $K$. We collect data from this distribution to estimate $d_{h}^\pi$ and compute the barycentric spanner for level $h$, which will bring the concentrability coefficient back to $\dlr$, so that the process can repeat inductively.

\begin{algorithm*}[t!]
\caption{\onalglong (\onalg) \label{alg:online_known}}
\begin{algorithmic}[1]
\REQUIRE policy class $\Pi$, density feature $\mu^*$, $n = \nmle + \nreg$.
\STATE Initialize $\wh{d}_0^\pi=\initdist$ and $\wt{d}_0^\pi = \initdist,\forall \pi\in\Pi$. 
\FOR{$h=1,\ldots,H$} 
\STATE Construct $\{\dlin_{h-1}^{\pi^{h-1,i}}\}_{i=1}^{\dlr}$ as the barycentric spanner of 
$\{\dlinpi_{h-1}\}_{\pi \in \Pi}$, and set $\Piexpl_{h-1} = \{\pi^{h-1,i}\}_{i=1}^{\dlr}$.
\label{line:piexpl}
\STATE Draw a tuple dataset $\Dcal_{h-1} = \{(x_{h-1}^{(i)}, a_{h-1}^{(i)}, x_{h}^{(i)})\}_{i=1}^{n}$ using $\unif(\Piexpl_{h-1}) \circ \unif(\Acal)$. \label{line:reg_on_start}
\FOR{$\pi\in\Pi$}
\STATE Estimate $\wh{d}_{h}^\pi$ using the $h$-level loop\footnotemark ~of \pref{alg:offline_known} (lines \ref{line:mle_off}-\ref{line:multiply}) with $\Dcal_{h-1}$, $\wh d_{h-1}^\pi$, $\Bx_{h-1} = \dlr$, $\Ba_{h-1} = K$. 
\STATE Find the closest linear approximation $\dlinpi_{h} = \langle \mutrue_{h-1}, \thetalin_{h} \rangle$ where $\thetalin_{h} = \argmin_{\theta_{h} \in \RR^{\dlr}} \|\langle \mutrue_{h-1}, \theta_{h} \rangle - \hatdpi_{h}\|_1$. \label{line:line_approx}
\ENDFOR \label{line:reg_on_end}
\ENDFOR
\ENSURE estimated state occupancy measure $\{\wh d_{h}^\pi\}_{h\in[H],\pi \in \Pi}$.
\end{algorithmic}
\end{algorithm*}
\footnotetext{MLE only needs to be done once and not for every $\pi\in\Pi$.} 

The above reasoning makes an idealized assumption that $d_{h}^\pi$ can be estimated perfectly. In such a case, the constructed distribution will provide perfect coverage, so that the clipping introduced in \pref{sec:offline} becomes completely unnecessary: all clipping operations would be inactive (by setting $\Bx_h = \dlr$ and $\Ba_h = K$), and   $\ol d_h^\pi \equiv d_h^\pi$. Unfortunately, when the estimation error of $d_h^\pi$ is taken into consideration, the reasoning breaks down seriously. 

The first problem is that our estimate $\wh d_{h}^\pi$ from \offalg is not necessarily linear due to its product form. However, that is not a concern as we can linearize it 
(corresponding to \pref{line:line_approx} in \pref{alg:online_known});
we also have an alternative procedure for \offalg that directly produces linear $\wh d_{h}^\pi$ (see \pref{app:alg_anal}), so in this section we will ignore this issue and pretend that $\wh d_{h}^\pi$ is linear (thus is the same as $\wt d_{h}^\pi$ in \pref{alg:online_known}) for ease of presentation.

\subsection{Taming error exponentiation}
Now that the issue of (non-)linear $\wh d_h^\pi$ is out of the way, we are ready to see where the real trouble is: note that the barycentric spanner computed from $\{\wh d_{h}^\pi\}_{\pi\in\Pi}$ satisfies 
\begin{align}
\label{eq:approx_bary}
\nbr{\frac{\wh d_h^\pi} {\frac{1}{\dspanner} \sum_{i=1}^\dspanner \wh d_h^{\pi^{h,i}}}}_\infty \le \dspanner, \quad \forall \pi \in \Pi .
\end{align}
However, the actual distribution induced by the policy cover $\{\pi^{h,i}\}_{i=1}^{\dlr}$ is $d_h^D = \frac{1}{\dspanner} \sum_{i=1}^\dspanner d_h^{\pi^{h,i}}$. Suppose for now we have $\nmle= \infty$ for perfect estimation of $d^D_h$; even then, the regression target in \cref{eq:obj_reg} will no longer be bounded without clipping, as the boundedness of $\wh d/ \wh d$ does not imply that of $\wh d / d$, and the latter can be very large or even infinite.

While the unbounded regression target can be easily controlled by clipping, analyzing the algorithm and bounding its error still prove to be very challenging. A natural strategy is to inductively bound $\|\wh d_h^\pi - d_h^\pi\|_1$ using $\|\wh d_{h-1}^\pi - d_{h-1}^\pi\|_1$. Unfortunately, this approach fails miserably, as directly analyzing $\|\wh d_h^\pi - d_h^\pi\|_1$ yields 
\begin{align} \label{eq:exp_error}
\|\wh d_h^\pi - d_h^\pi\|_1 \le (1+ \dspanner) \|\wh d_{h-1}^\pi - d_{h-1}^\pi\|_1 + \cdots,
\end{align}
implying an $O(\dlr)^H$ exponential error blow-up. (The concrete reason for this failure will be made clear shortly.) In \pref{app:no_mle}, we also discuss an alternative approach that ``pretends'' data to be perfectly exploratory, which only addresses the problem superficially and still suffers $O(\dlr)^H$ error exponentiation, just in a different way. 
Issues that bear high-level similarities are commonly encountered in level-by-level exploration algorithms, which often demand the so-called reachability assumption \citep[Definition 2.1]{du2019provably}, which we do not need.

As all the earlier hints allude to, the key to breaking error exponentiation is to 
split the error using $\ol{d}_h^\pi$ into its two sources with very different natures: a ``two-sided'' regression error $\|\wh{d}_h^\pi - \ol{d}_h^\pi \|_1$, and a ``one-sided'' missingness error $\|\ol{d}_h^\pi - d_h^\pi\|_1$ (in the sense that $\ol{d}_h^\pi \le d_h^\pi$). Because the offline occupancy estimation module of \pref{alg:online_known} is the same as that of \pref{alg:offline_known}, \pref{lem:regression_decomposition} still holds (left $\times 1$ chain of \pref{fig:double-chain}), implying that $\|\wh{d}_h^\pi - \ol{d}_h^\pi \|_1$ can be bounded \textit{irrespective of the data distribution}. 

This observation disentangles the regression error from the rest of the analysis, allowing us to focus on bounding the missingness error. For the latter, \pref{prop:clipd} also exhibits  linear error propagation, as it takes the form of $A_h \le A_{h-1} + B_{h-1}$ where $A_h = \|\ol{d}_h^\pi - d_h^\pi\|_1$. However, it still remains to show that the additional error  (``$B_{h-1}$'') has no dependence on the inductive error (``$A_{h-1}$''), otherwise we would still have error exponentiation.\footnote{For example, if $B_{h-1}$ can only be bounded as $B_{h-1} \le A_{h-1}$, we would still have $A_h \le 2 A_{h-1}$.} This is shown in the following key lemma:
\begin{lemma}\label{lem:missingness_decomposition}
    For any $h \in [H]$ and $\pi \in \Pi$ in \pref{alg:online_known}, 
    \[
        \|\ol{d}_{h}^\pi - d_h^\pi\|_1 \le \|\ol{d}_{h-1}^\pi - d_{h-1}^\pi \|_1  + 4\dspanner\max_{\pi' \in \Pi} \|\wh{d}_{h-1}^{\pi'} - \ol{d}^{\pi'}_{h-1}\|_1.
    \]
\end{lemma}
To understand this lemma, recall that the additional error in \pref{prop:clipd} characterizes the mass  clipped away at the current level. This mass can be bounded by the regression error of the previous level ($\max_{\pi' \in \Pi} \|\wh{d}_{h-1}^{\pi'} - \ol{d}^{\pi'}_{h-1}\|_1$): intuitively, had we had perfect estimation of $\wh d_{h-1}^{\pi'} = \ol d_{h-1}^{\pi'}$, our barycentric spanner would also be perfect and we would not need any clipping at all in level $h$, implying $0$ additional error in the bound. More generally, the closer $\wh d_{h-1}^{\pi'}$ is to $\ol d_{h-1}^{\pi'}$, the less mass we need to clip away.

\begin{figure}[t!]
\centering
\includegraphics[scale=0.3]{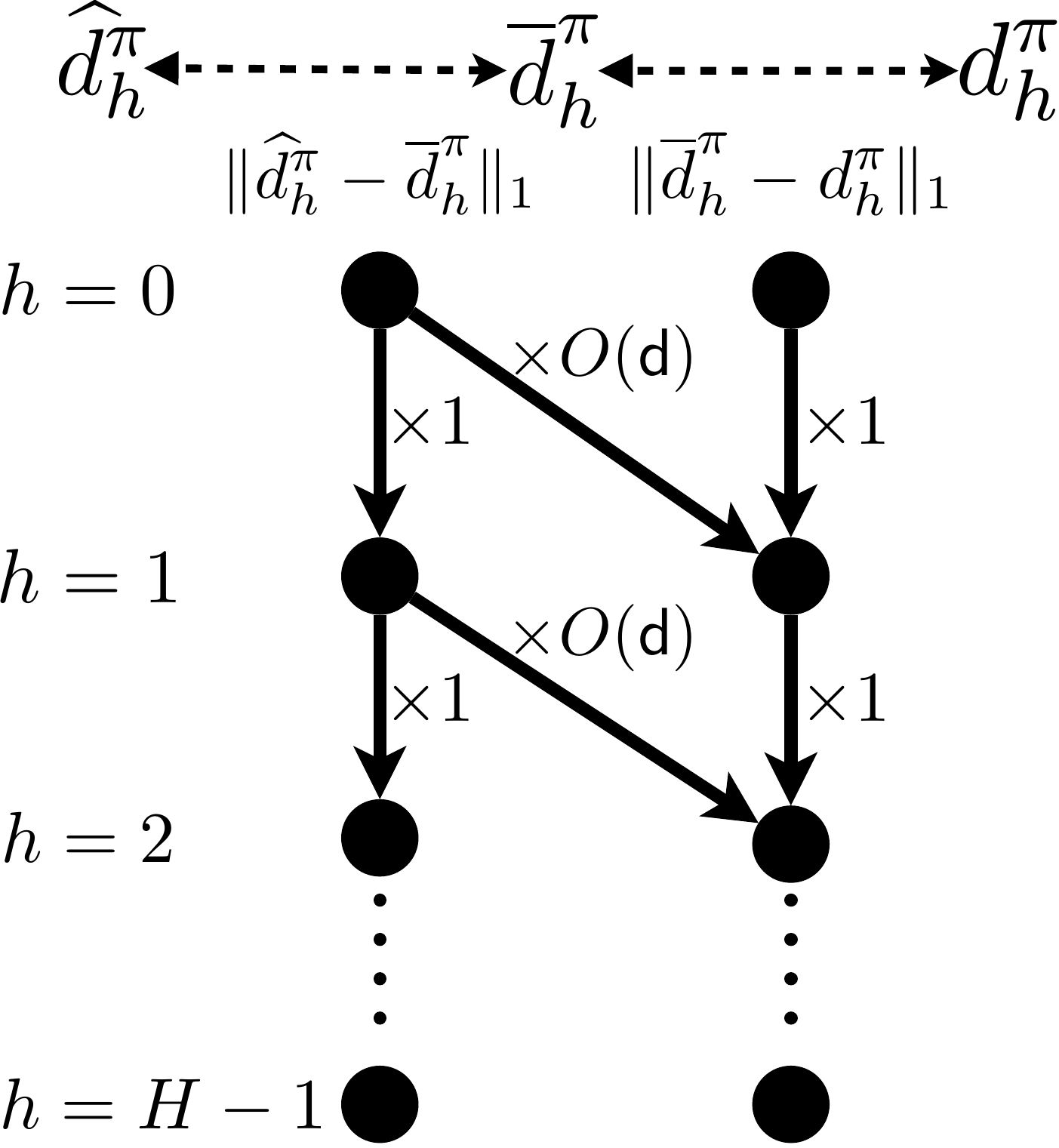}
\caption{ Error propagation diagram for \onalg. ``$\bullet \to \bullet$'' with $\times c$ means $(\bullet) \le c \times (\bullet)$ + (other instantaneous errors that do not accumulate over horizon), and multiple incoming arrows imply sum of errors. The left $\times 1$ chain is from \pref{lem:regression_decomposition}, the right $\times 1$ chain from \pref{prop:clipd}, and the $\times O(\dlr)$ edges from \pref{lem:missingness_decomposition}.
\label{fig:double-chain}}
\end{figure}

That said, this term is not instantaneous and depends inductively on quantities in the previous time step, still raising concerns of error exponentiation. To see why this is not a problem, we visualize error propagation in \pref{fig:double-chain}: it can be clearly seen that such a dependence corresponds to a ``cross-edge'', and appears at most once along any long chain. This also explains the destined failure of directly analyzing $\|\wh d_h^\pi - d_h^\pi\|_1$ in \cref{eq:exp_error}, as that corresponds to merging the two chains into one, where every edge along the only chain acquires an $O(\dspanner)$ multiplicative factor.

With this, we can now state the formal guarantee for our algorithm, \onalg. See \pref{alg:online_known} for its pseudo-code, and the proof of the guarantee is deferred to \pref{app:online_occu}.

\label{sec:online_res}
\begin{theorem}[Online $d^\pi$ estimation]
\label{thm:online_d}
Fix $\delta\in(0,1)$ and consider an MDP $\Mcal$ that satisfies \pref{assum:lowrank}, and $\mu^*$ is known. Then by setting $\nmle = \wt O\rbr{\dlr^3 K^2 H^4\log(1/\delta)/\veps^2},$$\nreg=\wt O\rbr{\dlr^{5} K^2 H^4\log(|\Pi|/\delta)/\veps^2},n=\nmle + \nreg,$ with probability at least $1-\delta$, \onalg returns state occupancy estimates $\{\wh d^\pi_h\}_{h\in[H],\pi\in\Pi}$ satisfying that 
\begin{align*}
    \nbr{\wh{d}_{h}^\pi - d_h^\pi}_1 \le \veps, \forall h\in[H],\pi\in\Pi.
\end{align*}
The total number of episodes required by the algorithm is 
\[
\wt O(nH)=\wt O\rbr{\dlr^{5} K^2 H^5\log(|\Pi|/\delta)/\veps^2}.
\]
\end{theorem}

\pref{thm:online_d} also immediately translates to a policy optimization guarantee when combined with \pref{prop:density2return}:
\begin{theorem}[Online policy optimization]
\label{thm:online_rf}
Fix $\delta \in (0, 1)$ and suppose \pref{assum:lowrank} and \pref{assum:data} hold, and $\mutrue$ is known. Given a policy class $\Pi$, let $\{\wh d_h^\pi\}_{h\in[H],\pi \in \Pi}$ be the output of running \onalg. Then with probability 
at least $1-\delta$, for any reward function $R$ 
and policy selected as   
$ \wh\pi_R = \argmax_{\pi\in \Pi} \wh{v}_R^\pi, $
we have
\[
    v_R^{\wh\pi_R} \ge \argmax_{\pi \in \Pi} v_R^\pi - \veps,
\]
where $v_R^\pi$ and $\wh{v}_R^\pi$ are defined in \pref{prop:density2return}. 
The total number of episodes required by the algorithm is
\[
\textstyle
\tilde{O}\rbr{\dspanner^5 K^2 H^7 \log(|\Pi|/\delta)/\veps^2}.
\] 
\end{theorem}
The proof is deferred to \pref{app:online_rf}. 
We remark that \pref{thm:online_rf} is a \emph{reward-free} learning guarantee \citep{jin2020reward, chen2022statistical}, and it is easy to see that \pref{alg:online_known} is deployment efficient \citep{huang2022towards}. 

\section{Representation learning}
\label{sec:repr}
In this section, we extend the offline (\pref{sec:offline}) and online (\pref{sec:online}) results to the representation learning setting. Here, the true density feature $\mutrue$ is unknown, 
but the learner has access to a realizable density feature class $\Upsilon$, defined formally below. 
For simplicity, we consider finite and normalized $\Upsilon$, as is standard in the literature \citep{agarwal2020flambe,modi2021model,uehara2021representation}. 
\begin{assum}\label{assum:realizability} 
    We have a finite density feature class $\Upsilon = \bigcup_{h \in [H]} \Upsilon_{h}$ such that $\mutrue_h \in \Upsilon_h$ for each $h \in [H]$, thus $\mutrue \in \Upsilon$. Further, for any $\mu_h \in \Upsilon_h$, we have $\int \|\mu_h(x)\|_1 (\dd x) \le \munorm$.
\end{assum}

The algorithms and analyses for the representation learning case mostly follow the same template as the known feature case, so we restrict our discussion to their differences. 
Recall that, in order to have realizable function classes for regression and MLE  in \pref{sec:offline}, we constructed $\Fcal_{h},\Wclip_{h}$ using functions linear in the known $\mutrue_{h-1}$. In order to maintain this realizability when $\mutrue_{h-1}$ is unknown, we instead construct $\Fcal_{h},\Wclip_{h}$ using the union of all functions linear in some candidate $\mu_{h-1} \in \Upsilon_{h-1}$, i.e., 
$\bigcup_{\mu_{h-1} \in \Upsilon_{h-1}} \{\langle \mu_{h-1}, \theta_h \rangle, \theta_h \in \RR^\dspanner \}$ (see \cref{eq:F_unknown} and \cref{eq:wclip_unknown} for their formal definitions). 

While such union classes allow most of \pref{sec:offline} and \pref{sec:online} to straightforwardly extend to the representation learning setting, a nontrivial modification must be made to the online algorithm.  
Recall in \pref{line:line_approx} of \pref{alg:online_known}, we constructed our policy cover using the barycentric spanner of $\{\wt d_h^\pi\}_{\pi \in \Pi}$, the set of linearized approximations to the density estimates. Importantly, this guaranteed a concentrability coefficient of $\dspanner$ because all $\wt d_h^\pi$ are linear in the same feature $\mutrue_{h-1}$. This is no longer the case with unknown features because, if linearized in the same way (but over all feasible $\mu_{h-1} \in \Upsilon_{h-1}$), each $\wt d_{h}^\pi$ can be composed of a different $\mu_{h-1}$ feature, resulting in a CC linear in $|\Pi|$. To overcome this issue, we replace  \pref{line:line_approx}  with the following ``joint linearization'' step (see \pref{line:feasle_unknown} in \pref{alg:online_unknown}):
\begin{align*}
\wh \mu_{h-1} = \min_{\mu_{h-1}\in\Upsilon_{h-1}}\max_{\pi\in\Pi}\min_{\theta_{h}\in\RR^{\dlr}} \|\langle \mu_{h-1}, \theta_{h} \rangle - \hatdpi_{h}\|_1,
\end{align*}
where all density estimates are linearized using a single feature $\wh \mu_{h-1}$, whose linear span approximates all $\wh{d}_{h}^\pi$ well. 
We provide theorems for offline/online $d^\pi$ estimation with representation learning below.

\begin{theorem}[Offline $d^\pi$ estimation with representation learning]
\label{thm:offline_d_unknown}
Fix $\delta\in(0,1)$. Suppose \pref{assum:lowrank}, \pref{assum:data}, and \pref{assum:realizability} hold. Then, given an evaluation policy $\pi$, by setting \\ 
$\nmle = \tilde{O}(\dspanner (\sum_{h\in[H]} \Bx_h \Ba_{h})^2 \log(|\Upsilon|/\delta)/\veps^2) \text{ and } \nreg = \tilde{O}( \dspanner 
 (\sum_{h\in[H]} \Bx_h \Ba_{h} )^2 \log(|\Upsilon|/\delta)/\veps^2 ),$ 
with probability at least $1-\delta$, \offrepralg (\pref{alg:offline_unknown}) returns state occupancy estimates $\{\wh d^\pi_h\}_{h=0}^{H-1}$ satisfying that 
\[\nbr{\wh{d}_{h}^\pi - \ol d_h^\pi}_1 \le \veps, \forall h\in[H].\] 
The total number of episodes required by the algorithm is 
\[ \textstyle
\tilde{O}\rbr{\dspanner H \rbr{\sum_{h \in [H]} \Bx_h \Ba_{h} }^2\log(|\Upsilon|/\delta)/\veps^2}.
\]
\end{theorem}

\begin{theorem}[Online $d^\pi$ estimation with representation learning]
\label{thm:online_d_unknown}
Fix $\delta\in(0,1)$ and suppose \pref{assum:lowrank} and \pref{assum:realizability} hold. Then by setting 
$\nmle = \wt O(\dlr^3 K^2 H^4\log(|\Upsilon|/\delta)/\veps^2), \nreg=\wt O(\dlr^{5} K^2 H^4\log(|\Pi||\Upsilon|/\delta)/\veps^2),$ $n=\nmle + \nreg,$
with probability at least $1-\delta$, \onrepralg (\pref{alg:online_unknown}) returns state occupancy estimates $\{\wh d^\pi_h\}_{h=0}^{H-1}$ satisfying that
\begin{align*}
    \|\wh{d}_{h}^\pi - d_h^\pi\|_1 \le \veps, \forall h\in[H],\pi\in\Pi.
\end{align*}
The total number of episodes required by the algorithm is 
\[ \textstyle
\wt O\rbr{\dlr^{5} K^2 H^5\log(|\Pi||\Upsilon|/\delta)/\veps^2}.
\]
\end{theorem}


The detailed proofs of these two theorems are given in \pref{app:repr}. We also present the theorems and proofs for offline/online policy optimization with representation learning as well as the formal representation learning algorithms in \pref{app:repr}.

\section{Conclusion} 
We have shown how to leverage density features for statistically efficient state occupancy estimation and reward-free exploration in low-rank MDPs, culminating in policy optimization guarantees. 
An important open problem lies in investigating the computational efficiency of our algorithms (e.g., through off-policy policy gradient).

\section*{Acknowledgements}
The authors thank Akshay Krishnamurthy and Dylan Foster for discussions related to MLE generalization error bounds. NJ acknowledges  funding support from NSF IIS-2112471 and NSF CAREER IIS-2141781.

\bibliography{refs}

\input{appendix.tex}

\end{document}

%% file: appendix.tex
\newpage
\appendix
\onecolumn

\section{Related works}
\label{app:related}
In this section, we discuss a few lines of related work in detail. 

First, the closest related works involve RL with unsupervised-learning oracles \citep{du2019provably,feng2020provably}. Instead of investigating low-rank MDPs, they consider more restricted block MDPs and need stronger assumptions such as reachability, identifiability, and separatability (we refer the reader to their works for the definitions). Their notion of ``decoder'' looks like density features in low-rank MDPs, but they are incomparable. The crucial property of ``decoder'' is that it is a map from the $\Xcal$ space to the low $\dlr$ dimensional space. This map itself no longer exists in low-rank MDPs. In addition, the density feature serves a different purpose in our paper, as its primary purpose is for constructing the weight function class.

A second line of related work is model-based representation learning in low-rank MDPs \citep{agarwal2020flambe,uehara2021representation, ren2022spectral}, which assumes that both a realizable left feature class $\Phi\ni\phi^*$ and realizable density (right) feature class $\Upsilon \ni \mutrue $ are given to the learner, essentially inducing a realizable dynamics model class. The learned model (features) are subsequently used for downstream planning. In comparison, we utilize a much weaker inductive bias as we only require a realizable density feature class $\Upsilon$, and we do not try to learn a dynamics model. Though we additionally need a policy class $\Pi$, this is a very basic and natural function class to include. It can be immediately obtained from the (Q-)value function class in the value-based approach, and from the dynamics model class (given a reward function) in the model-based approach above. In terms of the algorithm design, we also use MLE, but for a different objective (the data distribution, instead of the dynamics model).

The importance weight (density-ratio) learning used within our algorithms is related to the marginalized importance sampling of the offline RL algorithms in \citet{nachum2019dualdice,lee2021optidice,uehara2021finite,zhan2022offline,chen2022offline,huang2022beyond,ozdaglar2022revisiting}. 
These works do not make the low-rank MDP assumption and study the problem in general MDPs, and require both a weight function class and value function class for learning. 
We leverage the true density $\mutrue$ or density feature class $\Upsilon$ to construct the realizable weight function class, allowing us to achieve statistically faster rates in the low-rank MDP setting. We do not need a value function class and instead only need a weaker (as discussed in the previous paragraph) policy class $\Pi$. Lastly, we note that the aforementioned works all learn weights, while our goal is to learn the densities. Extracting the densities from the weights allows us to efficiently explore the MDP using its low-dimensional structure, and additionally enables our return maximization guarantees of \pref{prop:density2return} by separating them from the underlying data distribution.   

\section{Hardness result without the policy class}
\label{app:hard}
In this section, we show that without policy class $\Pi$, learning in low-rank MDPs (or an easier simplex feature setting) is provably hard even when the true density feature $\mu^*$ is known to the learner. The crux is that low-rank MDPs can readily emulate a fully general contextual bandit problem, where $\mu^*$ is useless. For the hardness result, we adapt Theorem 2 of \citet{dann2015sample} to our case by 
only keeping their second to third level to get a contextual bandit problem. 

To provide specifics for the reward and transition functions, we first note that the subscript of the reward/transition function denotes which level it applies to (e.g., $P_0$ are the transitions to $x_1$ from $x_0$). 
Level $h = 0$ is composed of $|\Xcal|-3$ states with zero reward, i.e., $x_0 \in \{1,\ldots,|\Xcal|-3\}$ and $R_0(i) = 0, \forall i \in \{1,\ldots,|\Xcal|-3\}$. Level $h = 1$ is composed of 2 states, i.e., $x_1 \in \{+, -\}$, where $R_1(+) = 1$ and $R_1(-) = 0$. Lastly, at level $h = 2$ we have a single null absorbing state $x_2$. 

For the transition functions, in level $h = 0$ the transitions $P_0$ are Bernoulli distributions where for any state $i \in \{1,\ldots,|\Xcal|-3\}$ and action $a_0 \in \Acal$, we have $P_0(+|i,a_0) = \frac{1}{2} + \veps'_i(a_0)$ and $P_0(-|i,a_0) = \frac{1}{2} - \veps'_i(a_0)$. Here, $\veps'_i$ is defined in a per-state manner given a parameter $\veps$. We have $\veps'_i(a_0) = \veps/2$ if $a_0 = a_0^*$, where $a_0^*$ is a fixed action; $\veps'_i(a_0) = \veps$ if $a_0 = a_0^{i,*}$ where $a_0^{i,*}$ is an unknown action defined per state $i$; and $\veps'_i(a_0) = 0$ otherwise. In level $h=1$, the transitions $P_1$ simply transmit deterministically to the absorbing state $x_2$, i.e., $P_1(x_2 | x_1, a_1) = 1$ for all $x_1 \in \{+, -\}$ and $a_1 \in \Acal$.  

It is easy to see that the dynamics of this contextual bandit can be modeled using simplex features, thus it is an instantiation of low-rank MDPs. Since we only have two levels ($H=2$), we only need to verify that $P_0$ and $P_1$ can be written in the desired form (\pref{assum:lowrank}). In level $h = 0$, we add two latent states corresponding to the rewarding and non-rewarding state, thus $\dspanner = 2$. Then in level $h = 0$, we have right features $\mu^*_0(+) = [1, 0]$ and $\mu^*_0(-) = [0, 1]$, and left features $\phi^*_0(x_0,a_0) = [P_1(+|x_0, a_0), P_1(-|x_0, a_0)]$ for any $(x_0, a_0)$, corresponding to the original Bernoulli distribution. It is easy to see that this satisfies \pref{assum:lowrank}, i.e., for any $(x_0, a_0, x_1)$ we have 
$P_0(x_1|x_0,a_0) = \langle \phi^*_0(x_0,a_0), \mu^*_0(x_1) \rangle$. In level $h = 1$ we can simply set a single latent state representing the singleton $x_2$, and observe that \pref{assum:lowrank} is trivially satisfied with $\mu^*_1(x_2) = 1$, and $\phi^*_1(x_1,a_1) = 1$ for any $(x_1, a_1)$. 

Finally, from Theorem 2 of \citet{dann2015sample}, we know that the sample complexity of learning in this contextual bandit problem is $\Omega(|\Xcal|)$, demonstrating that efficient learning is impossible in low-rank MDPs (or the simplex feature setting) given only $\mutrue$.

\paragraph{The necessity of $K=|\Acal|$ dependence} It is well known that learning contextual bandits with just a policy class requires a dependence on $|\Acal|$ in regret and sample complexity; see \citet{agarwal2014taming} and the references therein. This can also be reproduced in the above hardness result: first, we can scale up the construction by adding more actions, and show an $\Omega(|\Xcal|K)$ lower bound. Second, we now provide the learner with a policy class that contains all Markov deterministic policies. The size of the class is $O(K^{|\Xcal|})$, and the log-size is $O(|\Xcal|\log (K))$. Given the logarithmic dependence on $K$, no polynomial dependence on $\log(|\Pi|)$ can explain away the linear-in-$K$ dependence in the lower bound, and we must introduce $K$ as a separate factor in the sample complexity.

\section{RL with objectives on state distributions} \label{app:convex}
\pref{prop:density2return} also extends to general optimization objectives $f(\{d_h\})$ that are Lipschitz in the input $\{d_h\}$ (note the Lipschitz property does not require the input to be a valid distribution).  This Lipschitzness property is key for many recent results in convex RL \citep{zahavy2021reward,mutti2022challenging}, and also holds for return maximization where $f(\{d_h^\pi\}) = v_R^\pi$, in which case the Lipschitz constant is related to the maximum reward $\max_{h,x,a} R_h(x,a)$. While we write the objective $f(\{d_h\})$ using state densities $d_h(x_h)$ as input for simplicity, it is straightforward to instead use state-action densities $d_h(x_h)\pi(a_h|x_h)$ formed by directly composing the state density $d_h$ with the policy $\pi$. If $f$ is Lipschitz in state-action densities, it will still be Lipschitz in the state-action densities in the $\ell_1$ norm, which is the exactly the case in return maximization, since any input density will be composed with same $\pi$. Lastly, we note that constraints can also be added to the objective and to result in a similar statement. 

\begin{proposition} \label{prop:lipschitz_density2return}
Suppose the optimization objective is $f(\{d_h\})$, where $f$ is Lipschitz in $\{d_h\}$ under the $\ell_1$ norm, i.e., there exists a constant $L > 0$ such that for any $\{d_h'\}$ and $\{d_h''\}$
\[
    \abr{f(\{d_h'\}) - f(\{d_h''\})} \le L \sum_{h \in [H]} \|d_h' - d_h''\|_1. 
\]
Then for $\{\wh{d}_h^\pi\}$ such that $\|\wh{d}_h^\pi - d_h^\pi\|_1 \le \frac{\veps}{2H}$ for all $\pi \in \Pi$ and $h \in [H]$, and $\wh\pi$ maximizing the plug-in estimate of the objective: 
\[
    \wh\pi = \argmax_{\pi \in \Pi}f(\{\wh{d}_h^\pi\}),
\]
we have 
\[
    f(\{d_h^{\wh\pi}\}) \ge \max_{\pi \in \Pi}f(\{d_h^{\pi}\}) - L\veps.
\]
\end{proposition}
\begin{proof}
    For any $\pi \in \Pi$, from the Lipschitz assumption, 
    \begin{align*}
        \abr{f(\{d^\pi_h\}) - f(\{\wh{d}^\pi_h\})} \le L \sum_{h \in [H]} \|d^\pi_h - \wh{d}_h^\pi\|_1 \le L\veps/2.
    \end{align*}
    Then, letting $\pi^* = \argmax_{\pi \in \Pi}f(\{d_h^{\pi}\})$ denote the maximizer of the true objective and using the above inequality, 
    \begin{align*}
        f(\{d_h^{\wh\pi}\}) - f(\{d_h^{\pi^*}\}) = f(\{d_h^{\wh\pi}\}) - f(\{\wh{d}_h^{\wh\pi}\}) + f(\{\wh{d}_h^{\wh\pi}\}) - f(\{\wh{d}_h^{\pi^*}\}) + f(\{\wh{d}_h^{\pi^*}\}) - f(\{d_h^{\pi^*}\}) \ge -L\veps.  \tag*{\qedhere}
    \end{align*}
\end{proof}

\paragraph{On $\wh d_h^\pi$ being invalid distributions} One potential issue is that some of the objective functions $f$ considered in the literature are only well defined for valid probability distributions (e.g., entropy).  This is easy to deal with in the online setting, as we can simply project $\wh d_h^\pi$ onto the probability simplex, which picks up a multiplicative factor of $2$ in $\|\wh d_h^\pi - d_h^\pi\|_1$ (c.f.~the analysis of the linearization step in \pref{alg:online_known}). 

For the offline setting, however, the situation can be trickier. For example, the above projection idea is clearly bad for return maximization, since after projection all $\wh d_h^\pi$ satisfy $\|\wh d_h^\pi\|_1 = 1$ and we lose pessimism. From an analytical point of view, pessimistic approaches (e.g., \pref{thm:offline_rf}) only pays one factor of the missingness error $\|\ol d_h^\pi - d_h^\pi\|_1$ by leveraging its one-sidedness, and a factor of $2$ introduced by projection is simply unacceptable. Therefore, the question is whether we can generalize the pessimism in \pref{thm:offline_rf} to general objective functions. We only answer this question with a rough sketch and leave the full investigation to future work: roughly speaking, since we know $\|\wh d_h^\pi - \ol d_h^\pi\|_1 \le  \veps' $ (for some appropriate value of $\veps'$ from our analysis), we can form a version space for $d_h^\pi$ as:
$$
d_h^\pi \in \{d_h: \exists d_h', \textrm{s.t.~} d_h\ge d_h'  \textrm{ and } \|d_h' - \wh d_h^\pi\|_1 \le \veps' \}.
$$
Then we can simply come up with pessimistic evaluation of $f(\{d_h^\pi\})$ by minimizing $f(\{d_h\})$ over the above set. It is not hard to see that such an approach will provide similar guarantees to \pref{thm:offline_rf} when applied to return maximization.

\section{Alternative setups, algorithm designs, and analyses} \label{app:alt}
\subsection{Offline data assumptions} 
As mentioned in \pref{sec:offline}, our offline data assumption allows sequentially dependent batches, where in-batch tuples are i.i.d.~samples. This is already weaker than the standard fully i.i.d.~settings considered in the offline RL literature, and here we further comment on how to handle various extensions. 

\paragraph{Trajectory data} One simple setting is when data are i.i.d.~trajectories sampled from a fixed policy. (This setting does not fit our need for the online algorithm, but is a representative setup for the purpose of offline learning.) 
While our protocol directly handles it (we can simply split the data in $H$ chunks and call them $\Dcal_0, \Dcal_1, \ldots$), it seems somewhat wasteful as we only extract 1 transition tuple per trajectory, potentially worsening the sample complexity by a factor of $H$. This is because in our analysis of the regression step (\pref{alg:offline_known}, \pref{line:reg_off}), we treat the regression target (which depends on $\wh d_{h}^\pi$) as fixed and independent of the current dataset. If we want to use all the data, we would need to union bound over the target as well; see similar considerations in the work of \citet{fan2020theoretical}. A slow-rate analysis follows straightforwardly, and we leave the investigation of fast-rate analysis to future work. We also remark that our current offline setup (\pref{assum:data}) is the most natural protocol for the data collected from the online algorithm (\pref{sec:online}), and using full trajectory data does not seem to improve the theoretical guarantees of the online setting. 

\paragraph{Fully adaptive data} 
A more general setting than \pref{assum:data} is that the data is fully adaptive, i.e., each trajectory is allowed to depend on all trajectories that before it. To handle such a case, we will need to replace the i.i.d.~concentration inequalities with their martingale versions. Some special treatment in the concentration bounds will also be needed to handle the random data-splitting step in \pref{alg:offline_known}, \pref{line:split} \citep[c.f.~][]{mohri2008rademacher}; alternatively, if we union bound over regression targets (see previous paragraph), the data splitting step will no longer be needed. 

\paragraph{Unknown and/or non-Markov $\pi^D$} In \pref{assum:data} we assume that the last-step policy in the data-collecting policy is Markov and known, as we need it to form the importance weights on actions. When $\pi^D$ is still Markov and unknown, we can use behavior cloning to back it out from data, which would require some additional assumptions (e.g., having access to a policy class that realizes $\pi^D$), and we do not further expand on such an analysis. When $\pi^D$ is non-Markov, it is well known that the action in the data tuple $(x_h, a_h, x_{h+1})$ can be still treated as if it were generated from a Markov policy---one can compute the state-action occupancy for $(x_h, a_h)$ (which is well-defined even if $\pi^D$ is non-Markov) and then obtain the equivalent Markov policy by conditioning on $x_h$. Incidentally, the algorithmic solution is the same as the case of unknown Markov $\pi^D$, i.e., behavior cloning. 

\subsection{Stochastic and/or unknown reward functions}\label{app:reward}
When the reward function is stochastic but still known, \pref{prop:density2return} and all policy optimization guarantees extend straightforwardly, since we can still directly compute the return. The more nontrivial case is when the reward function $R$ is unknown and comes as part of the data, i.e., we have the usual format of data tuples that include (possibly) stochastic reward signals, 
$\{(x_h^{(i)}, a_h^{(i)}, r_h^{(i)})\}_{i=1}^{n_{\mathrm{ret}}} \sim d^D_h$. Then given estimates $\{\wh d^D_h\}$ (from MLE) and $\{\wh d^\pi_h\}$ (from \pref{alg:offline_known} or \pref{alg:online_known}), the expected return can be estimated by reweighting the rewards according to the importance weight $\wh d_h^\pi / \wh d^D_h$, and assuming this ratio is well-defined:
\[
    \wh v^\pi_R = \frac{1}{n_{\mathrm{ret}}} \sum_{i=1}^{n_{\mathrm{ret}}} \sum_{h\in[H]} \frac{\wh d_h^\pi (x_h^{(i)})}{ \wh{d}^D_h (x_h^{(i)}) } \frac{\pi_h(a_h^{(i)}|x_h^{(i)}) }{ \pi^D_h(a_h^{(i)}|x_h^{(i)}) } r_h^{(i)}.
\]
It can be shown that we then have $|\wh v^\pi_R - v^\pi_R| \le \veps + \text{(additive terms)}$, where the additive terms correspond to the statistical error of return and MLE estimation, which is $O((n_{\mathrm{ret}})^{-1/2})$. 
If $\wh{d}^D_h$ does not cover $\wh{d}_h^\pi$, which may generally be the case, clipping (e.g., according to thresholds $\Bx_h, \Ba_h$) can again be used, which will lead to additional error corresponding to clipped mass.

\subsection{Algorithm design and analyses}
\label{app:alg_anal}

In this section, we discuss alternative designs of the offline density learning algorithm (\pref{alg:offline_known}), as well as their downstream impacts on the online and representation learning algorithms, which use the offline module in their inner loops. For simplicity, most discussions are in the case of offline density learning with known representation $\mutrue$.

\para{Point estimate in denominator} 
First, we discuss alternative parameterizations of the weight function class. To enable more ``elementary" $\ell_\infty$ covering arguments, one may consider instead parameterizing the weight function class as a ratio of linear functions over a fixed function $v_h : \Xcal \rightarrow \RR$, specifically
\[
    \Wcal_{h}(v_h) = \cbr{w_{h} = \frac{\langle \mutrue_{h-1}, \theta_{h}\rangle}{v_h} :\nbr{w_{h}}_\infty \le  \Bx_{h-1}\Ba_{h-1},\theta_{h} \in \RR^{\dlr}}.
\]
When $\mutrue$ consists of simplex features, it can be shown that an $\ell_\infty$ covering with scale $\gamma$ of size $(1/\gamma)^{\dspanner}$ can be constructed for $\Wcal_{h}(v_h)$, because it can be induced by an $\ell_\infty$ covering of the low-dimensional parameter space that has scale adaptively chosen according to how much the weight can be perturbed with respect to the denominator, thus fixed size. 
It is unclear how to construct such $\ell_\infty$ coverings for ``linear-over-linear" function classes such as $\Wcal_{h}$ of \pref{alg:offline_known}. One may consider compositions of standard $\ell_\infty$ coverings generated separately for the linear numerator and denominator, but bounding the covering error is challenging due to sensitivity of the denominator to perturbations. 

As we will see, however, the key issue with such fixed-denominator parameterizations is that the Bayes-optimal solution is no longer realizable. To handle this in the analysis, we can introduce an additional \textit{approximation error} (similar to \citet[Assumption 3]{chen2019information} in the value learning setting) that will appear in the final bound, corresponding to how well the Bayes-optimal solution is approximated by the function class. Depending on the choice of denominator, the approximation error may not be controlled, or may lead to a slower rate of estimation; loosely, it is defined as 
\[
\veps^{\mathrm{approx}}_h = \max_{\substack{w_{h-1} :  \|w_{h-1}\|_\infty \le \Bx_{h-1} }} \min_{w_{h} \in \Wcal_{h}(v_h)} \nbr{ w_{h} - \opex^\pi_{h-1}(d^D_{h-1}w_{h-1}) }_{2, \dnext_{h-1}}.
\]

One obvious choice for the fixed denominator is $v_h = \hatdnext_{h-1}$, since it is immediately available from the MLE data estimation step, plus the linear numerator can then be extracted exactly through the elementwise multiplication $\wh{d}_h^\pi = \wh{w}_h^\pi \hatdnext_{h-1}$. However, the Bayes-optimal predictor $\opex^\pi_{h-1}(d_{h-1})$ is no longer realizable, since $\opex^\pi_{h-1}(d_{h-1}) = \opp^\pi_{h-1}(d_{h-1}) / \dnext_{h-1}$ is a linear function over the true data distribution $\dnext_{h-1}$. In this case, using \pref{lem:clipped_concentrability} gives a more interpretable upper bound on the approximation error involves the difference between the ratio of any linear $d_{h}$ covered on $\dnext_{h-1}$ and the corresponding ratio over $\hatdnext_{h-1}$: 
\[
\veps^{\mathrm{approx}}_h \le \max_{ \substack{d_{h} = \langle \mutrue_{h-1}, \theta_h \rangle : \\ d_{h} \le \Bx_{h-1} \Ba_{h-1} \dnext_{h-1}}} \nbr{\frac{d_{h}}{\hatdnext_{h-1}} -\frac{d_{h}}{\dnext_{h-1}}}_{2, \dnext_{h-1}}.
\] 
However such approximation error may be difficult to control even with small data estimation error due to sensitivity of the denominator (for example if $\|\hatdnext_{h-1} - \dnext_{h-1} \|_1 \le \veps_{\mle}$ but they have disjoint support).

\para{Barycentric spanner in denominator} 
To avoid the above support issue and control the approximation error, we can instead consider a denominator function upon which $\dnext_{h-1}$ is supported. This is satisfied by the barycentric spanner of the version space of the estimate $\hatdnext_{h-1}$,  
\[
\Vcal_{h} = \cbr{ v_{h} = \langle \mutrue_{h-1}, \theta_{h} \rangle : \|v_{h} - \hatdnext_{h-1}\|_1 \le \emle, \theta_{h} \in \RR^\dspanner },
\]
noting that $\dnext_{h-1} \in \Vcal_{h}$ with high probability due to the MLE guarantee.  
Then letting $\wt v_{h}$ denote the spanner, \pref{lem:barycentric} guarantees that $\frac{\dnext_{h-1}}{\wt v_{h}} \le \dspanner$, and the approximation error of $\Wcal_{h}(\wt v_{h})$ can be controlled by the error of MLE estimation, since for any $d_{h} \le \Bx_{h-1} \Ba_{h-1} \dnext_{h-1}$ we have
\begin{align*}
    \nbr{\frac{d_{h}}{\wt v_{h}} -\frac{d_{h}}{\dnext_{h-1}}}_{2, \dnext_{h-1}}^2 \le&~ (\Bx_{h-1} \Ba_{h-1})^2 \int \frac{\dnext_{h-1}(x)}{\wt v_{h}(x)} \rbr{1+\frac{\dnext_{h-1}(x)}{\wt v_{h}(x)}  } \abr{\wt v_{h}(x) - \dnext_{h-1}(x)} (\dd x) 
    \\
    \le&~ 2 (\Bx_{h-1} \Ba_{h-1} \dspanner)^2 \| \wt{v}_h - \dnext_{h-1}\|_1
\end{align*}
which implies that $\veps^{\mathrm{approx}}_h \le 2 \Bx_{h-1} \Ba_{h-1} \dspanner \sqrt{\veps_{\mle}}$ by the definition of $\Vcal_h$.  
However, since $\veps_{\mle}$ is $O(n_{\mle}^{-1/2})$, this results in a slow rate of $1/\veps^{4}$ total sample complexity for offline density estimation, and from a computational standpoint, introduces another barycentric spanner construction step in the algorithm which can be expensive. The representation learning setting has the additional challenge that there will be approximation error if the wrong representation $\wh\mu_{h-1}\in\Upsilon_{h-1}$ is chosen for $\hatdnext_{h-1}$, since $\dnext_{h-1} \notin \Vcal_{h}(\wh\mu_h)$ (we extend the definition to $\Vcal_{h}(\mu_{h-1}) = \cbr{ v_{h} = \langle \mu_{h-1}, \theta_{h} \rangle : \|v_{h} - \hatdnext_{h-1}\|_1 \le \emle, \theta_{h} \in \RR^\dspanner }$), which, as in the first case above, may be difficult to bound.

\para{Clipped function class with point estimate in denominator} 
Generalizing and improving upon the previous analyses, using a clipped version of the function class $\Wcal_h(v_h)$ 
\[
    \Wcal^\clip_{h}(v_h) = \cbr{w_{h} = \frac{\langle \mutrue_{h-1}, \theta_{h}\rangle \wedge \Bx_{h-1} \Ba_{h-1} v_h}{v_h} : \theta_{h+1} \in \RR^{\dlr}}
\]
will allow us to bound the approximation error for general denominator functions $v_h$. 
For any $d_{h}$ such that $d_h \le \Bx_{h-1} \Ba_{h-1} \dnext_{h-1}$, we can approximate the ratio $\frac{d_h}{\dnext_{h-1}}$ with $\frac{d_h \wedge \Bx_{h-1} \Ba_{h-1} v_h}{v_h} \in \Wcal^\clip_{h}(v_h)$, and separate the approximation error into two terms, based on whether $\dnext_{h-1}$ is covered by $v_h$ according to a threshold $C \ge 1$: 
\begin{align*}
    &~\nbr{\frac{d_h \wedge \Bx_{h-1} \Ba_{h-1} v_h}{v_h} -\frac{d_{h}}{\dnext_{h-1}}}_{2, \dnext_{h-1}}^2 
    \\
    \le&~  \nbr{\rbr{\frac{d_h \wedge \Bx_{h-1} \Ba_{h-1} v_h}{v_h} -\frac{d_{h}}{\dnext_{h-1}}} \cdot \one\sbr{\frac{\dnext_{h-1}(x)}{v_h(x)} \le C}}_{2, \dnext_{h-1}}^2 \tag{``covered"}
    \\
    &\quad +  \nbr{\rbr{\frac{d_h \wedge \Bx_{h-1} \Ba_{h-1} v_h}{v_h} -\frac{d_{h}}{\dnext_{h-1}}} \cdot \one\sbr{\frac{\dnext_{h-1}(x)}{v_h(x)} > C}}_{2, \dnext_{h-1}}^2 \tag{``not covered"}\\
\end{align*}
Bounding the two terms individually, for the ``covered" term, we have
\begin{align*}
    \text{(``covered")} \le&~ \int_{x : \frac{\dnext_{h-1}(x)}{v_h(x)} \le C} \dnext_{h-1}(x) \rbr{\frac{d_h(x)}{v_h(x)} -\frac{d_{h}(x)}{\dnext_{h-1}(x)}}^2 (\dd x)
    \\
    \le&~ (\Bx_{h-1} \Ba_{h-1})^2\int_{x : \frac{\dnext_{h-1}(x)}{v_h(x)} \le C} \frac{\dnext_{h-1}(x)}{v_h(x)}\frac{(\dnext_{h-1}(x) - v_h(x))^2}{v_h(x)} (\dd x)
    \\
    \le&~ (\Bx_{h-1} \Ba_{h-1})^2C(1+C)\int_{x : \frac{\dnext_{h-1}(x)}{v_h(x)} \le C} \abr{\dnext_{h-1}(x) - v_h(x)} 
    \\
    \le&~ (\Bx_{h-1} \Ba_{h-1})^2C(1+C)\nbr{\dnext_{h-1} - v_h}_1.
\end{align*}

For the ``not covered" term, noticing that both parenthesized ratios are bounded on $[0, \Bx_{h-1}\Ba_{h-1}]$, we have
\begin{align*}
    \text{(``not covered")} \le&~ (\Bx_{h-1} \Ba_{h-1})^2 \int \dnext_{h-1}(x) \cdot \one\sbr{\frac{\dnext_{h-1}(x)}{v_h(x)} > C}  (\dd x) 
    \\
    \le&~ (\Bx_{h-1} \Ba_{h-1})^2 \rbr{ 1 - \frac{1}{C} }^{-1} \nbr{ \dnext_{h-1} - v_h }_1,
\end{align*}
where the second inequality is because 
\begin{align*}
    \rbr{ 1 - \frac{1}{C} } \int_{x : \frac{\dnext_{h-1}(x)}{v_h(x)} > C} \dnext_{h-1}(x) (\dd x) <  \int_{x : \frac{\dnext_{h-1}(x)}{v_h(x)} > C} (\dnext_{h-1}(x) - v_h(x)) (\dd x) \le \nbr{ \dnext_{h-1} - v_h}_1 
\end{align*}
since $\frac{\dnext_{h-1}}{C}  > v_h$. Thus in total, we have 
\begin{align*}
    \veps^{\mathrm{approx}}_h \le \Bx_{h-1} \Ba_{h-1} \rbr{C  + C^2 + \frac{C}{C-1}} \sqrt{\nbr{ \dnext_{h-1} - v_h }_1}.
\end{align*}
The bound depends on how close the point estimate $v_h$ is to the true $\dnext_{h-1}$, as well as the threshold $C$. 
In the case where $v_h = \hatdnext_{h-1}$ is the point estimate, we are now able to bound $\veps^{\mathrm{approx}}_h \le \Bx_{h-1} \Ba_{h-1} (C  + C^2 + \frac{C}{C-1}) \sqrt{
\veps_{\mle}}$, which results in a slower rate than our results in the main text. If $v_h = \wt{v}_h$ is the barycentric spanner of the version space, then it suffices to set $C = \dspanner$, in which case only the ``covered" part of the error is nonzero, and we recover the analysis in the previous paragraph. 

In general, the best choice of threshold $C$ is not obvious because $\dnext_{h-1}$ is not known, and will trade off between the two errors. When $C$ is large, the ``covered" error will be large since it is proportional to $C^2$, while if $C$ is too small (too close to 1), the ``not-covered" error will be large since it is proportional to $\frac{C}{C - 1}$.

\para{Direct extraction of the estimate} 
Putting aside the discussion of point estimates in the denominator, we now present an alternative to pointwise multiplication + linearization used to extract $\wh{d}_h^\pi$ from \pref{alg:offline_known}. 
Instead, we can directly extract the numerator, which will already be a linear function (in $\mutrue$), from weight ratio and use it as the estimate for $\wh{d}_h^\pi$. The regression objective might then be (replacing \pref{line:reg_off} in \pref{alg:offline_known})
\[
    \_ ~, \wh{d}_h^\pi = \argmin_{v_h \in \Vcal_h} \argmin_{d_h \in \Fcal_h(v_h)} \Lcal_{\Dcal_{h-1}^{\reg}}\rbr{\frac{d_h}{v_h}, \frac{\cliphatdpi{h-1}{\hatd}}{\wh d^D_{h-1}}, \pi_{h-1} \wedge \Ba_{h-1} \pi_{h-1}^D }, 
\]
where the version space of denominator functions $\Vcal_h$ is defined above, and $\Fcal_h(v_h) = \{d_{h} = \langle \mutrue_{h-1}, \theta_h \rangle :\nbr{d_h / v_h}_\infty \le  \Bx_{h-1}\Ba_{h-1},\theta_{h} \in \RR^{\dlr}\}$ represents linear numerator functions covered by $v_h$. It is necessary to constrain the denominator functions to the version space in order to ensure that the numerator is close to the true density, since regression only guarantees quality of estimated weight. For example, even if $\wh{w}_h^\pi = w_h^\pi$, if the denominator function is $c \cdot \dnext_{h-1}$ then the numerator will be $c \cdot d^\pi_h$, leading to large $\wh{d}_h^\pi$ estimation error. In terms of the analysis, this is quantified as the error between the denominator and true $\dnext_{h-1}$ in \cref{eq:hatbar3}, which is controlled by $\emle$ when the denominator is constrained to the version space $\Vcal_h$, and will result in the same guarantee as we have for \pref{alg:offline_known} and \pref{alg:online_known} in the known feature setting. In the online setting with known features, direct extraction has the advantage of no longer requiring the linearization step (\pref{line:line_approx} in \pref{alg:online_known}), though it is computationally more expensive because the function classes are jointly optimized, and the version space must be maintained. This advantage is lost in the representation learning setting because the estimates $\{\wh d_h^\pi\}_{\pi \in \Pi}$ must be jointly re-linearized with the same representation in order to construct the policy cover (\pref{line:line_approx_unknown} of \pref{alg:online_unknown}). 

\para{MLE instead of regression} An alternative to using regression to estimate the occupancy is instead using MLE-type estimation. Along similar veins as the regression algorithm, (a clipped version of) the previous-level estimate $\wh{d}_{h-1}^\pi$ must be reused to reweight the data distribution in order to estimate $\wh{d}_h^\pi$: 
\[
    \wh{d}_{h}^\pi = \argmin_{ f_{h} \in \Fcal_{h} }\frac{1}{n} \sum_{i=1}^n \frac{ \wh{d}_{h-1}^\pi \wedge \Bx_{h-1}\wh{d}^D_{h-1}}{\wh{d}^D_{h-1}} \frac{\pi_{h-1} \wedge \Ba_{h-1} \pi^D_{h-1}}{ \pi^D_{h-1} } \log (f_{h}). 
\]
where $\Fcal_{h}$ is some linear function class. One possible advantage of such an approach is that a linear density estimate can be directly learned, but establishing formal guarantees for an MLE-type algorithm remains future work. 
After separating the missingness error $\|d_h^\pi - \ol{d}_h^\pi\|_1$ in the same way as in \pref{sec:offline}, 
similar methods as classical MLE analysis (\pref{app:mle}) might be used to control $\|\wh{d}_h^\pi - \ol{d}_h^\pi\|_1$. The challenge is that such MLE analyses require $\Fcal_{h}$ to include only valid densities $\in \Delta(\Xcal)$, but this is at odds with reweighted MLE objectives such as the one above, since the weights $\frac{ \wh{d}_{h-1}^\pi \wedge \Bx_{h-1}\wh{d}^D_{h-1}}{\wh{d}^D_{h-1}}$ generally will not induce a valid density when multiplied with the data distribution. 

\subsection{Discussion of other approaches for controlling error exponentiation in the online setting} \label{app:no_mle}

\paragraph{Barycentric spanner in regression target (without clipping)}
In \pref{sec:online} we controlled the error exponentiation arising from having only approximately exploratory data by first clipping the regression target $\wh d_h^\pi / \wh d_h^D$ (since the MLE estimate $\wh d_h^D$ does not necessarily cover $\wh d_h^\pi$), then separating the error $\|\wh{d}_h^\pi - d_h^\pi\|_1$ into the ``two-sided regression error" and ``one-sided missingness error". 
It  will be instructive to also look at an alternative approach that avoids clipping and ``pretends'' that  data is perfectly exploratory, which provides interesting insights on the underlying issue and the delicacy of error propagation in our problem from a different perspective. 

The seemingly feasible solution is based on the observation that $\frac{1}{\dspanner} \sum_{i=1}^\dspanner \wh d_h^{\pi^{h,i}}$, the barycentric spanner of $\{\wh d_h^\pi\}_{\pi \in \Pi}$ in the denominator of \cref{eq:approx_bary}, \emph{is} a good approximation of $d_h^D$. So instead of using MLE to estimate $d^D_h = \frac{1}{\dspanner} \sum_{i=1}^\dspanner d_h^{\pi^{h,i}}$, we could simply use $\frac{1}{\dspanner} \sum_{i=1}^\dspanner \wh d_h^{\pi^{h,i}}$, which will keep the regression target bounded in \pref{alg:offline_known} without any clipping.  

However, a closer look reveals that this only sweeps the issue under the rug. The problem does not go away, and only appears in a different form: recall from \pref{lem:regression_decomposition} that the bound includes a term of $2 \dspanner \nbr{ \wh{d}^D_{h} - d^D_{h}}_1$, and when we use $\frac{1}{\dspanner} \sum_{i=1}^\dspanner \wh d_h^{\pi^{h,i}}$ to replace $\wh{d}^D_{h}$, we obtain
\[
    \nbr{ \wh{d}^D_{h} - d^D_{h}}_1 = \nbr{\frac{1}{\dspanner} \sum_{i=1}^\dspanner \wh d_h^{\pi^{h,i}} - \frac{1}{\dspanner} \sum_{i=1}^\dspanner d_h^{\pi^{h,i}}}_1 \le \max_{\pi \in \Pi} \|\wh d_h^\pi - d_h^\pi \|_1 
\]
which,
in addition to merging the two inductive chains, 
gives us $\|\wh d_h^\pi - d_h^\pi \|_1 \le (1 + \dspanner) \max_{\pi \in \Pi} \|\wh d_h^\pi - d_h^\pi \|_1  + \ldots$, resulting in  $O(\dspanner)^H$ error. In other words, because the error of the denominator distribution depends on the quality of regression, even with full coverage we will suffer the same error exponentiation issues.

\paragraph{Reachability-based approach} 
Error exponentiation can be avoided if a \textit{reachability} assumption \citep{du2019provably,modi2021model} is satisfied in the underlying MDP.  Formally, this assumption requires that  there exists a constant $\etamin$ such that $\forall h \in [H], z \in \Zcal_{h+1}$ we have $\max_{\pi \in \Pi} \PP_\pi[z_{h+1} = z] \ge \etamin$, where $\Zcal_{h+1}$ correspond to the latent states of the MDP. For example, in the case where $\mutrue_h$ is full-rank and composed of simplex features, $\Zcal_{h+1} = \{1,\ldots,\dlr\}$ and $\theta_h[i]$ directly corresponds to $\PP_\pi[z_{h+1} = i]$ for $i \in \{1,\ldots,\dlr\}$. 
The direct implication is that we can construct a fully exploratory policy cover that reaches all latent states (and thus covers all $\pi \in \Pi$) as long as we find, for each latent state, the policy that reaches it with probability at least $\etamin$. This policy can be found as long as $\wh{d}_h^\pi$ is estimated sufficiently well, which when backed up implies the latent state visitation is estimated sufficiently well. 

Specifically, in the offline module used in \pref{alg:online_known}, we can instead set $n_{\reg}$ such that $\|\wh{d}_h^\pi - d_h^\pi\|_1 \le \sigma_{\min}(\mutrue_{h-1})\etamin/4$ for all $\pi \in \Pi$, which implies that when backed up to latent states the error of estimation is $\|\wh\theta_h^\pi - \theta_h^\pi\|_\infty \le \etamin/4$. Then the exploratory policy cover can be chosen as $\Piexpl_h = \{\pi^{h,i}\}_{i=1}^\dspanner$ where for each $i \in \{1,\ldots,\dlr\}$, $\pi^{h,i}$ is such that $\wh\theta_h^{\pi^{h,i}}[i] \ge \etamin/4$, which implies $\theta_h^{\pi^{h,i}}[i] \ge \etamin/2$ with high probability, and such a policy is guaranteed to exist from the reachability assumption. 
Since the policy cover is fully exploratory, a single induction chain in the error analysis (instead of the two in \pref{fig:double-chain}) will suffice.

\section{Off-policy occupancy estimation proofs (\pref{sec:offline})} 
\label{app:offline}

\subsection{Discussion of clipping thresholds for $\bar{d}^\pi$}\label{app:clipd_threshold}

As we have previously mentioned, the clipped occupancy $\ol{d}_h^\pi$ depends on clipping thresholds $\{\Bx_h\}$ and $\{\Ba_h\}$ that are hyperparameter inputs to the offline estimation algorithm (\pref{alg:offline_known}). To better understand the effects of $\Bx_h, \Ba_h$ on $\ol{d}_h^\pi$ and downstream analysis, we highlight three properties below, which we have written only for $\Bx_h$ (but that take analogous forms for $\Ba_h$). 

Importantly, property 3 shows that the missingness error $\|\ol d_h^\pi - d_h^\pi\|_1$ is Lipschitz in the clipping thresholds $\{\Bx_h\}$, indicating that small changes in $\Bx_h$ will only lead to small changes in the missingness error, and thus the result of \pref{thm:offline_d_known}. For practical purposes, this serves as a reassurance that, within some limit, misspecifications of $\Bx_h, \Ba_h$ in the algorithm do not have catastrophic consequences. 

\begin{proposition}
    For two sets of clipping thresholds $\{\Bx_h\}, \{(\Bx_h)'\}$, following \pref{def:clipd}, for each $h = 1,\ldots,H$ let their corresponding clipped occupancies be defined recursively as 
    \begin{align*}
        \ol{d}_{h}^\pi =&~ \opp^{\ol\pi}_{h-1} ~\rbr{\clipbardpi{h-1}{d^D}}
        \\ 
        (\ol{d}_{h}^\pi)' =&~ \opp^{\ol\pi}_{h-1} ~\rbr{ (\ol{d}_{h-1}^\pi)' \wedge (\Bx_{h-1})' d^D_{h-1}}
    \end{align*}
    with $\ol{d}_0^\pi = (\ol{d}_0^\pi)' = d_0$. Then the following two properties hold for each $h \in [H]$: 
    \begin{enumerate}[leftmargin=*]
    \item (Monotonicity) $\ol{d}_h^\pi \le (\ol{d}_h^\pi)' $ if $\Bx_{h'} \le (\Bx_{h'})'$ for all $h' < h$. The relationship also holds in the other direction, i.e., replacing ``$\le$" with ``$>$". 
    \item (Clipped occupancy Lipschitz in thresholds) $\|(\ol d_h^\pi)'- \ol d_h^\pi \|_1 \le \sum_{h' < h} |(\Bx_{h'})' - \Bx_{h'}|$.
    \item (Missingness error Lipschitz in thresholds) $\abr{\|d_h^\pi - (\ol d_h^\pi)'\|_1 - \|d_h^\pi - \ol d_h^\pi\|_1} \le \sum_{h' < h} |(\Bx_{h'})' - \Bx_{h'}|$.
\end{enumerate}
\end{proposition}
\begin{proof}
We prove these three claims one by one.
\paragraph{Proof of Claim 1}
We will prove Claim 1 via induction. Suppose $\ol d_{h'-1}^\pi \le (\ol d_{h'-1}^\pi)'$ for some $h' \le h$. This holds for the base case $h' = 1$ since $\ol d_0^\pi = (\ol d_0^\pi)'$. Then since $\Bx_{h'-1} \le (\Bx_{h'-1})'$,
\begin{align*}
\ol d_{h'}^\pi = \opp^{\ol\pi}_{h'-1}\rbr{ \ol d_{h'-1}^\pi \wedge \Bx_{h'-1} d^D_{h'-1}} \le \opp^{\ol\pi}_{h'-1}\rbr{ (\ol d_{h'-1}^\pi)' \wedge (\Bx_{h'-1})' d^D_{h'-1}} = (\ol d_{h'}^\pi)'. 
\end{align*}
Then by induction we have that $\ol d_h^\pi \le (\ol d_h^\pi)'$. 
\paragraph{Proof of Claim 2}
For Claim 2, using \pref{lem:opexp_ineq}, we have
\begin{align*}
 &~\|(\ol d_h^\pi)' - \ol d_h^\pi\|_1 
 \\
 \le&~ \nbr{ \rbr{\clipbardpi{h-1}{d^D}} - \rbr{(\ol{d}_{h-1}^\pi)' \wedge (\Bx_{h-1})' d^D_{h-1}} }_1 
 \\
 \le&~ \nbr{ \rbr{\clipbardpi{h-1}{d^D}} - \rbr{(\ol{d}_{h-1}^\pi)' \wedge \Bx_{h-1} d^D_{h-1}} }_1 + \nbr{ \rbr{(\ol{d}_{h-1}^\pi)' \wedge \Bx_{h-1} d^D_{h-1}} - \rbr{(\ol{d}_{h-1}^\pi)' \wedge (\Bx_{h-1})' d^D_{h-1}} }_1
 \\
 \le&~ \nbr{ \ol d_{h-1}^\pi - (\ol{d}_{h-1}^\pi)' }_1 + \nbr{ \Bx_{h-1} d^D_{h-1} - (\Bx_{h-1})' d^D_{h-1} }_1 
 \\
 =&~ \nbr{ \ol d_{h-1}^\pi - (\ol{d}_{h-1}^\pi)' }_1 + \abr{ \Bx_{h-1} - (\Bx_{h-1})' }.
\end{align*}
Unfolding this recursion from level $h-1$ through level $0$ gives the result. 
\paragraph{Proof of Claim 3}
For Claim 3, we have
\begin{align*}
\abr{\|d_h^\pi - (\ol d_h^\pi)'\|_1 - \|d_h^\pi - \ol d_h^\pi\|_1} =&~ \abr{\int |d_h^\pi(x) - (\ol d_h^\pi)'(x)| - |d_h^\pi(x) - \ol d_h^\pi(x)| (\dd x) } 
\\
\le&~ \int \abr{|d_h^\pi(x) - (\ol d_h^\pi)'(x)| - |d_h^\pi(x) - \ol d_h^\pi(x)|} (\dd x) 
\\
\le&~ \int \abr{(\ol d_h^\pi)'(x)- \ol d_h^\pi(x)} (\dd x) &&\text{(since $||x|-|y|| \le |x-y|$)}
\\
=&~ \nbr{(\ol d_h^\pi)' - \ol d_h^\pi}_1
\end{align*}
Then applying Claim 2 gives the stated claim. 
\end{proof}

\subsection{Proof of occupancy estimation}
\label{app:off_occu}

\begin{proposition*}[Restatement of \pref{prop:clipd}]  We have the following properties for $\ol{d}_{h}^\pi$:
\begin{enumerate}[leftmargin=*]
\item $\ol{d}_{h}^\pi \le d_h^\pi$.
\item $\ol{d}_{h}^\pi = d_h^\pi$ when data covers $\pi$, i.e., $\forall h' < h$ we have $d_{h'}^\pi \le \Bx_{h'} d^D_{h'}$ and $\pi_{h'} \le \Ba_{h'} \pi^D_{h'}$. 
\item $\|\ol{d}_h^\pi - d_h^\pi\|_1 \le \|\ol{d}_{h-1}^\pi - d_{h-1}^\pi\|_1 + \|  \ol{d}_{h-1}^\pi-\clipbardpi{h-1}{d^D}\|_1  + \| \opp^{\pi}_{h-1} d_{h-1}^\pi-\opp^{\ol\pi}_{h-1} d_{h-1}^\pi\|_1.$
\end{enumerate}
\end{proposition*}
\begin{proof}
We prove these three claims one by one.
\paragraph{Proof of Claim 1} 
Firstly, we have $\ol{d}_{h}^\pi=d_h^\pi=\initdist$. Assuming the claim holds for $h'-1$, then we have $\ol{d}_{h'}^\pi =\opp^{\ol\pi}_{h'-1} (\clipbardpi{h'-1}{d^D})\le\opp^{\pi}_{h'-1} (\clipbardpi{h'-1}{d^D})\le\opp^{\pi}_{h'-1} (d_{h'-1}^\pi \wedge \Bx_{h'-1} d^D_{h'-1})\le\opp^{\pi}_{h'-1} d_{h'-1}^\pi=d_{h'}^\pi.$ By induction, we complete the proof.
\paragraph{Proof of Claim 2} It is easy to see that $d_{h'}^\pi \le \Bx_{h'} d^D_{h'}$ together with Claim 1 implies $\ol{d}_{h'}^\pi \le \Bx_{h'} d^D_{h'}$, thus
$\|  \ol{d}_{h'}^\pi-\clipbardpi{h'}{d^D}\|_1=0$. In addition, $\pi_{h'} \le \Ba_{h'} \pi^D_{h'}$ gives us $\pi_{h'} = \ol \pi_{h'}$, therefore $\nbr{\opp^{\pi}_{h'-1} d_{h'-1}^\pi - \opp^{\ol \pi}_{h'-1} d_{h'-1}^\pi}_1=0$. Now we can prove Claim 2 inductively. For $h'=0$, we know the claim holds since $\ol{d}_{0}^\pi = d_0^\pi=\initdist$. Assuming the claim holds for $h'-1$, by Claim 3 we have that
\[
0\le\|\ol{d}_{h'}^\pi - d_{h'}^\pi\|_1 \le \|\ol{d}_{h'-1}^\pi - d_{h'-1}^\pi\|_1 + \|  \ol{d}_{h'-1}^\pi-\clipbardpi{h'-1}{d^D}\|_1  + \| \opp^{\pi}_{h'-1} d_{h'-1}^\pi-\opp^{\ol\pi}_{h'-1} d_{h'-1}^\pi\|_1=0.
\]

This means the claim holds for $h'$. By induction, we complete the proof.
\paragraph{Proof of Claim 3} For the third part, we have the following decomposition
\begin{align*}
    \|\ol{d}_{h}^\pi - d_h^\pi\|_1 =&~ \nbr{\opp^{\ol\pi}_{h-1}  ~\rbr{\clipbardpi{h-1}{d^D}} - \opp^{\pi}_h d_{h-1}^\pi}_1 
    \\
    \le &~ \nbr{\opp^{\ol\pi}_{h-1} ~\rbr{\clipbardpi{h-1}{d^D}} - \opp^{\ol\pi}_{h-1} d_{h-1}^\pi}_1 + \nbr{\opp^{\ol\pi}_{h-1} d_{h-1}^\pi - \opp^{\pi}_h d_{h-1}^\pi}_1 
    \\
    \le&~ \nbr{\clipbardpi{h-1}{d^D} - d_{h-1}^\pi}_1 + \nbr{\opp^{\ol\pi}_{h-1} d_{h-1}^\pi - \opp^{\pi}_h d_{h-1}^\pi}_1  \tag{\pref{lem:opexp_ineq}}
    \\
    \le&~ \nbr{\clipbardpi{h-1}{d^D} - \ol{d}_{h-1}^\pi}_1 + \nbr{\ol{d}_{h-1}^\pi - d_{h-1}^\pi}_1 + \nbr{\opp^{\ol\pi}_{h-1} d_{h-1}^\pi - \opp^{\pi}_h d_{h-1}^\pi}_1. \qedhere
\end{align*}
\end{proof}

\begin{lemma*}[Restatement of \pref{lem:regression_decomposition}]
    For every $h \in [H]$, the error between estimates $\wh{d}_h^\pi$ from \pref{alg:offline_known} and the clipped target $\ol{d}_h^\pi$ is decomposed recursively as
    \begin{align*}
        \left\| \wh{d}_{h}^\pi - \ol{d}_{h}^\pi \right\|_1 \le&~ \nbr{ \wh{d}_{h-1}^\pi - \ol{d}_{h-1}^\pi }_1 + 2\Bx_{h-1} \nbr{ \wh{d}^D_{h-1} - d^D_{h-1} }_1 + \Bx_{h-1}\Ba_{h-1} \nbr{ \hatdnext_{h-1} - \dnext_{h-1} }_1
        \\
        &\quad + \sqrt{2}\nbr{\hatwpi_{h} - \opex^{\ol\pi}_{h-1} \rbr{d^D_{h-1}  \frac{\cliphatdpi{h-1}{\hatd}}{\wh{d}_{h-1}^D}}}_{2,\dnext_{h-1}},
    \end{align*}
where $(\opexp_h d_h) := (\oppro_{h} d_{h}) / \dnext_h$.
\end{lemma*}
\begin{proof}
We start by separating out the recursive term
\begin{align}
    &~\nbr{\wh{d}_{h}^\pi - \ol{d}_{h}^\pi}_1 = \nbr{\wh{d}_{h}^\pi - \opp^{\ol\pi}_{h-1}~\rbr{\clipbardpi{h-1}{d^D}}}_1 \notag
    \\
    \le&~ \nbr{\wh{d}_{h}^\pi - \opp^{\ol\pi}_{h-1}~\rbr{\cliphatdpi{h-1}{\hatd}}}_1 + \nbr{\opp^{\ol\pi}_{h-1} ~\rbr{\cliphatdpi{h-1}{\hatd}} - \opp^{\ol\pi}_{h-1} ~\rbr{\clipbardpi{h-1}{\hatd}}}_1
    \notag
    \\
    &~ \quad + \nbr{\opp^{\ol\pi}_{h-1} ~\rbr{\clipbardpi{h-1}{\hatd}} - \opp^{\ol\pi}_{h-1} ~\rbr{\clipbardpi{h-1}{d^D}}}_1 \notag
    \\ 
    \le&~ \nbr{\wh{d}_{h}^\pi - \opp^{\ol\pi}_{h-1} ~\rbr{\cliphatdpi{h-1}{\hatd}}}_1 + \nbr{\cliphatdpi{h-1}{\hatd} - \clipbardpi{h-1}{\hatd}}_1  \notag
    \\
    &~\quad + \nbr{ \clipbardpi{h-1}{\hatd} - \clipbardpi{h-1}{d^D}}_1\notag
    \\ 
    \le&~ \nbr{\wh{d}_{h}^\pi - \opp^{\ol\pi}_{h-1}~ \rbr{\cliphatdpi{h-1}{\hatd}}}_1 + \nbr{ \wh{d}_{h-1}^\pi - \ol{d}_{h-1}^\pi}_1 + \Bx_{h-1}\nbr{\hatd_{h-1}-d^D_{h-1}}_1.\label{eq:hatbar1}
\end{align}
Here, we apply \pref{lem:opexp_ineq} in the second inequality. The last inequality is due to $|\min(x,y)-\min(x,z)| \le |y-z|$ for $x,y,z\in \RR$.

Now, we consider the first term in \cref{eq:hatbar1} and get 
\begin{align}
    &~\nbr{\wh{d}_{h}^\pi - \opp^{\ol\pi}_{h-1}~ \rbr{\cliphatdpi{h-1}{\hatd}}}_1  \notag
    \\
    \le&~ \nbr{\wh{d}_{h}^\pi - \opp^{\ol\pi}_{h-1} \rbr{ \frac{\cliphatdpi{h-1}{\hatd}}{\hatd_{h-1}}\,d^D_{h-1} }}_1 \notag
    \\
    &~ \quad + \nbr{ \opp^{\ol\pi}_{h-1} \rbr{ \frac{\cliphatdpi{h-1}{\hatd}}{\hatd_{h-1}}\,d^D_{h-1}}-\opp^{\ol\pi}_{h-1} \rbr{ \frac{\cliphatdpi{h-1}{\hatd}}{\hatd_{h-1}}\,\hatd_{h-1}}}_1 \notag
    \\
    \le&~ \nbr{\wh{d}_{h}^\pi - \opp^{\ol\pi}_{h-1} \rbr{ \frac{\cliphatdpi{h-1}{\hatd}}{\hatd_{h-1}}\,d^D_{h-1}}}_1 + \Bx_{h-1} \nbr{d^D_{h-1} - \hatd_{h-1}}_1. \label{eq:hatbar2}
\end{align}
In the last inequality, we notice $\nbr{\frac{\cliphatdpi{h-1}{\hatd}}{\hatd_{h-1}}}_\infty \le \Bx_{h-1}$ by our convention $\frac{0}{0}=0$ and apply \pref{lem:opexp_ineq} again.

Let $\wt{w}_{h-1} := \frac{\cliphatdpi{h-1}{\hatd}}{\hatd_{h-1}}$ for short. Since $\|\wt{w}_{h-1}\|_\infty \le \Bx_{h-1}$, \pref{lem:clipped_concentrability} guarantees $\frac{\rbr{\opp^{\ol\pi}_{h-1}\rbr{d^D_{h-1} \wt{w}_{h-1}}}}{\dnext_{h-1}} \le \Bx_{h-1}\Ba_{h-1}$, thus the ratio is well-defined. Then we can further upper-bound the first term in \cref{eq:hatbar2} as 
\begin{align}
&~ \nbr{\wh{d}_{h}^\pi - \opp^{\ol\pi}_{h-1}\rbr{d^D_{h-1} \wt{w}_{h-1}}}_1 
= \nbr{\hatwpi_{h} \, \hatdnext_{h-1}- \frac{\opp^{\ol\pi}_{h-1} \rbr{d^D_{h-1}  \wt{w}_{h-1}}}{\dnext_{h-1}}\,\dnext_{h-1} }_1 \notag
\\
\le&~\nbr{\hatwpi_{h} \, \hatdnext_{h-1}- \hatwpi_{h} \, \dnext_{h-1} }_1 
+ \nbr{\hatwpi_{h} \, \dnext_{h-1} -  \frac{\opp^{\ol\pi}_{h-1} \rbr{d^D_{h-1}  \wt{w}_{h-1}}}{\dnext_{h-1}} \, \dnext_{h-1}}_{1} \notag
\\
=&~\nbr{\hatwpi_{h} \, \hatdnext_{h-1}- \hatwpi_{h} \, \dnext_{h-1} }_1 
+ \nbr{\hatwpi_{h} -  \frac{\opp^{\ol\pi}_{h-1} \rbr{d^D_{h-1}  \wt{w}_{h-1}}}{\dnext_{h-1}}}_{1,\dnext_{h-1}} \notag
\\
\le&~\nbr{\hatwpi_{h}}_{\infty} \nbr{\hatdnext_{h-1} - \dnext_{h-1}}_1 
+ \nbr{\hatwpi_{h} -  \frac{\opp^{\ol\pi}_{h-1} \rbr{d^D_{h-1}  \wt{w}_{h-1}}}{\dnext_{h-1}}}_{1,\dnext_{h-1}} \notag
\\
\le&~ \Bx_h \Ba_h \nbr{\hatdnext_{h-1} - \dnext_{h-1}}_1  + \nbr{\hatwpi_{h} -  \frac{\opp^{\ol\pi}_{h-1} \rbr{d^D_{h-1}  \wt{w}_{h-1}}}{\dnext_{h-1}}}_{2,\dnext_{h-1}}. \label{eq:hatbar3} 
\end{align}

Combining \cref{eq:hatbar1}, \cref{eq:hatbar2}, and \cref{eq:hatbar3} and noticing the definition of $\opexp_h$ and $\wt{w}_{h-1}$ completes the proof.
\end{proof}

\begin{theorem*}[Restatement of
\pref{thm:offline_d_known}]
Fix $\delta\in(0,1)$. Suppose \pref{assum:lowrank} and \pref{assum:data} hold, and $\mu^*$ is known. 
Then, given an evaluation policy $\pi$, by setting 
\[\nmle = \tilde{O}\rbr{\dspanner \rbr{\sum_{h\in[H]} \Bx_h \Ba_{h} }^2 \log(1/\delta)/\veps^2} \text{ and } \nreg = \tilde{O}\rbr{ \dspanner \rbr{\sum_{h\in[H]} \Bx_h \Ba_{h} }^2 \log(1/\delta)/\veps^2 },
\]
with probability at least $1-\delta$, \offalg (\pref{alg:offline_known}) returns state occupancy estimates $\{\wh d^\pi_h\}_{h=0}^{H-1}$ satisfying 
\[\nbr{\wh{d}_{h}^\pi - \ol d_h^\pi}_1 \le \veps, \forall h\in[H].\] 
The total number of episodes required by the algorithm is 
\[
\tilde{O}\rbr{\dspanner H \rbr{\sum_{h \in [H]} \Bx_h \Ba_{h} }^2\log(1/\delta)/\veps^2}.
\]
\end{theorem*}

\begin{proof}
We first make two claims on MLE estimation and error propagation.
\paragraph{Claim 1}
Our estimated data distributions satisfy that with probability $1-\delta/2$, for any $h\in[H]$
\begin{align}
    \nbr{ \hatd_h - d_h^D }_1 \le \emle \,\text{ and }\, \nbr{ \hatdnext_h - \dnext_h }_1 \le \emle,
\label{eq:emle_bound}
\end{align}
where 
\[
\emle:=6 \sqrt{\frac{\dlr \log (16H \munorm n_{\mle}/\delta)}{\nmle}}. 
\]

\paragraph{Claim 2} Under the high-probability event that \cref{eq:emle_bound} holds, we further have that with probability at least $1-\delta/2$, for any $1\le h\le H$,
\begin{align}
    \nbr{ \wh{d}_{h}^\pi - \ol{d}_{h}^\pi }_1 
    \le \nbr{ \wh{d}_{h-1}^\pi - \ol{d}_{h-1}^\pi }_1 + 3\Bx_{h-1}\Ba_{h-1} \evnext + \sqrt{2}\ereg{h-1},
\label{eq:offline_hattobar}
\end{align} 
where 
\begin{align*}
    \ereg{h-1} \defeq \sqrt{\frac{ 221184 \dspanner (\Bx_{h-1} \Ba_{h-1})^2\log\rbr{2Hn_{\reg}/\delta} }{ n_{\reg} }}.
\end{align*}

Now we establish the final error bound with these two claims. Notice that the total failure probability is less than $\delta$. Unfolding \cref{eq:offline_hattobar} from $h'=h$ to $h'=1$ and noticing that $\wh{d}_0^\pi = \ol{d}_0^\pi=\initdist$ yields that for any $h\in[H]$
\begin{align}
    \label{eq:offline_hattobar_unfold}
    \nbr{ \wh{d}_{h}^\pi - \ol{d}_{h}^\pi }_1 \le \sum_{h'=0}^{h-1}
    \rbr{3\Bx_{h'}\Ba_{h'} \evnext + \sqrt{2}\ereg{h'} }.
\end{align}
Substituting in the expressions for $\emle$ and $\ereg{}$, we have 
\begin{align}
    \label{eq:offline_hattobar_unfold_sub}
    \nbr{ \wh{d}_{h}^\pi - \ol{d}_{h}^\pi }_1 \le \sum_{h'=0}^{h-1}
    \rbr{18\Bx_{h'}\Ba_{h'} \sqrt{\frac{\dlr \log (16H \munorm n_{\mle}/\delta)}{\nmle}} + 666\Bx_{h'}\Ba_{h'} \sqrt{\frac{ \dspanner \log\rbr{2Hn_{\reg}/\delta} }{ n_{\reg} }} }.
\end{align}
It is easy to see that if we set
\[\nmle = \tilde{O}\rbr{\dspanner \rbr{\sum_{h\in[H]} \Bx_h \Ba_{h} }^2 \log(1/\delta)/\veps^2} \text{ and } \nreg = \tilde{O}\rbr{ \dspanner \rbr{\sum_{h\in[H]} \Bx_h \Ba_{h} }^2 \log(1/\delta)/\veps^2 },
\]
then we have 
\[\nbr{\wh{d}_{h}^\pi - \ol d_h^\pi}_1 \le \veps,\forall h\in[H].\]

In the following, we provide the proof of these two claims 
respectively.

\paragraph{Proof of Claim 1}
We start with a fixed $h\in[H]$ and bounding $\|\hatd_h-d^D_h\|_1$, where we recall that $\hatd_h$ is the MLE solution in \cref{eq:obj_mle}. By \pref{lem:opt_cover}, we know that function class $\Fcal_h$ has an $\ell_1$ optimistic cover with scale $1/\nmle$ of size $\rbr{2\lceil \munorm n_{\mle} \rceil}^{\dlr}$. It is easy to see that the true marginal distribution $d_h^D\in\Fcal_h$ from \pref{lem:mle_realizability} and any $d_h\in\Fcal_h$ is a valid probability distribution over $\Xcal$. From \pref{assum:data}, we know that once conditioned on prior dataset $\Dcal_{0:h-1}$, the current dataset $\Dcal_h^\mle$ is drawn i.i.d. from the fixed distribution denoted as $d_h^D$. Thus, \pref{lem:mle} tells us that when conditioned on $\Dcal_{0:h-1}$, with probability at least $1-\delta/(4H)$
\begin{align}
\label{eq:hatd_bound}
    \|\hatd_h - d_h^D\|_1 \le&~ \frac{1}{\nmle} + \sqrt{\frac{12\log (4H\rbr{2\lceil \munorm n_{\mle} \rceil}^{\dlr}/\delta)}{\nmle} + \frac{6}{\nmle}}
    \\
    \le&~ \frac{1}{\nmle} + \sqrt{\frac{12\dlr \log (16H \munorm n_{\mle}/\delta)}{\nmle} + \frac{6}{\nmle}}\notag
    \\
    \le&~ 6 \sqrt{\frac{\dlr \log (16H \munorm n_{\mle}/\delta)}{\nmle}}=\emle.
\end{align}

Since \cref{eq:hatd_bound} holds for any such fixed $\Dcal_{0:h-1}$, applying the law of total expectation gives us this that \cref{eq:hatd_bound} holds with probability $1-\delta/(4H)$ without conditioning on $\Dcal_{0:h-1}$. 

Similarly, with probability at least $1-\delta/(4H)$, for the MLE solution $\hatdnext_h$ we have $\|\hatdnext_h - \dnext_h\|_1 \le \emle.$ Union bounding these two high-probability events and further union bounding over $h\in[H]$ gives us that \cref{eq:emle_bound} holds with probability $1-\delta/2$.

\paragraph{Proof of Claim 2}
Notice that the proof in this part is under the high-probability event that \cref{eq:emle_bound} holds. 
We consider a fixed $h\in[H]$. From \pref{lem:regression_decomposition}, we have the error propagation result that
\begin{align}
    \left\| \wh{d}_{h}^\pi - \ol{d}_{h}^\pi \right\|_1 \le&~ \nbr{ \wh{d}_{h-1}^\pi - \ol{d}_{h-1}^\pi }_1 + 2\Bx_{h-1} \nbr{ \wh{d}^D_{h-1} - d^D_{h-1} }_1 + \Bx_{h-1}\Ba_{h-1} \nbr{ \hatdnext_{h-1} - \dnext_{h-1} }_1 \notag
    \\
    &\quad + \sqrt{2}\nbr{\hatwpi_{h} - \frac{\opp^{\ol\pi}_{h-1} \rbr{d^D_{h-1}  \wt{w}_{h-1}}}{\dnext_{h-1}}}_{2,\dnext_{h-1}},
    \label{eq:hatbar7}
\end{align}
where $\wt{w}_{h-1} := \frac{\cliphatdpi{h-1}{\hatd}}{\hatd_{h-1}}$.

Since $\wh{w}_h^\pi \in \Wcal_h$, we have $\|\wh{w}_h^\pi\|_\infty \le \Bx_h \Ba_h$. 
The last term on RHS isolates the finite-sample error of regression, involving the difference between the empirical minimizer $\wh{w}_h^\pi$ and the population minimizer $\frac{\opp^{\ol\pi}_{h-1} \rbr{d^D_{h-1}  \wt{w}_{h-1}}}{\dnext_{h-1}}$ of the regression objective. To bound this error, we apply \pref{lem:variance_decomposition} and \pref{lem:conc}, which give us that, with probability at least $1-\delta/(2H)$,  
\begin{align}
&~\nbr{\hatwpi_{h} - \frac{\opp^{\ol\pi}_{h-1} \rbr{d^D_{h-1}  \wt{w}_{h-1}}}{\dnext_{h-1}}}_{2, \dnext_{h-1}}^2 \notag
\\
=&~ \EE\sbr{\Lcal_{\Dcal_{h-1}^\reg}\rbr{\hatwpi_{h}, \wt w_{h-1},\ol\pi}} - \EE\sbr{\Lcal_{\Dcal_{h-1}^\reg}\rbr{\frac{\opp^{\ol\pi}_{h-1} \rbr{d^D_{h-1}  \wt{w}_{h-1}}}{\dnext_{h-1}},\wt w_{h-1},\ol\pi}} \notag
\\
\le&~ 2\rbr{\Lcal_{\Dcal_{h-1}^\reg}\rbr{\hatwpi_{h}, \wt w_{h-1},\ol\pi} - \Lcal_{\Dcal_{h-1}^\reg}\rbr{\frac{\opp^{\ol\pi}_{h-1} \rbr{d^D_{h-1}  \wt{w}_{h-1}}}{\dnext_{h-1}},\wt w_{h-1},\ol \pi}} + 2\ereg{h-1}^2
\label{eq:hatbar4}
\end{align}
where 
\[
    \ereg{h-1} \defeq \sqrt{\frac{ 221184 \cdot \dspanner (\Bx_{h-1} \Ba_{h-1})^2\log\rbr{2Hn_{\reg}/\delta} }{ n_{\reg} }}
\]
The first term in \cref{eq:hatbar4} compares the empirical regression loss of the empirical minimizer $\wh{w}_h^\pi$ against the population solution. In order to show that this is $\le 0$, we first need to check that $\frac{\opp^{\ol\pi}_{h-1} \rbr{d^D_{h-1}  \wt{w}_{h-1}}}{\dnext_{h-1}} \in \Wcal_{h}$. As we have previously seen, we have $\frac{\opp^{\ol\pi}_{h-1} \rbr{d^D_{h-1}  \wt{w}_{h-1}}}{\dnext_{h-1}} \le \Bx_{h-1} \Ba_{h-1}$ from \pref{lem:clipped_concentrability}, thus satisfying the norm constraints of $\Wcal_h$. Further, \pref{lem:opp_linear} guarantees that both the numerator and denominator are linear functions of $\mutrue_{h-1}$, i.e., $\opp^{\ol\pi}_{h-1} \rbr{d^D_{h-1}  \wt{w}_{h-1}} = \langle \mutrue_{h-1}, \thetaup_h \rangle$ and $\dnext_{h-1} = \langle \mutrue_{h-1}, \thetadown_h \rangle$ for some $\thetaup_h, \thetadown_h \in \RR^d$. Then since $\hatwpi_h$ minimzes the empirical regression loss \cref{eq:obj_reg}, we have 
\begin{align}
 \Lcal_{\Dcal_{h-1}^\reg}\rbr{\hatwpi_{h-1}, \wt w_{h-1},\ol\pi} - \Lcal_{\Dcal_{h-1}^\reg}\rbr{\frac{\opp^{\ol\pi}_{h-1} \rbr{d^D_{h-1}  \wt{w}_{h-1}}}{\dnext_{h-1}},\wt w_{h-1},\ol \pi}\le 0.
 \label{eq:hatbar5}
\end{align}

Combining \cref{eq:hatbar7}, \cref{eq:hatbar4}, \cref{eq:hatbar5} with the MLE bound of \cref{eq:emle_bound}, with probability at least $1-\delta/(2H)$ we have 
\begin{align*}
     \|\wh{d}_{h}^\pi - \ol{d}_{h}^\pi\|_1 \le&~ \|\wh{d}_{h-1}^\pi - \ol{d}_{h-1}^\pi\|_1 + 2\Bx_{h-1} \ev + \Bx_{h-1}\Ba_{h-1} \evnext + \sqrt{2}\ereg{h-1}
     \\
     \le&~ \|\wh{d}_{h-1}^\pi - \ol{d}_{h-1}^\pi\|_1 + 3\Bx_{h-1}\Ba_{h-1} \evnext + \sqrt{2}\ereg{h-1}.
\end{align*}

Finally, union bounding over $h\in[H]$, plugging in the definition of $\emle$, and rearranging gives that \cref{eq:offline_hattobar} holds with probability at least $1-\delta/2$.
\end{proof}

\subsection{Proof of offline policy optimization}
\label{app:offline_rf}
\begin{theorem*}[Restatement of \pref{thm:offline_rf}]
Fix $\delta \in (0, 1)$ and suppose \pref{assum:lowrank} and \pref{assum:data} hold. Given a policy class $\Pi$, let $\{\wh d_h^\pi\}_{h\in[H],\pi \in \Pi}$ be the output of running \pref{alg:offline_known}. Then with probability at least $1-\delta$, for any deterministic reward function $R$ 
and policy selected as   
$ \wh\pi_R = \argmax_{\pi\in \Pi} \wh{v}_R^\pi, $
we have 
\[
    v_R^{\wh\pi_R} \ge \argmax_{\pi \in \Pi} \ol{v}_R^\pi - \veps,
\]
where $\wh{v}_R^\pi \defeq \sum_{h=0}^{H-1} \iint \wh d_h^\pi(x_h) R(x_h,a_h) \pi(a_h|x_h) (\dd x_h) (\dd a_h)$ and $\ol{v}_R$ is defined similarly for $\{\ol{d}_h^\pi\}$. 
The total number of episodes required by the algorithm is
\[\tilde{O}\rbr{\dspanner H^3 \rbr{\sum_{h \in [H]} \Bx_h \Ba_{h}}^2\log(|\Pi|/\delta)/\veps^2}.
\]
Additionally, define the set of policies fully covered by the data to be 
    \begin{align*}
        \Pi^{\covered} = \cbr{ \pi \in \Pi : d_h^\pi = \ol{d}_h^\pi, \forall h \in [H] }. 
    \end{align*}
    Then with the same total number of episodes required by the algorithm, for any reward function $R$ and policy selected as $ \wh\pi_R = \argmax_{\pi\in \Pi^{\covered}} \wh{v}_R^\pi,$  with probability at least $1-\delta$, we have
    \[v_R^{\wh\pi_R} \ge \argmax_{\pi \in \Pi^{\covered} } v_R^\pi - \veps.\] 
\end{theorem*}
\begin{proof}


Firstly, \pref{thm:offline_d_known} states that, with probability at least $1-\delta/|\Pi|$,  $\tilde{O}\rbr{\dspanner H^3 \rbr{\sum_{h \in [H]} \Bx_h \Ba_{h} }^2\log(|\Pi|/\delta)/\veps^2}$ samples are sufficient for learning $\{\wh{d}_h^\pi\}$ such that $\|\wh{d}_{h}^\pi - \ol d_h^\pi\|_1 \le \frac{\veps}{2H}$ for all $h \in [H]$ and each $\pi \in \Pi$. Taking a union bound over $\pi \in \Pi$, with probability at least $1-\delta$, we have that for all $h\in[H],\pi \in \Pi$, 
\[
    \|\wh{d}_{h}^\pi - \ol d_h^\pi\|_1 \le \frac{\veps}{2H}. 
\]  
Then since the $R$ is bounded on $[0,1]$, for any $\pi \in \Pi$ we have 
\begin{align*}
    |\wh v^\pi_R - \ol{v}^\pi_R| = &~\sum_{h=0}^{H-1} \iint (\wh d_h^\pi(x_h)- \ol d_h^\pi(x_h)) R(x_h,a_h) \pi(a_h|x_h) (\dd x_h) (\dd a_h)
    \\
    \le &~ \sum_{h=0}^{H-1} \int |\wh d_h^\pi(x_h)- \ol d_h^\pi(x_h)|\rbr{\int \pi(a_h|x_h)  (\dd a_h)} (\dd x_h)
    \\
    =&~ \sum_{h=0}^{H-1} \|\wh{d}_{h}^\pi - \ol d_h^\pi\|_1 \le \veps / 2. 
\end{align*}  
Denote $\ol \pi^*_R = \argmax_{\pi \in \Pi} \ol v_R^\pi$, and recall that we pick $\wh\pi_R =\argmax_{\pi\in\Pi} \wh v^\pi_R$. Then 
\begin{align*}
    v^{\wh \pi_R}_R - \max_{\pi\in\Pi} \ol{v}^\pi_R = v^{\wh \pi_R}_R - \ol{v}^{\ol \pi_R^*}_R \ge \ol{v}^{\wh \pi_R}_R - \ol{v}^{\ol\pi_R^*}_R 
    = \ol v^{\wh \pi_R}_R - \wh v^{\wh \pi_R}_R + \wh v^{\wh \pi_R}_R - \wh v^{\ol\pi_R^*}_R + \wh v^{\ol\pi_R^*}_R - \ol v^{\ol\pi_R^*}_R \ge -\veps, 
\end{align*}
where the first inequality follows from the fact that $d_h^\pi \ge \ol{d}_h^\pi$, thus $v_R^\pi \ge \ol{v}_R^\pi$. The second inequality results from the fact that $\wh v^{\wh \pi_R}_R \ge \wh v^{\pi_R^*}_R$ and $|\wh v^\pi_R - \ol{v}^\pi_R| \le \veps / 2$ for all $\pi \in \Pi$. 

The result for $\Pi^{\covered}$ is a straightforward from the observation that for each $\pi \in \Pi^{\covered}$, we have $d_h^\pi = \ol{d}_h^\pi$ for all $h \in [H]$ and $v_R^\pi=\ol v_R^\pi$. 
\end{proof}

\section{Online policy cover construction proofs (\pref{sec:online})}
\label{app:online}

\subsection{Proof of occupancy estimation}
\label{app:online_occu}
\begin{lemma*}[Restatement of \pref{lem:missingness_decomposition}]
    For any $h \in [H]$ and $\pi \in \Pi$ in \pref{alg:online_known}, 
    \[
        \nbr{ \ol{d}_{h}^\pi - d_h^\pi }_1 \le \nbr{\ol{d}_{h-1}^\pi - d_{h-1}^\pi}_1  + 4\dspanner\max_{\pi' \in \Pi} \nbr{\wh{d}_{h-1}^{\pi'} - \ol{d}^{\pi'}_{h-1}}_1.
    \]
\end{lemma*}
\begin{proof}
Firstly, from the third claim of \pref{prop:clipd}, we have that for any $h\in[H], \pi\in\Pi$
\begin{align}
\label{eq:online_bartotrue1}
\|\ol{d}_h^\pi - d_h^\pi\|_1 \le \|\ol{d}_{h-1}^\pi - d_{h-1}^\pi\|_1 + \|  \ol{d}_{h-1}^\pi-\clipbardpi{h-1}{d^D}\|_1  + \| \opp^{\pi}_{h-1} d_{h-1}^\pi-\opp^{\ol\pi}_{h-1} d_{h-1}^\pi\|_1.
\end{align}

Now we further simplify the latter two error terms on the RHS of \cref{eq:online_bartotrue1} by noticing that $\Bx_h = \dspanner$ and $\Ba_h = K$ for all $h \in [H]$. For the last term, $\pi^D=\unif(\Acal)$ gives us 
\[\ol \pi (a_{h-1}|x_{h-1}) = \min\{\pi(a_{h-1}|x_{h-1}), \Ba_{h-1} \pi^D(a_{h-1}|x_{h-1})\}=\min\{\pi(a_{h-1}|x_{h-1}), 1)\}=\pi(a_{h-1}|x_{h-1})\]
and thus $\nbr{\opp^{\pi}_{h-1} d_{h-1}^\pi - \opp^{\ol \pi}_{h-1} d_{h-1}^\pi}_1=0$. For the middle term, we expand the expression as
\[
    \nbr{\ol{d}_{h-1}^\pi - \clipbarddpi{h-1}{d^D}}_1 = \int \ol{d}_{h-1}^\pi(x_{h-1}) - \rbr{\clipbarddpi{h-1}{d^D}}(x_{h-1}) (\dd x_{h-1}).
\]

Consider a fixed $x_{h-1}\in\Xcal$. Note that $\ol{d}_{h-1}^\pi(x_{h-1}) - \rbr{\clipbarddpi{h-1}{d^D}}(x_{h-1})$ is nonzero only if $\dspanner d^D_{h-1}(x_{h-1}) < \ol{d}_{h-1}^\pi(x_{h-1})$, for which we have 
\begin{align*}
    &~\ol{d}_{h-1}^\pi(x_{h-1}) - \rbr{\clipbarddpi{h-1}{d^D}}(x_{h-1})=\ol{d}_{h-1}^\pi(x_{h-1}) - \dspanner d^D_{h-1}(x_{h-1}) 
    \\
    \le&~\wh{d}_{h-1}^\pi(x_{h-1}) - \dspanner d^D_{h-1}(x_{h-1}) + \abr{\ol{d}_{h-1}^\pi(x_{h-1}) - \wh{d}_{h-1}^\pi(x_{h-1})}.
\end{align*}

To bound $\wh{d}_{h-1}^\pi(x_{h-1}) - \dspanner d^D_{h-1}(x_{h-1})$, we have 
\begin{align*}
    &~\wh{d}_{h-1}^\pi(x_{h-1}) - \dspanner d_{h-1}^D(x_{h-1}) 
    \\
    \le&~ \dlinpi_{h-1}(x_{h-1}) - \dspanner d_{h-1}^D(x_{h-1}) + \abr{\wh{d}_{h-1}^\pi(x_{h-1}) - \dlinpi_{h-1}(x_{h-1})} 
    \\
    \le&~ \sum_{i=1}^{\dlr} \abr{\dlin_{h-1}^{\pi^{h-1,i}}(x_{h-1})}  - \dlr d_{h-1}^D(x_{h-1}) + \abr{\wh{d}_{h-1}^\pi(x_{h-1}) - \dlinpi_{h-1}(x_{h-1})}  
    \\ 
    \le&~ \sum_{i=1}^{\dlr} \abr{\wh{d}_{h-1}^{\pi^{h-1,i}}(x_{h-1})}  - \dlr d_{h-1}^D(x_{h-1}) + (\dspanner+1) \max_{\pi' \in \Pi} \abr{\wh{d}_{h-1}^{\pi'}(x_{h-1}) - \dlin^{\pi'}_{h-1}(x_{h-1})}  
    \\ 
    \le&~ \sum_{i=1}^{\dlr} \abr{\ol{d}_{h-1}^{\pi^{h-1,i}}(x_{h-1})}  - \dlr d_{h-1}^D(x_{h-1}) + (\dspanner+1) \max_{\pi' \in \Pi} \abr{\wh{d}_{h-1}^{\pi'}(x_{h-1}) - \dlin^{\pi'}_{h-1}(x_{h-1})}  \\
    &+~ \dspanner \max_{\pi' \in \Pi} \abr{\wh{d}_{h-1}^{\pi'}(x_{h-1}) - \ol{d}^{\pi'}_{h-1}(x_{h-1})}  
    \\
    =&~ \sum_{i=1}^{\dlr} \ol{d}_{h-1}^{\pi^{h-1,i}}(x_{h-1})  - \dlr d_{h-1}^D(x_{h-1}) + (\dspanner+1) \max_{\pi' \in \Pi}\abr{\wh{d}_{h-1}^{\pi'}(x_{h-1}) - \dlin^{\pi'}_{h-1}(x_{h-1})}  \\
    &+~ \dspanner \max_{\pi' \in \Pi}\abr{\wh{d}_{h-1}^{\pi'}(x_{h-1}) - \ol{d}^{\pi'}_{h-1}(x_{h-1})}  
    \\
    \le&~ \sum_{i=1}^{\dlr} d_{h-1}^{\pi^{h-1,i}}(x_{h-1})  - \dspanner d^D_{h-1}(x_{h-1})  + (\dspanner +1) \max_{\pi' \in \Pi}\abr{\wh{d}_{h-1}^{\pi'}(x_{h-1}) - \dlin^{\pi'}_{h-1}(x_{h-1})}  \\
    &+~ \dspanner \max_{\pi' \in \Pi}\abr{\wh{d}_{h-1}^{\pi'}(x_{h-1}) - \ol{d}^{\pi'}_{h-1}(x_{h-1})} 
    \\
    =&~ (\dspanner +1) \max_{\pi' \in \Pi} \abr{\wh{d}_{h-1}^{\pi'}(x_{h-1}) - \dlin^{\pi'}_{h-1}(x_{h-1})} + \dspanner \max_{\pi' \in \Pi}\abr{\wh{d}_{h-1}^{\pi'}(x_{h-1}) - \ol{d}^{\pi'}_{h-1}(x_{h-1})}.  
\end{align*}
In the second inequality, we use that $\Piexpl_{h-1}=\{\pi^{h-1,1},\ldots,\pi^{h-1,\dspanner}\}$ are the policies corresponding to the barycentric spanner,  which \pref{lem:barycentric} guarantees to be of cardinality no larger than $\dspanner$.  
The first equality is because $\ol{d}_{h-1}^{\pi}(x_{h-1}) \ge 0,\forall \pi$, which can be seen by the induction definition in \cref{eq:def_dbar} and the non-negativity of $\initdist$. The fifth inequality is due to $\ol d^{\pi}_{h-1}(x_{h-1}) \le  d^{\pi}_{h-1}(x_{h-1}),\forall \pi$, which can be shown inductively by noticing $\ol d^{\pi}_{0}\le d^{\pi}_{0}$ and the definition of $\ol d^{\pi}_h$ in \cref{eq:def_dbar}. The last equality can be seen from that $d^D_{h-1}(x_{h-1})$ is the marginal distribution of $\Dcal_{h-1}$ and $\Dcal_{h-1}$ is rolled in with $\unif(\Piexpl_{h-1})$.

Integrating over $x_{h-1}$ yields 
\begin{align*}
    \nbr{ \ol{d}_{h-1}^\pi - \clipbarddpi{h-1}{d^D} }_1 \le (\dspanner + 1) \max_{\pi' \in \Pi} \nbr{\wh{d}_{h-1}^{\pi'} - \dlin^{\pi'}_{h-1}}_1 + (\dspanner + 1)\max_{\pi' \in \Pi} \nbr{\wh{d}_{h-1}^{\pi'} - \ol{d}^{\pi'}_{h-1}}_1.
\end{align*}

Since $\ol{d}^{\pi'}_{h-1} = \opp^{\pi'}_{h-2} (\ol{d}_{h-2}^\pi \wedge \Bx_{h-2} d^D_{h-2})= \opp^{\pi'}_{h-2} ( \ol{d}_{h-2}^\pi \wedge \dspanner d^D_{h-2})$ is linear in the features $\mu_{h-2}^*$ (\pref{lem:opp_linear}), and $\dlin^{\pi'}_{h-1}$ is the closest linear approximation in the $\ell_1$ norm to $\wh{d}^{\pi'}_{h-1}$ (\pref{line:line_approx}), for any $\pi' \in \Pi$ we have 
\begin{align}
\label{eq:diff_repr}
\nbr{\wh{d}^{\pi'}_{h-1} - \dlin^{\pi'}_{h-1}}_1 \le \nbr{\wh{d}^{\pi'}_{h-1} - \ol{d}^{\pi'}_{h-1}}_1
\end{align}
and thus  
\begin{align}
    \nbr{ \ol{d}_{h-1}^\pi - \clipbarddpi{h-1}{d^D} }_1 \le  2(\dspanner + 1)\max_{\pi' \in \Pi}\nbr{\wh{d}_{h-1}^{\pi'} - \ol{d}^{\pi'}_{h-1}}_1.
    \label{eq:online_bartotrue8}
\end{align}

Then combining \cref{eq:online_bartotrue1} with \cref{eq:online_bartotrue8} gives
\begin{equation*}
     \nbr{ \ol{d}_{h}^\pi - d_h^\pi }_1 \le \nbr{\ol{d}_{h-1}^\pi - d_{h-1}^\pi}_1  + 4\dspanner\max_{\pi' \in \Pi} \nbr{\wh{d}_{h-1}^{\pi'} - \ol{d}^{\pi'}_{h-1}}_1. \qedhere
\end{equation*}
\end{proof}

\begin{theorem*}[Restatement of \pref{thm:online_d}]
Fix $\delta\in(0,1)$ and consider an MDP $\Mcal$ that satisfies \pref{assum:lowrank}, where the right feature $\mu^*$ is known. Then by setting 
\[\nmle = \wt O\rbr{\frac{\dlr^3 K^2 H^4\log(1/\delta)}{\veps^2}}, \nreg=\wt O\rbr{\frac{\dlr^{5} K^2 H^4\log(|\Pi|/\delta)}{\veps^2}},n=\nmle + \nreg,
\]
with probability at least $1-\delta$, \onalg returns state occupancy estimates $\{\wh d^\pi_h\}_{h=0}^{H-1}$ satisfying that
\begin{align*}
    \|\wh{d}_{h}^\pi - d_h^\pi\|_1 \le \veps, \forall h\in[H],\pi\in\Pi.
\end{align*}
The total number of episodes required by the algorithm is 
\[
\wt O(nH)=\wt O\rbr{\frac{\dlr^{5} K^2 H^5\log(|\Pi|/\delta)}{\veps^2}}.
\]
\end{theorem*}

\begin{proof}
From \pref{alg:online_known}, we know that dataset $\Dcal_{0:H-1}$ satisfies \pref{assum:data} and for each $\pi\in\Pi$, $\hatdpi_{h}$ is estimated in the same way as that in \pref{alg:offline_known}. Therefore, we can follow the same steps as the proof of \pref{thm:offline_d_known}. By setting $\Bx_h = \dspanner$ and $\Ba_h = K$ for all $h \in [H]$ in \cref{eq:offline_hattobar_unfold_sub}, with probability at least $1-\delta$, for any policy $\pi\in\Pi$, we get that
\begin{align}
    \nbr{ \wh{d}_{h}^\pi - \ol{d}_{h}^\pi }_1 
    \le 18h\dspanner^{3/2} K \sqrt{\frac{\log (16H \munorm n_{\mle}/\delta)}{\nmle}} + 666h\dspanner^{3/2} K\sqrt{\frac{\log\rbr{2|\Pi|Hn_{\reg}/\delta} }{ n_{\reg} }}.
    \label{eq:online_hattobar3}
\end{align}  

The primary difference between the above results and the corresponding statements in \pref{thm:offline_d_known} is that the regression error in \cref{eq:online_hattobar3} includes an additional union bound over all $\pi \in \Pi$. This is because \pref{alg:online_known} performs estimation for all policies, while \pref{alg:offline_known} only concerns a single fixed policy. We note that this change in the proof occurs only through application of \pref{lem:conc}, which is stated generally and already includes a union bound over all policies of interest. Because MLE estimation occurs only for the data distribution and is policy-agnostic, the MLE error (second term) does not require such a union bound.  

Next, to bound the missingness error, from \pref{lem:missingness_decomposition}, we have
\begin{equation}\label{eq:online_bartotrue2}
     \nbr{ \ol{d}_{h}^\pi - d_h^\pi }_1 \le \nbr{\ol{d}_{h-1}^\pi - d_{h-1}^\pi}_1  + 4\dspanner\max_{\pi' \in \Pi} \nbr{\wh{d}_{h-1}^{\pi'} - \ol{d}^{\pi'}_{h-1}}_1.
\end{equation}

Unfolding \cref{eq:online_bartotrue2} yields
\begin{equation}\label{eq:online_bartotrue3}
    \nbr{ \ol{d}_{h}^\pi - d_h^\pi }_1 \le 4\dspanner \sum_{h' = 0}^{h-1} \max_{\pi' \in \Pi}\nbr{\wh{d}_{h'}^{\pi'} - \ol{d}^{\pi'}_{h'}}_1.
\end{equation}

Plugging the bound for $\nbr{\wh{d}_{h'}^{\pi'} - \ol{d}^{\pi'}_{h'}}_1$ from \cref{eq:online_hattobar3} into \cref{eq:online_bartotrue3} gives
\begin{equation}\label{eq:online_bartotrue}
    \nbr{ \ol{d}_{h}^\pi - d_h^\pi }_1 \le 72h^2\dspanner^{3/2} K \sqrt{\frac{\log (16H \munorm n_{\mle}/\delta)}{\nmle}} + 2664h^2\dspanner^{5/2} K\sqrt{\frac{\log\rbr{2|\Pi|Hn_{\reg}/\delta} }{ n_{\reg} }}.
\end{equation}

Combining \cref{eq:online_hattobar3} and \cref{eq:online_bartotrue} via triangle inequality and simplifying, we have 
\[
    \nbr{ \wh{d}_{h}^\pi - d_h^\pi }_1 \le 90h^2\dspanner^{3/2} K \sqrt{\frac{\log (16H \munorm n_{\mle}/\delta)}{\nmle}} + 3330h^2\dspanner^{5/2} K\sqrt{\frac{\log\rbr{2|\Pi|Hn_{\reg}/\delta} }{ n_{\reg} }}. 
\]

Finally, noticing that $\nmle = \wt O\rbr{\frac{\dlr^3 K^2H^4\log(1/\delta)}{\veps^2}}, \nreg=\wt O\rbr{\frac{\dlr^{5} K^2H^4\log(|\Pi|/\delta)}{\veps^2}},n=\nmle + \nreg$ completes the proof.
\end{proof}

\subsection{Proof of online policy optimization}
\label{app:online_rf}

First, we prove \pref{prop:density2return}, from which our online policy optimization guarantee (\pref{thm:online_rf}) follows when combined with \pref{thm:online_d}. 

\begin{proposition} [Restatement of \pref{prop:density2return}]
Given any policy $\pi$ and reward function\footnote{We assume known \& deterministic rewards, and can easily handle unknown/stochastic versions (\pref{app:reward}).} $R = \{R_h\}$ with 
$R_h:\Xcal\times\Acal\rightarrow[0,1]$, 
define expected return as 
$v^\pi_R := \EE_\pi[\sum_{h=0}^{H-1} R_h(x_h,a_h)]=\sum_{h=0}^{H-1} \iint d_h^\pi(x_h) R_h(x_h,a_h) $ $\pi(a_h|x_h) (\dd x_h) (\dd a_h). $
Then for  $\{\wh{d}_h^\pi\}$ such that $\|\wh{d}_h^\pi - d_h^\pi\|_1 \le \veps/(2H)$ for all $\pi \in \Pi$ and $h \in [H]$, and policy chosen as 
\[
    \wh\pi_R = \argmax_{\pi \in \Pi}\wh{v}_R^\pi,
\]
we have 
\[
    v^{\wh\pi_R}_R \ge \max_{\pi \in \Pi}v^\pi_R - \veps,
\]  
where $\wh{v}_R^{\pi}=\sum_{h=0}^{H-1} \iint \wh d_h^\pi(x_h) R_h(x_h,a_h) \pi(a_h|x_h) (\dd x_h) (\dd a_h)$ is the expected return calculated using $\{\wh{d}_h^\pi\}$.  
\end{proposition}
\begin{proof}
    Since the $R$ is bounded on $[0,1]$, for any $\pi \in \Pi$ we have 
    \begin{align*}
        |\wh v^\pi_R - v^\pi_R| = &~\sum_{h=0}^{H-1} \iint (\wh d_h^\pi(x_h)- d_h^\pi(x_h)) R(x_h,a_h) \pi(a_h|x_h) (\dd x_h) (\dd a_h)
        \\
        \le &~ \sum_{h=0}^{H-1} \int |\wh d_h^\pi(x_h)-  d_h^\pi(x_h)|\rbr{\int \pi(a_h|x_h)  (\dd a_h)} (\dd x_h)
        \\
        =&~ \sum_{h=0}^{H-1} \| d_h^\pi - \wh{d}_h^\pi \|_1 \le \veps / 2. 
    \end{align*}  
    Next, recall we pick $\wh\pi_R =\argmax_{\pi\in\Pi} \wh v^\pi_R$, and denote $\pi_R^* = \argmax_{\pi \in \Pi} \wh v_R^\pi$. Then using the above inequality, we have 
    \begin{align*}
        v^{\wh \pi_R}_R - \max_{\pi\in\Pi} v^\pi_R = v^{\wh \pi_R}_R - v^{\pi_R^*}_R = v^{\wh \pi_R}_R - \wh v^{\wh \pi_R}_R + \wh v^{\wh \pi_R}_R - \wh v^{\pi_R^*}_R + \wh v^{\pi_R^*}_R - v^{\pi_R^*}_R \ge -\veps 
    \end{align*}
    since $\wh v^{\wh \pi_R}_R \ge \wh v^{\pi_R^*}_R$, completing the proof. 
\end{proof}

\begin{theorem}[Restatement of \pref{thm:online_rf}] 

Fix $\delta \in (0, 1)$ and suppose \pref{assum:lowrank} and \pref{assum:data} hold, and $\mutrue$ is known. Given a policy class $\Pi$, let $\{\wh d_h^\pi\}_{h\in[H],\pi \in \Pi}$ be the output of running \onalg. Then with probability 
at least $1-\delta$, for any reward function $R$ 
and policy selected as   
$ \wh\pi_R = \argmax_{\pi\in \Pi} \wh{v}_R^\pi, $
we have
\[
    v_R^{\wh\pi_R} \ge \argmax_{\pi \in \Pi} v_R^\pi - \veps,
\]
where $v_R^\pi$ and $\wh{v}_R^\pi$ are defined in \pref{prop:density2return}. 
The total number of episodes required by the algorithm is
\[\tilde{O}\rbr{\frac{\dspanner^5 K^2 H^7 \log(|\Pi|/\delta)}{\veps^2}}.
\]
\end{theorem}
\begin{proof}
    The proof takes similar steps as the proof of \pref{thm:offline_rf}. From \pref{thm:online_d}, w.p. $\ge 1-\delta$, we obtain estimates $\{\wh d_h^\pi\}$ such that $\|d_h^\pi - \wh d_h^\pi\|_1 \le \frac{\veps}{2H}$ for all $\pi \in \Pi$ with $\tilde{O}\rbr{\frac{\dspanner^5 K^2 H^7 \log(|\Pi|/\delta)}{\veps^2}}$ total number of samples, where we use the union bound over $\pi\in\Pi$. Combining this with \pref{prop:density2return} gives the result. 
\end{proof}

\section{Representation learning}
\label{app:repr}
In this section, we present the detailed algorithms and results for the representation learning setting (\pref{sec:repr}), where the true density features are not given but must also be learned from an exponentially large candidate feature set. The algorithms and analyses mostly follow that of the known density feature case (\pref{sec:offline} and \pref{sec:online}), therefore, we mainly discuss the difference here.

\subsection{Off-policy occupancy estimation}

We start with describing our algorithm \offrepralg (\pref{alg:offline_unknown}), which estimates the occupancy distribution $d_h^\pi$ of any given policy $\pi$ using an offline dataset $\Dcal_{0:H-1}$ when the true density feature $\mu^*$ is unknown and the learner is given a realizable density feature class $\Upsilon \ni \mutrue$ (see \pref{assum:realizability}). 

As discussed in \pref{sec:repr}, instead of using $\mu^*$ to construct the function classes, a natural choice here is to use the union of all linear function classes. Since now the feature comes from candidate feature classes $\Upsilon_{h-2},\Upsilon_{h-1}$, in \pref{line:mle_off_unknown} of \pref{alg:offline_unknown}, we use different function classes $\Fcal_{h-1}(\Upsilon_{h-2}),\Fcal_{h}(\Upsilon_{h-1})$ as defined in \cref{eq:F_unknown} for the MLE objective. In addition, in \pref{line:reg_off_unknown} of \pref{alg:offline_unknown}, now we run regression with a different function class $\Wclip_{h}(\Upsilon_{h-1})$ as defined in \cref{eq:wclip_unknown}.

\begin{algorithm*}[ht!]
\caption{\textbf{F}itted \textbf{O}ccupancy Ite\textbf{r}ation with \textbf{C}lipping and Representation Learning (\offrepralg)
\label{alg:offline_unknown}}
\begin{algorithmic}[1]
\REQUIRE \hspace{-.3em}policy $\pi$, density feature class $\Upsilon$, dataset $\Dcal_{0:H-1}$, sample sizes $\nmle$ and $\nreg$, 
clipping thresholds \hspace{-.1em}$\{\Bx_{h}\}$ \hspace{-.1em}and \hspace{-.1em}$\{\Ba_{h}\}$. \hspace{-1em} 
\STATE Initialize $\wh d_0^\pi= \initdist,~\forall \pi \in \Pi$.
\FOR {$h=1,\ldots,H$}
\STATE Randomly split $\Dcal_{h-1}$ to two folds $\Dcal_{h-1}^\mle$ and $\Dcal_{h-1}^\reg$ with sizes $\nmle$ and $\nreg$ respectively.
\STATE Estimate marginal data distributions $\hatd_{h-1}(x_{h-1})$ and $\hatdnext_{h-1}(x_{h})$ by  MLE 
with dataset $\Dcal_{h-1}^\mle$. \vspace{-.5em}\label{line:mle_off_unknown}
\begin{align*}
    \hatd_{h-1}= \argmax_{d_{h-1} \in \Fcal_{h-1}(\Upsilon_{h-2})} \frac{1}{\nmle} \sum_{i=1}^{\nmle}  \log \rbr{d_{h-1}(x_{h-1}^{(i)})} \text{ and }
    \hatdnext_{h-1} =\argmax_{d_{h}\in\Fcal_{h}(\Upsilon_{h-1}) } \frac{1}{\nmle} \sum_{i=1}^{\nmle}  \log\rbr{d_{h}(x_{h}^{(i)})}
\end{align*}
where 
\begin{align}
\label{eq:F_unknown}
\Fcal_h(\Upsilon_{h-1}) = \cbr{d_h = \langle \mu_{h-1}, \theta_h \rangle : d_h \in \Delta(\Xcal), \mu_{h-1}\in\Upsilon_{h-1},\theta_h \in \RR^{\dlr},\|\theta_h\|_\infty \le 1 }. 
\end{align}\vspace{-1em}
\STATE  
Define $\Lcal_{\Dcal_{h-1}^{\reg}}(w_{h},w_{h-1},\ol\pi_{h-1}) \defeq \frac{1}{\nreg} \sum_{i=1}^{\nreg}  \rbr{ w_{h}(x_{h}^{(i)}) - w_{h-1}(x_{h-1}^{(i)}) \frac{\ol \pi_{h-1}(a_{h-1}^{(i)}|x_{h-1}^{(i)})}{\pi^D_{h-1}(a_{h-1}^{(i)}|x_{h-1}^{(i)})} }^2$ and estimate\vspace{-.5em}
\label{line:reg_off_unknown}
\begin{align*}
\hatwpi_{h} = \argmin_{w_{h} \in \Wclip_{h}(\Upsilon_{h-1})} \Lcal_{\Dcal_{h-1}^{\reg}}\rbr{w_{h},\frac{\cliphatdpi{h-1}{\hatd}}{\wh d^D_{h-1}}, \pi_{h-1} \wedge \Ba_{h-1} \pi_{h-1}^D} 
\end{align*} 
where\vspace{-.5em}
\begin{align}
\label{eq:wclip_unknown}
\Wclip_{h}(\Upsilon_{h-1}) = \cbr{w_{h} = \frac{\langle \mu_{h-1}, \thetaup_{h}\rangle}{\langle \mu_{h-1}, \thetadown_{h}\rangle} :\nbr{w_{h}}_\infty \le  \Bx_{h-1}\Ba_{h-1},\mu_{h-1}\in\Upsilon_{h-1},\thetaup_{h},\thetadown_{h} \in \RR^{\dlr}}. 
\end{align}\vspace{-.5em}
\STATE Set the estimate $\wh d_{h}^\pi= \hatwpi_{h} \, \hatdnext_{h-1} $. \label{line:multiply_unknown}
\ENDFOR
\ENSURE estimated state occupancies $\{\wh d_{h}^\pi\}_{h\in[H]}$.
\end{algorithmic}
\end{algorithm*}

Similar as in the known feature case counterpart (\pref{thm:offline_d_known}), we have the following guarantee for estimating $d^\pi$.

\begin{theorem*}[Restatement of \pref{thm:offline_d_unknown}]
Fix $\delta\in(0,1)$. Suppose \pref{assum:lowrank}, \pref{assum:data}, and \pref{assum:realizability} hold. Then, given an evaluation policy $\pi$, by setting 
\[\nmle = \tilde{O}\rbr{\dspanner \rbr{\sum_{h\in[H]} \Bx_h \Ba_{h} }^2 \log(|\Upsilon|/\delta)/\veps^2} \text{ and } \nreg = \tilde{O}\rbr{ \dspanner \rbr{\sum_{h\in[H]} \Bx_h \Ba_{h} }^2 \log(|\Upsilon|/\delta)/\veps^2 },
\]
with probability at least $1-\delta$, \offrepralg (\pref{alg:offline_unknown}) returns state occupancy estimates $\{\wh d^\pi_h\}_{h=0}^{H-1}$ satisfying that
\[\nbr{\wh{d}_{h}^\pi - \ol d_h^\pi}_1 \le \veps,\forall h\in[H].\] 
The total number of episodes required by the algorithm is 
\[
\tilde{O}\rbr{\dspanner H \rbr{\sum_{h \in [H]} \Bx_h \Ba_{h} }^2\log(|\Upsilon|/\delta)/\veps^2}.
\]
\end{theorem*}

\begin{proof}
The proof for this theorem largely follows its counterpart for the known feature case (\pref{thm:offline_d_known}), and we mainly discuss the different steps here. We now make the following two slightly different claims on MLE estimation and error propagation. Based on them, the final error bound is obtained in the same way as \pref{thm:offline_d_known}.
\paragraph{Claim 1}
Our estimated data distributions satisfy that with probability $1-\delta/2$, for any $h\in[H]$
\begin{align}
    \nbr{ \hatd_h - d_h^D }_1 \le \emle \,\text{ and }\, \nbr{ \hatdnext_h - \dnext_h }_1 \le \emle,
\label{eq:emle_bound_unknown}
\end{align}
where 
\[
\emle:=6 \sqrt{\frac{\dlr \log (16H|\Upsilon| \munorm n_{\mle}/\delta)}{\nmle}}. 
\]

\paragraph{Claim 2} Under the high-probability event that \cref{eq:emle_bound_unknown} holds, we further have with probability at least $1-\delta/2$, for any $1\le h\le H$, we have
\begin{align*}
    \nbr{ \wh{d}_{h}^\pi - \ol{d}_{h}^\pi }_1 
    \le \nbr{ \wh{d}_{h-1}^\pi - \ol{d}_{h-1}^\pi }_1 + 3\Bx_{h-1}\Ba_{h-1} \evnext + \sqrt{2}\ereg{h-1},
\end{align*} 
where 
\begin{align}
    \ereg{h-1} \defeq \sqrt{\frac{ 221184 \dspanner (\Bx_{h-1} \Ba_{h-1})^2\log\rbr{2H|\Upsilon|n_{\reg}/\delta} }{ n_{\reg} }}.
    \label{eq:hatbar4_unknown}
\end{align}

\paragraph{Proof of Claim 1}
Notice that for the term $\emle$ in \cref{eq:emle_bound_unknown}, we now have an additional $|\Upsilon|$ factor inside the $\log$. The reason is that here we use $\Fcal_{h-1}(\Upsilon_{h-2}),\Fcal_{h}(\Upsilon_{h-1})$ instead of $\Fcal_{h-1},\Fcal_{h}$. By \pref{lem:opt_cover}, the two function classes considered here have $\ell_1$ optimistic covers with scale $1/\nmle$ of size $|\Upsilon|\rbr{2\lceil \munorm n_{\mle} \rceil}^{\dlr}$.  
In addition, we still have that $d_{h-1}^D\in\Fcal_{h-1}(\Upsilon_{h-2}), \dnext_{h-1}\in\Fcal_{h}(\Upsilon_{h-1})$ from \pref{lem:mle_realizability_unknown}, and any $d_{h-1}\in\Fcal_{h-1}(\Upsilon_{h-2}),\Fcal_{h}(\Upsilon_{h-1})$ is a valid probability distribution over $\Xcal$. 

\paragraph{Proof of Claim 2}
This proof mostly follows the proof of Claim 2 in \pref{thm:offline_d_known}.
The difference is that the function class $\Wcal_h(\Upsilon_{h-1})$ now consists of all features in $\Upsilon_{h-1}$ instead of only the true feature $\mu_{h-1}^*$. 
Therefore, in \cref{eq:hatbar4_unknown}, the term $\ereg{h-1}$ has an additional $|\Upsilon|$ inside the $\log$, which is from the counterpart of \cref{eq:hatbar4}. 
It is also easy to see that $\frac{\opp^{\ol\pi}_{h-1} \rbr{d^D_{h-1}  \wt{w}_{h-1}}}{\dnext_{h-1}}\in\Wclip_{h}(\Upsilon_{h-1})$ by following the same logic before. Further noticing that $\mu_{h-1}^*\in\Upsilon_{h-1}$, we again have \cref{eq:hatbar5} holds here.
\end{proof}

\begin{theorem}[Offline policy optimization with representation learning]
\label{thm:offline_rf_unknown}
Fix $\delta \in (0, 1)$ and suppose \pref{assum:lowrank}, \pref{assum:data}, and \pref{assum:realizability} hold. Given a policy class $\Pi$, let $\{\wh d_h^\pi\}_{h\in[H],\pi \in \Pi}$ be the output of running \pref{alg:offline_unknown}. Then with probability at least $1-\delta$, for any deterministic reward function $R$ and policy selected as   
$ \wh\pi_R = \argmax_{\pi\in \Pi} \wh{v}_R^\pi, $
we have 
\[
    v_R^{\wh\pi_R} \ge \argmax_{\pi \in \Pi} \ol{v}_R^\pi - \veps,
\]
where $v_R^\pi$ and $\wh{v}_R^\pi$ are defined in \pref{prop:density2return}, and $\ol{v}_R$ is defined similarly for $\{\ol{d}_h^\pi\}$. 
The total number of episodes required by the algorithm is
\[\tilde{O}\rbr{\dspanner H^3 \rbr{\sum_{h \in [H]} \Bx_h \Ba_{h}}^2\log(|\Pi||\Upsilon|/\delta)/\veps^2}.
\]
Additionally, define the set of policies fully covered by the data to be 
    \begin{align*}
        \Pi^{\covered} = \cbr{ \pi \in \Pi : d_h^\pi = \ol{d}_h^\pi, \forall h \in [H] }. 
    \end{align*}
    Then with the same total number of episodes required by the algorithm, for any reward function $R$ and policy selected as $ \wh\pi_R = \argmax_{\pi\in \Pi^{\covered}} \wh{v}_R^\pi,$  with probability at least $1-\delta$, we have
    \[v_R^{\wh\pi_R} \ge \argmax_{\pi \in \Pi^{\covered} } v_R^\pi - \veps.\] 
\end{theorem}
\begin{proof}
The proof follows the same steps as that of \pref{thm:offline_rf}. Notice that now we will apply \pref{thm:offline_d_unknown} rather than \pref{thm:offline_d_known} to get the bound $\| \wh{d}_{h}^\pi - \ol{d}_{h}^\pi \|_1 $, which leads to the additional $\log(|\Upsilon|)$ factor.
\end{proof}

\subsection{Online policy cover construction}
Now we present the algorithm \onrepralg (\pref{alg:online_unknown}), which estimates the occupancy distribution $d_h^\pi$ of any given policy $\pi$ with the access of online interaction. Again the true density feature $\mu^*$ is unknown and the learner is given a realizable density feature class $\Upsilon$ ($\mu^*\in\Upsilon)$. 

Similar as the know feature case online algorithm (\pref{alg:online_known}), we use the offline algorithm (\pref{alg:offline_unknown}) as a submodule. However, as discussed in the main text, the crucial different step is to select a representation $\wh \mu_{h-1}$ in \cref{eq:rep_select_unknown} in  \pref{line:feasle_unknown} before setting $\wt d^{\pi}_{h}$. This guarantee the cardinality of the barycentric spanner is at most $\dlr$. Then the state occupancy $\wt d^{\pi}_{h}$ is set as the linear estimate using $\wh \mu_{h-1}$ (rather than using $\mutrue_{h-1}$ in the known feature case) in \pref{line:line_approx_unknown}.

\begin{algorithm*}[ht!]
\caption{\onrepralglong (\onrepralg) \label{alg:online_unknown}}
\begin{algorithmic}[1]
\REQUIRE policy class $\Pi$, density feature class $\Upsilon$, $n = n_{\mle} + n_{\reg}$
\STATE Initialize $\wh{d}_0^\pi=\initdist$ and $\wt d_0^\pi = d_0,~\forall \pi\in\Pi$. 
\FOR{$h=1,\ldots,H$} 
\STATE Construct $\{\dlin_{h-1}^{\pi^{h-1,i}}\}_{i=1}^{\dlr}$ as the barycentric spanner of 
$\{\dlinpi_{h-1}\}_{\pi \in \Pi}$, and set $\Piexpl_{h-1} = \{\pi^{h-1,i}\}_{i=1}^{\dlr}$. \label{line:spanner_unknown}
\STATE Draw a tuple dataset $\Dcal_{h-1} = \{(x_{h-1}^{(i)}, a_{h-1}^{(i)}, x_{h}^{(i)})\}_{i=1}^{n}$ using $\unif(\Piexpl_{h-1}) \circ \unif(\Acal)$. \label{line:reg_on_start_unknown}
\FOR{$\pi\in\Pi$}
\STATE Estimate $\wh{d}_{h}^\pi$ using the $h$-level loop\footnotemark of \pref{alg:offline_unknown} (lines \ref{line:mle_off_unknown}-\ref{line:multiply_unknown}) with $\Dcal_h$, $\wh d_{h-1}^\pi$, $\Bx_h = \dlr$, $\Ba_h = K$. 
\ENDFOR \label{line:reg_on_end_unknown}
\STATE Select feature $\wh \mu_{h-1}$ according to \label{line:feasle_unknown} 
\begin{align}
\label{eq:rep_select_unknown}
\wh \mu_{h-1} = \min_{\mu_{h-1}\in\Upsilon_{h-1}}\max_{\pi\in\Pi}\min_{\theta_{h}\in\RR^{\dlr}} \|\langle \mu_{h-1}, \theta_{h} \rangle - \hatdpi_{h}\|_1.
\end{align}
\STATE For all $\pi\in\Pi$, set the closest linear approximation  to $\hatdpi_{h}$ with feature $\wh \mu_{h-1}$ as $\dlinpi_{h} = \langle \wh \mu_{h-1}, \thetalin_{h} \rangle$, where $\thetalin_{h} = \argmin_{\theta_{h} \in \RR^{\dlr}} \|\langle \wh \mu_{h-1}, \theta_{h} \rangle - \hatdpi_{h}\|_1$. 
\label{line:line_approx_unknown}
\ENDFOR
\ENSURE estimated state occupancy measure $\{\wh d_{h}^\pi\}_{h\in[H],\pi \in \Pi}$.
\end{algorithmic}
\end{algorithm*}

Similar as in the known feature case counterpart (\pref{thm:online_d}), we have the following guarantee for estimating $d^\pi$.

\begin{theorem*}[Restatement of \pref{thm:online_d_unknown}]
Fix $\delta\in(0,1)$ and suppose \pref{assum:lowrank} and \pref{assum:realizability} hold. Then by setting 
\[\nmle = \wt O\rbr{\frac{\dlr^3 K^2 H^4\log(|\Upsilon|/\delta)}{\veps^2}}, \nreg=\wt O\rbr{\frac{\dlr^{5} K^2 H^4\log(|\Pi||\Upsilon|/\delta)}{\veps^2}},n=\nmle + \nreg,
\]
with probability at least $1-\delta$, \onrepralg (\pref{alg:online_unknown}) returns state occupancy estimates $\{\wh d^\pi_h\}_{h=0}^{H-1}$ satisfying that
\begin{align*}
    \|\wh{d}_{h}^\pi - d_h^\pi\|_1 \le \veps, \forall h\in[H],\pi\in\Pi.
\end{align*}
The total number of episodes required by the algorithm is 
\[
\wt O(nH)=\wt O\rbr{\frac{\dlr^{5} K^2 H^5\log(|\Pi||\Upsilon|/\delta)}{\veps^2}}.
\]
\end{theorem*}

\begin{proof}
The proof for this theorem largely follows its counterpart for the known feature case (\pref{thm:online_d}), and we only discuss the different steps here. 

Firstly, \pref{lem:missingness_decomposition} still holds. However, since we use ``joint linearization'' in 
\pref{line:feasle_unknown} and \pref{line:line_approx_unknown}, we need to modify the proof of \cref{eq:diff_repr} as the following. Again, we have $\ol{d}^{\pi'}_{h-1} = \opp^{\pi'}_{h-2} (\ol{d}_{h-2}^\pi \wedge \Bx_{h-2} d^D_{h-2})= \opp^{\pi'}_{h-2} ( \ol{d}_{h-2}^\pi \wedge \dspanner d^D_{h-2})$ is linear in the true feature $\mu_{h-2}^*$ (\pref{lem:opp_linear}). Together with the feature selection criteria \cref{eq:rep_select_unknown}, we have that
\begin{align*}
&~\max_{\pi'\in\Pi} \| \wt d_{h-1}^{\pi'} - \wh{d}^{\pi'}_{h-1}\|_1 = \max_{\pi'\in\Pi}\min_{ \theta_{h-1}\in\RR^d} \|\langle \wh \mu_{h-2}, \theta_{h-1} \rangle - \wh{d}^{\pi'}_{h-1}\|_1 
\\
\le&~ \max_{\pi'\in\Pi}\min_{ \theta_{h-1}\in\RR^d} \|\langle \mu_{h-2}^*, \theta_{h-1} \rangle - \wh{d}^{\pi'}_{h-1}\|_1 \le \max_{\pi'\in\Pi}  \| \ol{d}^{\pi'}_{h-1} - \wh{d}^{\pi'}_{h-1}\|_1.
\end{align*}

For \cref{eq:online_hattobar3}, we will have an additional $|\Upsilon|$ factor inside the $\log$ as 
\[
\emle:=6 \sqrt{\frac{\dlr \log (16H |\Upsilon| \munorm \nmle/\delta)}{\nmle}}. 
\]
The reason is that here we use $\Fcal_{h-1}(\Upsilon_{h-2}),\Fcal_{h}(\Upsilon_{h-1})$ instead of $\Fcal_{h-1},\Fcal_{h}$. By \pref{lem:opt_cover}, the two function classes considered here have $\ell_1$ optimistic covers with scale $1/\nmle$ of size $|\Upsilon|\rbr{2\lceil \munorm n_{\mle} \rceil}^{\dlr}$. 
In addition, we still have that $d_{h-1}^D\in\Fcal_{h-1}(\Upsilon_{h-2}), \dnext_{h-1}\in\Fcal_{h}(\Upsilon_{h-1})$ \pref{lem:mle_realizability_unknown}, and any $d_{h-1}\in\Fcal_{h-1}(\Upsilon_{h-2}),\Fcal_{h}(\Upsilon_{h-1})$ is a valid probability distribution over $\Xcal$. 

The remaining part of the proof is the same as that of \pref{thm:online_d}.
\end{proof}

\begin{theorem}[Online policy optimization with representation learning]
\label{thm:online_rf_unknown}

Fix $\delta \in (0, 1)$ and suppose \pref{assum:lowrank} and \pref{assum:realizability} hold. Given a policy class $\Pi$, let $\{\wh d_h^\pi\}_{h\in[H],\pi \in \Pi}$ be the output of running \pref{alg:online_unknown}. Then with probability at least $1-\delta$, for any deterministic reward function $R$ (as per \pref{prop:density2return}) and policy selected as   
$ \wh\pi_R = \argmax_{\pi\in \Pi} \wh{v}_R^\pi, $
we have 
\[
    v_R^{\wh\pi_R} \ge \argmax_{\pi \in \Pi} v_R^\pi - \veps,
\]
where $\wh{v}_R^\pi \defeq \sum_{h=0}^{H-1} \iint \wh d_h^\pi(x_h) R(x_h,a_h) \pi(a_h|x_h) (\dd x_h) (\dd a_h)$.
The total number of episodes required by the algorithm is
\[\tilde{O}\rbr{\frac{\dspanner^5 K^2 H^7 \log(|\Pi||\Upsilon|/\delta)}{\veps^2}}.
\]
\end{theorem}

\begin{proof}
The proof follows the same steps as that of \pref{thm:online_rf}. Notice that now we will apply \pref{thm:online_d_unknown} rather than \pref{thm:online_d} to get the bound $\| d_h^\pi - \wh{d}_h^\pi\|$, which leads to the additional $\log(|\Pi|)$ factor.
\end{proof}

\section{Maximum likelihood estimation}\label{app:mle}
In this section, we adapt the standard i.i.d. results of maximum likelihood estimation \citep{van2000empirical} to our setting, and in particular, to our (infinite) linear function class. We consider the problem of estimating a probability distribution over the instance space $\Xcal$, and note that we abuse some notations (e.g., $n,\Lcal,\Dcal,\Fcal$) in this section, as they have different meanings in other parts of the paper. Given an i.i.d. sampled dataset $\Dcal=\{x^{(i)}\}_{i=1}^{n}$ and a function class $\Fcal$, we optimize the MLE objective
\begin{align}
\label{eq:mle_appx}
\wh f = \argmin_{f\in\Fcal} \frac{1}{n} \sum_{i=1}^n \log\rbr{f(x^{(i)})}.
\end{align}

We consider the function class $\Fcal$ to be infinite, and as is common in statistical learning, our result will depends on its structural complexity. In particular, this will be quantified using the $\ell_1$ optimistic cover, defined below: 

\begin{definition}[$\ell_1$ optimistic cover]
\label{def:opt_cover}
For a function class $\Fcal\subseteq (\Xcal \rightarrow \RR)$, we call function class $\ol \Fcal$ an $\ell_\infty$ optimistic cover of $\Fcal$ with scale $\gamma$, if for any $f\in\Fcal$ there exists $\ol f \in \ol \Fcal$, such that $\|f-\ol f\|_1 \le \gamma$ and $f(x) \le \ol f(x),\, \forall x\in\Xcal$. Notice that here we do not require the cover to be proper, i.e., we allow $\ol \Fcal \not\subseteq \Fcal$.
\end{definition}

Now we are ready to state the MLE guarantee formally.

\begin{lemma}[MLE guarantee] 
\label{lem:mle}
Let $\Dcal=\{x^{(i)}\}_{i=1}^{n}$ be a dataset, where $x^{(i)}$ are drawn i.i.d. from some fixed probability distribution $f^*$ over $\Xcal$. 
Consider a function class $\Fcal$ that satisfies: (i) $f^*\in\Fcal$, (ii) each function $f\in\Fcal$ is a valid probability distribution over $\Xcal$ (i.e., $f\in\Delta(\Xcal)$), and (iii) $\Fcal$ has a finite $\ell_1$ optimistic cover (\pref{def:opt_cover}) $\ol \Fcal$ with scale $\gamma$ and $\ol\Fcal \subseteq (\Xcal \rightarrow \RR_{\ge 0})$. Then with probability at least $1-\delta$, the MLE solution $\wh f$ in \cref{eq:mle_appx} has an $\ell_1$ error guarantee
\begin{align*}
    \|\wh{f} - f^*\|_1 \le \gamma + \sqrt{\frac{12\log (|\ol\Fcal|/\delta)}{n} + 6\gamma}.
\end{align*}
\end{lemma}

\begin{proof}
Our proof is based on \citet{zhang2006e,agarwal2020flambe, liu2022partially} and is simpler since we assume the $\Dcal$ here is drawn i.i.d. instead of adaptively.  We first define $\Lcal(f,\Dcal)=\frac{1}{2}\sum_{i=1}^n \log \rbr{\frac{f(x^{(i)})}{f^*(x^{(i)})}}$. By Chernoff's method, for a fixed $f \in \ol \Fcal$ we have that 
\begin{align*}
    &~\PP\rbr{ \Lcal(f, \Dcal) - \log (\EE_{\Dcal}[\exp(\Lcal(f, \Dcal))]) \ge \log (|\ol\Fcal|/\delta)} \\
    \le&~ \exp(-\log(|\ol\Fcal|/\delta))\EE_{\Dcal}\sbr{ \exp\rbr{ \Lcal(f, \Dcal) - \log (\EE_{\Dcal}[\exp(\Lcal(f, \Dcal))]) } } \\
    =&~ \delta/|\ol\Fcal|.
\end{align*}

Union bounding over $ f \in \ol{\Fcal}$, with probability at least $1-\delta$, for any $f \in \ol{\Fcal}$ we have
\begin{equation}
\label{eq:mle1}
    - \log (\EE_{\Dcal}[\exp(\Lcal( f, \Dcal))]) \le - \Lcal(f, \Dcal) + \log(|\ol\Fcal|/\delta).
\end{equation}

Let $\ol f\in\ol\Fcal$ be the $\gamma$-close $\ell_1$ optimistic approximator of the MLE solution $\wh{f}\in\Fcal$. Since $\ol f(x) \ge \wh{f}(x), \,\forall x\in\Xcal$ due to the optimistic covering construction and $\wh{f}$ is the MLE estimator, for the RHS of \cref{eq:mle1}. we have  
\begin{align*}
    - \Lcal(\ol f, \Dcal) = \frac{1}{2} \sum_{i=1}^n \log \rbr{\frac{f^*(x^{(i)})}{ \ol f(x^{(i)})}} \le  \frac{1}{2} \sum_{i=1}^n \log \rbr{\frac{f^*(x^{(i)})}{ \wh{f}(x^{(i)})}}=  \frac{1}{2} \rbr{\sum_{i=1}^n \log (f^*(x^{(i)})) - \sum_{i=1}^n \log (\wh f(x^{(i)}))} \le 0. 
\end{align*}

Next, consider the LHS of \cref{eq:mle1}. From the definition of dataset $\Dcal$ and $\Lcal(\ol f,\Dcal)$, we get  
\begin{align*}
    &~- \log (\EE_{\Dcal}[\exp(\Lcal(\ol f, \Dcal))]) 
    = - \log \rbr{\EE_{\Dcal}\sbr{ \exp \rbr{ \frac{1}{2}\sum_{i=1}^n \log \rbr{\frac{\ol f(x^{(i)})}{f^*(x^{(i)})}} }}}  
    \\
    =&~ -n \log \rbr{\EE_{\Dcal} \sbr{\exp \rbr{\frac{1}{2} \log \rbr{\frac{\ol f(x)}{f^*(x)}} }}}  
    = -n \log \rbr{\EE_{\Dcal} \sbr{\sqrt{\frac{\ol f(x)}{f^*(x)}} }}.
\end{align*}

Furthermore, by $-\log (y) \ge 1 - y$, $\ell_1$ optimistic cover definition, and $f^*,\wh f$ are valid distributions over $x\in\Xcal$, we have
\begin{align*}
    &~-n \log \rbr{\EE_{\Dcal} \sbr{\sqrt{\frac{\ol f(x)}{f^*(x)}}}} 
    \ge n \rbr{ 1 - \EE_{\Dcal} \sbr{\sqrt{\frac{\ol f(x)}{f^*(x)} }}} 
    = n \rbr{ 1 - \int \sqrt{\ol f(x) f^*(x)} (\dd x)} \\ 
    =&~ \frac{n}{2} \int \rbr{ \sqrt{f^*(x)} - \sqrt{\ol f(x)}}^2 (\dd x) + \frac{n}{2} \rbr{1  - \int \ol f(x) (\dd x)}\\ 
    =&~ \frac{n}{2} \int \rbr{ \sqrt{f^*(x)} - \sqrt{\ol f(x)}}^2 (\dd x) + \frac{n}{2} \int \rbr{\wh f(x) -\ol f(x)} (\dd x) \\ 
    \ge&~  \frac{n}{2} \int  \rbr{ \sqrt{f^*(x)} - \sqrt{\ol f(x)}}^2 (\dd x) - \frac{n\gamma}{2}.
\end{align*}

Then notice that
$\int \rbr{\sqrt{f^*(x)} + \sqrt{\ol f(x)}}^2 (\dd x) \le 2 \int \rbr{f^*(x) + \ol f(x)} (\dd x) 
\le 2 \int (f^*(x) + \wh{f}(x) + |\ol f(x) - \wh{f}(x) |) (\dd x) \le 6$ and the Cauchy-Schwarz inequality, we obtain
\begin{align*}
    &~ \frac{n}{2} \int  \rbr{ \sqrt{f^*(x)} - \sqrt{\ol f(x)}}^2 (\dd x) - \frac{n\gamma}{2} 
    \\
    \ge&~ \frac{n}{12} \rbr{\int \rbr{ \sqrt{f^*(x)} - \sqrt{\ol f(x)}}^2 (\dd x)}\rbr{\int  \rbr{ \sqrt{f^*(x)} + \sqrt{\ol f(x)}}^2 (\dd x)} - \frac{n\gamma}{2} \\ 
    \ge&~ \frac{n}{12} \rbr{ \int |\ol f(x) - f^*(x)| (\dd x)}^2 - \frac{n\gamma}{2}=\frac{n}{12}\|\ol f - f^*\|_1^2 - \frac{n\gamma}{2}.
\end{align*}

Combining the above inequalities and rearranging yields
\begin{align*}
    \|\ol f - f^*\|_1^2 \le \frac{12\log (|\ol\Fcal|/\delta)}{n} + 6\gamma. 
\end{align*}

Finally, by the triangle inequality and the definition of the $\ell_1$ optimistic cover, we get
\begin{align*}
    \|\wh{f} - f^*\|_1 \le \|\wh{f} - \ol f \|_1 + \|\ol f - f^* \|_1 \le \gamma + \sqrt{\frac{12\log (|\ol\Fcal|/\delta)}{n} + 6\gamma}~,
\end{align*}
which completes the proof.
\end{proof}

\section{Auxiliary lemmas}
In this section, we provide detailed proofs for auxiliary lemmas.

\subsection{Squared loss regression results}
\begin{lemma}[Squared loss decomposition]
\label{lem:variance_decomposition}
For any $w_h, w_{h+1}: \Xcal\rightarrow \RR$, dataset $\Dcal_h^\reg = \{(x_h,a_h,x_{h+1})\} \sim d_h^D$, and a \pseudo-policy $\pi$, we have
\begin{equation}
    \nbr{w_{h+1} - \frac{\opp^{\pi}_{h} \rbr{d^D_{h} w_{h}}}{\dnext_{h}} }_{2, \dnext_{h}}^2 = \EE\sbr{\Lcal_{\Dcal_h^\reg}(w_{h+1}, w_h,\pi)} - \EE\sbr{\Lcal_{\Dcal_h^\reg}\rbr{\frac{\opp^{\pi}_{h} \rbr{d^D_{h} w_{h}}}{\dnext_{h}},w_h,\pi}}.
\end{equation}
\end{lemma}

\begin{proof}
We introduce a new notation 
\begin{align}
\label{eq:Epi}
(\opexp_h w_h)(x_{h+1})\defeq&~ \frac{\rbr{\opp^{\pi}_{h} \rbr{d^D_{h} w_{h}}}(x_{h+1})}{\dnext_{h}(x_{h+1})}\notag
\\
=&~\frac{\iint P_{h}(x_{h+1}|x_{h}, a_{h})\pi(a_{h}|x_{h}) d_{h}^D(x_{h})w_h(x_h) (\dd x_{h}) (\dd a_{h})}{\dnext_{h}(x_{h+1})},
\end{align}
which represents the conditional expectation. Then we have the decomposition
    \begin{align*}
        &~\EE\sbr{\Lcal_{\Dcal_h^\reg}(w_{h+1}, w_h,\pi)} 
        \\
        =&~ \iiint d^D_{h}(x_h,a_h,x_{h+1}) \rbr{ w_{h+1}(x_{h+1}) - \frac{\pi(a_{h}|x_{h})}{\pi^D(a_{h}|x_{h})} w_h(x_h) }^2 (\dd x_h)(\dd a_h)(\dd x_{h+1}) 
        \\ 
        =&~ \iiint d^D_{h}(x_h,a_h,x_{h+1})  \rbr{w_{h+1}(x_{h+1}) - (\opexp_h w_h)(x_{h+1}) + (\opexp_h w_h)(x_{h+1}) - \frac{\pi(a_{h}|x_{h})}{\pi^D(a_{h}|x_{h})} w_h(x_h)}^2 
        \\
        &\qquad\qquad\qquad\qquad\quad\quad\quad\quad\quad\quad\quad\quad\quad\quad\quad\quad\quad\quad\quad\quad\quad\quad\quad\quad\quad\quad\quad\quad\quad(\dd x_h)(\dd a_h)(\dd x_{h+1})
        \\ 
        =&~ \int \dnext_h(x_{h+1}) (w_{h+1}(x_{h+1}) - (\opexp_h w_h)(x_{h+1}))^2 (\dd x_{h+1})
        \\
        &\quad+ \iiint d^D_{h}(x_h,a_h,x_{h+1})  \rbr{(\opexp_h w_h)(x_{h+1}) - \frac{\pi(a_{h}|x_{h})}{\pi^D(a_{h}|x_{h})} w_h(x_h)}^2 (\dd x_h)(\dd a_h)(\dd x_{h+1})\\
        &\quad+ 2\iiint d^D_{h}(x_h,a_h,x_{h+1}) (w_{h+1}(x_{h+1}) - (\opexp_h w_h)(x_{h+1}))\rbr{(\opexp_h w_h)(x_{h+1}) - \frac{\pi(a_{h}|x_{h})}{\pi^D(a_{h}|x_{h})} w_h(x_h)} 
        \\
        &\qquad\qquad\qquad\qquad\quad\quad\quad\quad\quad\quad\quad\quad\quad\quad\quad\quad\quad\quad\quad\quad\quad\quad\quad\quad\quad\quad\quad\quad\quad(\dd x_h)(\dd a_h)(\dd x_{h+1})\\
        =&~ \|w_{h+1} - (\opexp_{h} w_{h}) \|_{2, \dnext_{h}}^2 + \EE\sbr{\Lcal_{\Dcal_h^\reg}(\opexp_{h} w_{h}, w_h,\pi)} \\
        &\quad+ 2\int \dnext_h(x_{h+1})(w_{h+1}(x_{h+1}) - (\opexp_h w_h)(x_{h+1})) (\opexp_h w_h)(x_{h+1}) (\dd x_{h+1}) \\
        &\quad - 2\int \dnext_h(x_{h+1})(w_{h+1}(x_{h+1}) - (\opexp_h w_h)(x_{h+1})) 
        \\
        &\qquad\qquad\cdot\rbr{ \iint d^D_h(x_h,a_h|x_{h+1})  \frac{\pi(a_{h}|x_{h})}{\pi^D(a_{h}|x_{h})} w_h(x_h)(\dd x_h)(\dd a_h)}(\dd x_{h+1}) \\
        =&~ \|w_{h+1} - (\opexp_{h} w_{h}) \|_{2, \dnext_{h}}^2 + \EE\sbr{\Lcal_{\Dcal_h^\reg}(\opexp_{h} w_{h}, w_h,\pi)} \\
        &\quad+ 2\int \dnext_h(x_{h+1})(w_{h+1}(x_{h+1}) - (\opexp_h w_h)(x_{h+1}))  ((\opexp_h w_h)(x_{h+1}) - (\opexp_h w_h)(x_{h+1})) (\dd x_{h+1})
        \\ 
        =&~ \|w_{h+1} - (\opexp_{h} w_{h}) \|_{2, \dnext_{h}}^2 + \EE\sbr{\Lcal_{\Dcal_h^\reg}(\opexp_{h} w_{h}, w_h,\pi)} . 
 \qedhere
    \end{align*}
\end{proof}

\begin{lemma}[Deviation bound for regression with squared loss]
\label{lem:conc}
    For $h\in[H]$, consider a dataset $\Dcal_{0:h}$ that satisfies \pref{assum:data} and a function $w_h:\Xcal \rightarrow [0,\Bx_h]$ that only depends on $\Dcal_{0:h-1}\bigcup \Dcal_h^\mle$. 
    Consider a finite feature class $\Upsilon_{h}$ and a finite policy class $\Pi'$ such that any $\pi \in \Pi'$ is a \pseudo-policy (\pref{def:pseudo_policy}) satisfying $\pi_h(a_h|x_h)\le \Ba_h \pi^D_h(a_h|x_h),\forall x_h\in\Xcal, a_h\in\Acal$. 
    Then with probability $1-\delta$,  for any $w_{h+1}\in \Wcal_{h+1}(\Upsilon_h)$ and $\pi \in \Pi'$, we have
\begin{align*}
&~\abr{\EE\sbr{\Lcal_{\Dcal_h^{\reg}}\rbr{w_{h+1},w_h,\pi} -  \Lcal_{\Dcal_h^{\reg}}(\opexp_{h} w_{h},w_h,\pi)} - \rbr{\Lcal_{\Dcal_h^{\reg}}\rbr{w_{h+1},w_h,\pi} -  \Lcal_{\Dcal_h^{\reg}}(\opexp_{h} w_{h},w_h,\pi)}}
     \\
     \le&~\frac{1}{2}\EE\sbr{\Lcal_{\Dcal_h^{\reg}}\rbr{w_{h+1},w_h,\pi} -  \Lcal_{\Dcal_h^{\reg}}(\opexp_{h} w_{h},w_h,\pi)} +  \frac{ 221184 \dspanner (\Bx_h \Ba_h)^2\log\rbr{n_{\reg}|\Pi'||\Upsilon_h|/\delta}}{ n_{\reg} }
\end{align*}
where the function class $\Wcal_{h+1}(\Upsilon_h)$ is defined in \pref{alg:offline_known} as in \cref{eq:F_unknown}
and the operator $\opexp_{h}$ is defined in \cref{eq:Epi}. 
\end{lemma}

\begin{proof}
We first fix the datasets $\Dcal_{0:h-1}\bigcup\Dcal_h^\mle$ and prove the desired bound when conditioned on these datasets, in which case $w_h, \dnext_h, \pi^D$ are fixed. In the following, the expectation $\EE$ and variance $\VV$ are w.r.t. $(x_h,a_h,x_{h+1})\sim d_h^D$, i.e., the data distribution from which the samples in $\Dcal_h^\reg$ are drawn i.i.d. from (\pref{assum:data}), when conditioned on $\Dcal_{0:h-1}\bigcup\Dcal_h^\mle$.  

Consider a single $\pi \in \Pi'$ and feature $\mu_h \in \Upsilon_h$, and consider the hypothesis class 
\begin{align*}
    \Ycal(\Wcal_{h+1}(\mu_h), w_h, \pi) = \cbr{Y(w_{h+1}, w_h, \pi) : w_{h+1} \in \Wcal_{h+1}(\mu_h)}. 
\end{align*}
where the random variable $Y(w_{h+1}, w_h,\pi)$ (suppressing the dependence on the $(x_h, a_h, x_{h+1})$ tuple) is defined for convenience as
\begin{align*}
    Y(w_{h+1}, w_h,\pi) := \rbr{w_{h+1}(x_{h+1}) - w_h(x_h)\frac{\pi(a_h|x_h)}{\pi^D(a_h|x_h)}}^2 - \rbr{(\opexp_{h} w_{h})(x_{h+1}) - w_h(x_h)\frac{\pi(a_h|x_h)}{\pi^D(a_h|x_h)}}^2,
\end{align*}
and we use $Y_i(w_{h+1}, w_h,\pi)$ to denote its realization on the $i$-th tuple data $(x_h^{(i)},a_h^{(i)},x_{h+1}^{(i)})\in \Dcal_h^\reg$. The function class $\Wcal_{h+1}(\mu_h)$ is defined as in \cref{eq:wclip_unknown}, i.e., 
\[
    \Wcal_{h+1}(\mu_h) = \cbr{w_{h+1} = \frac{\langle \mu_h, \thetaup_{h+1}\rangle}{\langle \mu_h, \thetadown_{h+1}\rangle} :\nbr{w_{h+1}}_\infty \le  \Bx_{h}\Ba_{h},\thetaup_{h+1},\thetadown_{h+1} \in \RR^{\dlr}}.
\]

It can be seen that $|Y(w_{h+1}, w_h,\pi)| \leq 4(\Bx_{h}\Ba_{h})^2$ from the following.  
From their respective definitions, we know $\|w_h\|_\infty \le \Bx_h, \|\frac{\pi}{\pi^D}\|_\infty \le \Ba_h$, and $\|w_{h+1}\|_\infty\le \Bx_{h}\Ba_{h}$. We also have $(\opexp_h w_h)(x_{h+1}) = \frac{\rbr{\opp^{\pi}_{h} \rbr{d^D_{h} w_{h}}}(x_{h+1})}{\dnext_{h}(x_{h+1})}
\in[0,\Bx_h\Ba_h]$ from \pref{lem:clipped_concentrability}. 

Further, for any $Y(w_{h+1}, w_h,\pi) \in \Ycal(\Wcal_{h+1}(\mu_h), w_h, \pi)$, we can bound the variance $\VV[Y(w_{h+1}, w_h,\pi)]$ as 
\begin{align*}
    &~\VV [Y(w_{h+1}, w_h,\pi)] \leq \EE\sbr{Y(w_{h+1}, w_h,\pi)^2} \\ 
    =&~ \EE \sbr{\rbr{\rbr{w_{h+1}(x_{h+1}) - w_h(x_h)\frac{\pi(a_h|x_h)}{\pi^D(a_h|x_h)}}^2 - \rbr{(\opexp_{h} w_{h})(x_{h+1}) - w_h(x_h)\frac{\pi(a_h|x_h)}{\pi^D(a_h|x_h)}}^2}^2} \\
    =&~ \EE \sbr{(w_{h+1}(x_{h+1}) - (\opexp_h w_h)(x_{h+1}))^2 \rbr{w_{h+1}(x_{h+1}) - 2w_h(x_h)\frac{\pi(a_h|x_h)}{\pi^D(a_h|x_h)} + (\opexp_h w_h)(x_{h+1})}^2} \\
    \leq&~ 4(\Bx_{h}\Ba_{h})^2 \EE \sbr{(w_{h+1}(x_{h+1}) - (\opexp_h w_h)(x_{h+1}))^2}
    \\ 
    =&~ 4(\Bx_{h}\Ba_{h})^2 \EE\sbr{Y(w_{h+1}, w_h,\pi)}. \tag{\pref{lem:variance_decomposition}}
\end{align*}

Next, we show that the uniform covering number $\Ncal_1(\gamma, \Ycal(\Wcal_{h+1}(\mu_h), w_h, \pi), m)$ (see \pref{def:covering_number}) for any $\gamma \in \RR, m \in \NN$ can be bounded by the covering number of $\Wcal_{h+1}(\mu_h)$.  
Let $Z^m = (x_h^{(i)}, a_h^{(i)}, x_{h+1}^{(i)})_{i=1}^m$ denote $m$ i.i.d. samples from $d_h^D$, and denote $X^m = (x_{h+1}^{(i)})_{i=1}^m$ the corresponding $x_{h+1}$ samples. For any $Z^m$ and $Y(w_{h+1}, w_h, \pi), Y(w'_{h+1}, w_h, \pi) \in \Ycal(\Wcal_{h+1}(\mu_h), w_h, \pi)$, 
\begin{align*}
    &\frac{1}{m}\sum_{i=1}^m \abr{Y_i(w_{h+1}, w_h, \pi) - Y_i(w'_{h+1}, w_h, \pi)} 
    \\
    =&~ \frac{1}{m}\sum_{i=1}^m \abr{\rbr{w_{h+1}(x^{(i)}_{h+1}) - w_h(x^{(i)}_h)\frac{\pi(a^{(i)}_h|x^{(i)}_h)}{\pi^D(a^{(i)}_h|x^{(i)}_h)}}^2 - \rbr{w'_{h+1}(x^{(i)}_{h+1}) - w_h(x^{(i)}_h)\frac{\pi(a^{(i)}_h|x^{(i)}_h)}{\pi^D(a^{(i)}_h|x^{(i)}_h)}}^2} 
    \\
    =&~ \frac{1}{m}\sum_{i=1}^m \abr{w_{h+1}(x^{(i)}_{h+1}) - 2w_h(x^{(i)}_h)\frac{\pi(a^{(i)}_h|x^{(i)}_h)}{\pi^D(a^{(i)}_h|x^{(i)}_h)} + w'_{h+1}(x^{(i)}_{h+1})} \cdot \abr{w_{h+1}(x^{(i)}_{h+1}) - w'_{h+1}(x^{(i)}_{h+1})} 
    \\
    \le&~ \frac{4\Bx_{h}\Ba_{h}}{m}\sum_{i=1}^m \abr{w_{h+1}(x^{(i)}_{h+1}) - w'_{h+1}(x^{(i)}_{h+1})} .
\end{align*}

Thus any $\gamma / (4\Bx_{h}\Ba_{h})$-covering of $\Wcal_{h+1}|_{X^m}$ in $\ell_1$ is a $\gamma$-covering of $Y(\Wcal_{h+1}, w_h, \pi)|_{Z^m}$ in $\ell_1$, and 
\[\Ncal_1(\gamma,Y(\Wcal_{h+1}(\mu_h), w_h, \pi),Z^m) \le \Ncal_1(\gamma / (4\Bx_{h}\Ba_{h}), \Wcal_{h+1}(\mu_h), X^m)
\]
which implies the same relationship for the uniform covering numbers: 
\begin{align*}
    \Ncal_1(\gamma,Y(\Wcal_{h+1}(\mu_h), w_h, \pi),m) =&~ \max_{Z^m} \Ncal_1(\gamma,Y(\Wcal_{h+1}(\mu_h), w_h, \pi),Z^m) 
    \\
    \le&~ \max_{X^m} \Ncal_1(\gamma / (4\Bx_{h}\Ba_{h}), \Wcal_{h+1}(\mu_h), X^m) = \Ncal_1(\gamma / (4\Bx_{h}\Ba_{h}), \Wcal_{h+1}(\mu_h), m).
\end{align*}

Then using this inequality and $b = 4 (\Bx_h \Ba_h)^2$ in \pref{lem:uni_bern_conf_covering} and conditioning on $\Dcal_{0:h-1} \bigcup \Dcal_h^{\mle}$, for any $w_{h+1} \in \Wcal_{h+1}(\mu_h)$, we have 
\begin{align*}
    &~\PP\rbr{\abr{\EE[Y(w_{h+1}, w_h, \pi)]-\frac{1}{n_{\reg}}\sum_{i=1}^n Y_i(w_{h+1}, w_h, \pi)} \ge \veps} 
    \\
    \le&~ 36\Ncal_1\rbr{\frac{\veps^3}{10240(\Bx_{h} \Ba_{h})^4},\Ycal(\Wcal_{h+1}(\mu_h), w_h, \pi),\frac{640n_{\reg}(\Bx_{h} \Ba_{h})^4}{\veps^2}}
    \\
    &\qquad\qquad\cdot\exp\rbr{-\frac{n_{\reg} \veps^2}{128\VV[Y(w_{h+1}, w_h, \pi)] + 2048\veps (\Bx_{h} \Ba_{h})^2}} 
    \\
    \le&~ 36\Ncal_1\rbr{\frac{\veps^3}{40960(\Bx_{h} \Ba_{h})^5},\Wcal_{h+1}(\mu_h),\frac{640n_{\reg}(\Bx_{h} \Ba_{h})^4}{\veps^2}}
    \\
    &\qquad\qquad\cdot\exp\rbr{-\frac{n_{\reg} \veps^2}{512 (\Bx_{h} \Ba_{h})^2 \EE[Y(w_{h+1}, w_h, \pi)] + 2048\veps (\Bx_{h} \Ba_{h})^2}} .
\end{align*}

Then setting the RHS equal to $\delta'$, we have 
\[
        n_{\reg} = \frac{ 512 (\Bx_{h} \Ba_{h})^2\rbr{ \EE[Y(w_{h+1}, w_h, \pi)] + 4\veps}\log\rbr{36\Ncal_1\rbr{\frac{\veps^3}{40960(\Bx_{h} \Ba_{h})^5},\Wcal_{h+1}(\mu_h),\frac{640n_{\reg}(\Bx_{h} \Ba_{h})^4}{\veps^2}}/\delta' }}{\veps^2} 
\]
implying 
\begin{align*}
    \veps \le&~ \sqrt{\frac{ 512 (\Bx_{h} \Ba_{h})^2 \EE[Y(w_{h+1}, w_h, \pi)]\log\rbr{36\Ncal_1\rbr{\frac{\veps^3}{40960(\Bx_{h} \Ba_{h})^5},\Wcal_{h+1}(\mu_h),\frac{640n_{\reg}(\Bx_{h} \Ba_{h})^4}{\veps^2}}/\delta' } }{ n_{\reg} }} 
    \\
    &\quad+ \frac{ 2048 (\Bx_{h} \Ba_{h})^2 \log\rbr{36\Ncal_1\rbr{\frac{\veps^3}{40960(\Bx_{h} \Ba_{h})^5},\Wcal_{h+1}(\mu_h),\frac{640n_{\reg}(\Bx_{h} \Ba_{h})^4}{\veps^2}}/\delta' } }{ n_{\reg} }.
\end{align*}

From \pref{lem:pdim} and \pref{lem:covering_pdim}, and noting that $n_{\reg} \ge \frac{2048 (\Bx_{h} \Ba_{h})^2}{\veps}$, we have that 
\begin{align*}
    &\log\rbr{36\Ncal_1\rbr{\frac{\veps^3}{40960(\Bx_{h} \Ba_{h})^5},\Wcal_{h+1}(\mu_h),\frac{640n_{\reg}(\Bx_{h} \Ba_{h})^4}{\veps^2}}/\delta'} 
    \\
    &\quad\le 4(\dspanner + 1)\log(8e) \log\rbr{\frac{655360 e^2 (\Bx_{h} \Ba_{h})^6}{\veps^3\delta'}} 
    \\
    &\quad\le 96\dspanner\log\rbr{\frac{n_{\reg}}{\delta'}}.
\end{align*}

Thus with probability at least $1-\delta'$, 
\begin{align*}
    &\abr{\EE[Y(w_{h+1}, w_h, \pi)]-\frac{1}{n_{\reg}}\sum_{i=1}^{n_{\reg}} Y_i(w_{h+1}, w_h, \pi)} 
    \\
    \le&~ \sqrt{\frac{ 49152 \dspanner (\Bx_{h} \Ba_{h})^2 \EE[Y(w_{h+1}, w_h, \pi)]\log\rbr{\frac{n_{\reg}}{\delta'}} }{ n_{\reg} }} + \frac{ 196608 \dspanner (\Bx_{h} \Ba_{h})^2 \log\rbr{\frac{n_{\reg}}{\delta'}} }{ n_{\reg} }.
\end{align*}

Then invoking the AM-GM inequality, 
\begin{align*}
    &\abr{\EE[Y(w_{h+1}, w_h, \pi)]-\frac{1}{n_{\reg}}\sum_{i=1}^{n_{\reg}} Y_i(w_{h+1}, w_h, \pi)} 
    \\
    \le&~ \frac{1}{2}\EE[Y(w_{h+1}, w_h, \pi)] + \frac{ 221184 \cdot \dspanner (\Bx_{h} \Ba_{h})^2\log\rbr{\frac{n_{\reg}}{\delta'}} }{ n_{\reg} }.
\end{align*}
Recall that this result holds for a fixed $\pi$ and $\Wcal_{h+1}(\mu_h)$ defined using a fixed $\mu_h$. Then setting $\delta' = \frac{\delta}{|\Pi'||\Upsilon_h|}$ and taking a union bound over $\Pi$ and $\Upsilon_h$, we have that with probability at least $1-\delta$ that for any $\pi \in \Pi'$ and $w_{h+1} \in \Wcal_{h+1}(\Upsilon_h)$ that
\begin{align*}
    &\abr{\EE[Y(w_{h+1}, w_h, \pi)]-\frac{1}{n_{\reg}}\sum_{i=1}^{n_{\reg}} Y_i(w_{h+1}, w_h, \pi)} 
    \\
    \le&~ \frac{1}{2}\EE[Y(w_{h+1}, w_h, \pi)] + \frac{ 221184 \dspanner (\Bx_{h} \Ba_{h})^2\log\rbr{\frac{n_{\reg}|\Pi'||\Upsilon_h|}{\delta}} }{ n_{\reg} }.
\end{align*}

Finally, since this result holds for any fixed $\Dcal_{0:h-1}\bigcup\Dcal_h^\mle$, by the law of total expectation, it also holds with probability at least $1-\delta'$ without conditioning on $\Dcal_{0:h-1}\bigcup\Dcal_h^\mle$. Using \pref{lem:variance_decomposition} with the definitions of $Y(w_{h+1}, w_h, \pi)$ and $Y_i(w_{h+1}, w_h, \pi)$  completes the proof.  
\end{proof}

\subsection{Barycentric spanner}
\label{app:bary}

In this section we first define the barycentric spanner \citep[Definition 2.1]{awerbuch2008online}, then prove that a spanner of size $\dspanner$ always exists for a set of functions linear in a feature $\mu_{h-1}$, from which \pref{prop:bary} follows straightforwardly. The proof is adapted from \citet[Proposition 2.2]{awerbuch2008online}, which only applies to square matrices, and we extend it to rectangular matrices for completeness. We close with a discussion of the computational complexity of finding the barycentric spanner. 

\begin{definition}[Barycentric spanner]\label{def:bary}
    Let $V$ be a vector space over the real numbers, and $S \subseteq V$ a subset whose linear span is a $m$-dimensional subspace of $V$. A set $X = \{x_1,\ldots,x_m\} \subseteq S$ is a \textit{barycentric spanner} of $S$ if every $x \in S$ may be expressed as a linear combination of elements of $X$ using coefficients in $[-1, +1]$. 
\end{definition}

\begin{lemma}[Barycentric spanner for linear functions]
\label{lem:barycentric}
For a feature $\mu_{h-1} \in \Upsilon_{h-1}$ with rank $\dlr$, any set of linear functions
$\Ucal \subseteq \{\langle \mu_{h-1}, \theta_h \rangle : \theta_h \in \RR^{\dlr} \}
$
has a barycentric spanner of cardinality $\min(|\Ucal|, \dlr)$. 
\end{lemma}

\begin{proof}
    We prove the proposition when $\rank(\mu_{h-1}) = \dspanner$ is full rank (the argument should be the same when $\rank(\mu_{h-1}) < d$), and $|\Ucal| > \dspanner$ (otherwise we can satisfy the lemma statement by picking all of $\Ucal$ to be the spanner). First, $\Ucal$ is a compact subset of $\RR^{|\Xcal|}$ because it is closed and bounded.  
    Because $\Ucal$ is linear in $\mu_{h-1}$, its linear span is a $\dspanner$-dimensional subspace of $\RR^{|\Xcal|}$, and any $u \in \Ucal$ can be written as the linear combination of a subspace basis. 
    
    We claim the barycentric spanner is any subset $B = \{b_1, \ldots, b_\dspanner\} \subseteq \Ucal$ with $B \in \RR^{\dspanner \times |\Xcal|}$ that maximizes the volume $|\det(BB^\top)|$. By compactness, the maximum is obtained by at least one subset of $\Ucal$. Since $\det(BB^\top) = (\prod_{i=1}^\dspanner \sigma_i(B))^2$, the maximizing $B$ will have $\dspanner$ singular values and full row rank (otherwise the determinant will be 0). As a result, any $u \in \Ucal$ will be a linear combination of the rows of $B$, i.e., there exists $\{c_i\}_{i=1}^\dspanner$ such that $u = \sum_{i=1}^\dspanner c_i b_i$. We will prove that $|c_i| \le 1$ by contradiction. 
    
    W.l.o.g, suppose there exists $u$ with coefficient $|c_1| > 1$. Then consider a new matrix $\wt{B} = \{u, b_2, \ldots, b_\dspanner\}$, which can be expressed as $\wt{B} = C B$, where $C \in \RR^{\dspanner \times \dspanner}$ is the coefficient matrix. Then $\wt{B}$ has determinant 
    \begin{align*}
        |\det(\wt{B} \wt{B}^\top)| = |\det(C)|^2 |\det(BB^\top)| = |c_1|^2|\det(BB^\top)| \ge \det(BB^\top).
    \end{align*}
    Then we have a contradiction because $B$ was volume-maximizing, and $|c_i| \le 1$. 
\end{proof}

\paragraph{Computation of barycentric spanner} Lastly, we discuss computation of the barycentric spanner. In the main results of the paper we assume that we can perfectly compute the barycentric spanner in an efficient manner. When this is not the case, the algorithm in Figure 2 in \citet{awerbuch2008online} (with similar adaptations to handle rectangular matrices as in the proof of \pref{lem:barycentric}) can be used to compute a $C$-approximate barycentric spanner, where $C > 1$, with $O(\dspanner^2 \log_C \dspanner)$ calls to a linear optimization oracle \citep[Proposition 2.5]{awerbuch2008online}. A $C$-approximate barycentric spanner is defined similarly as \pref{def:bary}, except that the coefficients are in the range $[-C, +C]$. This will only change our main results by increasing them by a factor of $C$, and we may simply set $C = 2$ with minimal effects on our sample complexity guarantees.

\subsection{Properties of low-rank MDPs}

\begin{lemma}
\label{lem:opp_linear}
In the low-rank MDP (\pref{assum:lowrank}), for any $h\in[H]$, function  $d_{h-1}:\Xcal\rightarrow \RR$, and \pseudo-policy $\ol\pi$ (\pref{def:pseudo_policy}), we have 
\[
    (\opp^{\ol\pi}_{h} d_{h})(x_{h+1}) = \iint P_{h}(x_{h+1}|x_{h},a_{h})\ol\pi_h(a_{h}|x_{h})d_{h}(x_{h})(\dd x_{h})(\dd a_{h}) = \langle \mu_{h}^*(x_{h+1}),\theta_{h+1}\rangle
\]
for some $\theta_{h+1}\in\RR^{\dlr}$ with $\|\theta_{h+1}\|_\infty \le \|d_{h}\|_1$. 
\end{lemma}

\begin{proof} 
By the definition of low-rank MDPs (\pref{assum:lowrank}), we have 
\begin{align*}
    \opp^{\ol\pi}_{h}d_h =&~ \iint P_{h}(x_{h+1}|x_{h},a_{h})\ol\pi_h(a_{h}|x_{h})d_{h}(x_{h})(\dd x_{h})(\dd a_{h})
    \\
    =&~ \iint \langle \mu_{h}^*(x_{h+1}), \phi_{h}^*(x_{h},a_{h}) \rangle \ol\pi_h(a_{h}|x_{h})d_{h}(x_{h})(\dd x_{h})(\dd a_{h})
    \\
    =&~\langle \mu_{h}^*(x_{h+1}), \theta_{h+1} \rangle,
\end{align*}
where $\theta_{h+1}=\iint \phi_{h}^*(x_{h},a_{h})\ol\pi_h(a_{h}|x_{h})d_{h}(x_{h})(\dd x_{h})(\dd a_{h}) \in \RR^{\dlr}$. In addition, 
\begin{align*}
    \|\theta_{h+1}\|_\infty \le&~ \iint \|\phi_{h}^*(x_{h},a_{h})\|_\infty \ol\pi_h(a_{h}|x_{h})|d_{h}(x_{h})|(\dd x_{h})(\dd a_{h})
    \\
    \le&~  \int \rbr{\int\ol\pi_h(a_{h}|x_{h}) (\dd a_{h})} |d_{h}(x_{h})|(\dd x_{h}) 
    \\
    \le&~ \int |d_{h}(x_{h})|(\dd x_{h}) = \|d_{h}\|_1 
\end{align*}
where we use \pref{lem:pseudopolicy_norm} in the last inequality. 
\end{proof}

\begin{lemma}\label{lem:mle_realizability}
     In low-rank MDPs (\pref{assum:lowrank}), given a dataset $\Dcal_{h}$ satisfying \pref{assum:data} for $h \in [H]$, let $d^D_h$ and $\dnext_h$ be the corresponding current-state and next-state data distributions. Then for the function class 
    \[
        \Fcal_h = \cbr{d_h = \langle \mutrue_{h-1}, \theta_h \rangle : d_h \in \Delta(\Xcal), \theta_h \in \RR^{\dlr},\|\theta_h\|_\infty \le 1 }, 
    \]
    we have that $d_h^D \in \Fcal_h$ and $\dnext_h \in \Fcal_{h+1}$. 
\end{lemma}
\begin{proof}
    Recall that under \pref{assum:data}, $\Dcal_h$ is collected by $\rho^{h-1} \circ \pi_h^D$ where $a_{0:h-1} \sim \rho^{h-1}$, an $(h-1)$-step non-Markov policy, and $a_h \sim \pi_h^D$, a Markov policy. 
    
    First we prove the lemma statement for $\dnext_h$.   Since $d^D_h$ is a valid distribution and $\pi_h^D$ is a valid Markov policy, from \pref{lem:opp_linear} we know that $\dnext_h = \opp^{\pi^D_h}_h(d^D_h)$ can be written as $\langle \mutrue_{h}, \theta_{h+1}\rangle$ with $\|\theta_{h+1}\|_\infty \le 1$. Finally, since $\dnext_h$ is a valid marginal distribution, $\dnext_h \in \Delta(\Xcal)$, thus satisfying all constraints of $\Fcal_{h+1}$. 

    To prove the lemma statement for $d^D_h$, we first prove a variant of \pref{lem:opp_linear} for non-Markov policies. With some overload of notation, let $d^D_{h-1}(x_{h-1})$ denote the marginal distribution of $x_{h-1}$ induced by rolling the non-Markov policy $\rho^{h-1}$ to level $h-1$. Then 
    \[
    d^D_h(x_h) = \iint P_{h}(x_h|x_{h-1},a_{h-1})\rho^{h-1}(a_{h-1}|x_{0:h-1})d^D_{h-1}(x_{h-1})(\dd x_{h-1})(\dd a_{h-1}). 
    \]
    Using similar steps as the proof of \pref{lem:opp_linear}, we have that 
    \begin{align*}
        d^D_h(x_h) =&~ \iint P_{h}(x_h|x_{h-1},a_{h-1})\rho^{h-1}(a_{h-1}|x_{0:h-1})d^D_{h-1}(x_{h-1})(\dd x_{h-1})(\dd a_{h-1})
        \\
        =&~\iint \langle \phi_{h-1}^*(x_{h-1},a_{h-1}),\mu_{h-1}^*(x_h)\rangle\rho^{h-1}(a_{h-1}|x_{0:h-1})d^D_{h-1}(x_{h-1})(\dd x_{h-1})(\dd a_{h-1})
        \\
        =&~\langle \mu_{h-1}^*(x_h), \theta_{h} \rangle,
    \end{align*}
    where $\theta_{h}=\iint \phi_{h-1}^*(x_{h-1},a_{h-1})\rho^{h-1}(a_{h-1}|x_{0:h-1})d^D_{h-1}(x_{h-1})(\dd x_{h-1})(\dd a_{h-1}) \in \RR^{\dlr}$. Since $d^D_{h-1}$ and $\rho^{h-1}(\cdot|x_{0:h-1})$ are valid probability distributions over states $x_h$ and actions $a_h$, respectively, it is easy to see that 
    \begin{align*}
        \|\theta_{h}\|_\infty \le&~ \iint \|\phi_{h-1}^*(x_{h-1},a_{h-1})\|_\infty \rho^{h-1}(a_{h-1}|x_{0:h-1})d^D_{h-1}(x_{h-1})(\dd x_{h-1})(\dd a_{h-1}) \le 1
    \end{align*}
    since $\|\phi_{h-1}^*(\cdot)\|_\infty \le 1$ from \pref{assum:lowrank}. Finally, since $d^D_h$ is a valid distribution, we have $d^D_h \in \Fcal_h$. 
\end{proof}

\begin{lemma}\label{lem:mle_realizability_unknown}
     In low-rank MDPs (\pref{assum:lowrank}), given a dataset $\Dcal_{h}$ satisfying \pref{assum:data} for $h \in [H]$, let $d^D_h$ and $\dnext_h$ be the corresponding current-state and next-state data distributions. Then for the function class 
    \[
        \Fcal_h(\Upsilon_{h-1}) = \cbr{d_h = \langle \mu_{h-1}, \theta_h \rangle : d_h \in \Delta(\Xcal), \mu_{h-1} \in \Upsilon_{h-1}, \theta_h \in \RR^{\dlr},\|\theta_h\|_\infty \le 1 }, 
    \]
    we have that $d_h^D \in \Fcal_h(\Upsilon_{h-1})$ and $\dnext_h \in \Fcal_{h+1}(\Upsilon_{h})$. 
\end{lemma}
\begin{proof}
    From \pref{lem:mle_realizability} we know that $d_h^D \in \Fcal_h$ (where $\Fcal_h$ is linear in the true features $\mutrue_{h-1}$, as defined in the \pref{lem:mle_realizability}), and $\dnext_h \in \Fcal_{h+1}$. Noting that $\Fcal_h \subseteq \Fcal_h(\Upsilon_{h-1})$ and $\Fcal_{h+1} \subseteq \Fcal_{h+1}(\Upsilon_{h})$ completes the proof. 
\end{proof}

\begin{lemma}\label{lem:clipped_concentrability}
    For $h \in [H]$, suppose we have a dataset $\Dcal_h$ satisfying \pref{assum:data}, with corresponding data distributions $d^D_h$ and $\dnext_h$. Given a function $w_h : \Xcal \rightarrow [-\Bx_h, \Bx_h]$ and \pseudo-policy $\ol\pi$ (\pref{def:pseudo_policy}) with $\frac{\ol\pi_h(a|x)}{\pi^D_h(a|x)} \le \Ba_h,\forall x \in \Xcal, a \in \Acal$, we have 
    \[
        \nbr{\frac{\opp_h^{\ol\pi}(d^D_h w_h)}{\dnext_h}}_{\infty} \le \Bx_h \Ba_h.  
    \]
\end{lemma}
\begin{proof}
    For any $x_{h+1} \in \Xcal$, we have
\begin{align*}
    \rbr{\opp^{\ol\pi}_{h}\rbr{d^D_{h} w_h}}(x_{h+1}) \le&~ \Bx_{h}  \rbr{\opp^{\ol\pi}_{h} d^D_{h}}(x_{h+1})
    \\
    =&~ \Bx_{h} \iint P_{h}(x_{h+1}|x_{h},a_{h}) \ol\pi_h(a_{h}|x_{h})d_{h}^D(x_{h}) (\dd x_{h}) (\dd a_{h}) 
    \\
    \le&~ \Bx_{h} \Ba_{h} \iint P_{h}(x_{h+1}|x_{h},a_{h}) \pi^D_h(a_{h}|x_{h})d_{h}^D(x_{h}) (\dd x_{h}) (\dd a_{h})
    \\
    =&~ \Bx_{h}\Ba_{h} \dnext_{h}(x_{h+1}).
\end{align*}
The last equality follows from the Bellman flow equation and \pref{assum:data}. The convention that $\frac{0}{0} = 0$ gives the lemma statement. 
\end{proof}

\begin{lemma}
\label{lem:opexp_ineq}
For any two state distributions $d_h,d_h'$ and a \pseudo-policy $\pi$ (\pref{def:pseudo_policy}), we have the following inequality
\[
\| \oppro_h d_{h} - \oppro_h d'_{h}\|_{1} \le \|d_h - d'_{h}\|_{1},
\]
where we recall that $(\oppro_{h} d_{h})(x_{h+1}) = \iint P_{h}(x_{h+1}|x_{h}, a_{h})\pi(a_{h}|x_{h}) d_{h}(x_{h}) (\dd x_{h}) (\dd a_{h}).$
\end{lemma}
\begin{proof}
    From definition of $\oppro_h$ and \pref{lem:pseudopolicy_norm}, we have
    \begin{align*}
    &~\|\oppro_h d_{h} - \oppro_h d'_{h}\|_{1} 
    =\iint \abr{P_{h}(x_{h+1}|x_{h}, a_{h})\pi(a_{h}|x_{h}) \rbr{d_{h}(x_{h}) -d'_{h}(x_{h})} (\dd x_{h}) (\dd a_{h})} (\dd x_{h+1}).
    \\ 
    \leq&~ \int \rbr{|d_{h}(x_h) - d_{h}'(x_h)| \rbr{\iint \pi(a_h|x_h)P_h(x_{h+1}|x_{h},a_{h})(\dd x_{h+1})(\dd a_{h})}} (\dd x_{h})\\ 
    \le&~ \int |d_{h}(x_h) - d_{h}'(x_h)| (\dd x_{h})
    = \|d_{h} - d_{h}'\|_{1}. \qedhere
\end{align*}
\end{proof}

\begin{lemma}\label{lem:pseudopolicy_norm}
    For any pseudo-policy $\ol\pi$ (\pref{def:pseudo_policy}), we have 
    $$
    \int \ol\pi_h(a_h|x_h) (\dd a_h) \le 1\quad \forall x_h \in \Xcal, h \in [H]. 
    $$ 
\end{lemma}
\begin{proof}
    Recall $\ol\pi_h(a_h|x_h) = \min\cbr{\pi_h(a_h|x_h), \Ba_h \pi_h^D(a_h|x_h) }$ where $\pi_h$ is a valid Markov policy. Then 
    \begin{align*}
        \int \ol\pi_h(a_h|x_h)(\dd a_h) = \int \min\cbr{ \pi_h(a_h|x_h), \Ba_h \pi_h^D(a_h|x_h) }(\dd a_h) \le \int \pi_h(a_h|x_h)(\dd a_h) = 1. 
    \end{align*}
\end{proof}

\subsection{Covering lemmas}
In this subsection, we provide the $\ell_1$ optimistic cover lemma used in MLE (\pref{lem:opt_cover}) and pseudo-dimension bound for the weight function class (\pref{lem:pdim}) respectively.
\begin{lemma}
\label{lem:opt_cover}
Suppose \pref{assum:realizability} holds. 
Then for the function class
\[
    \Fcal_h(\Upsilon_{h-1})= \{ d_h = \langle \mu_{h-1}, \theta_h \rangle : \mu_{h-1} \in \Upsilon_{h-1}, \theta_h \in \RR^{\dlr}, \|\theta_h\|_\infty \le 1, d_h \in \Delta(\Xcal)\}, 
\]
there exists an $\ell_1$ optimistic cover $\ol \Fcal_h(\Upsilon_{h-1})$ (according to \pref{def:opt_cover}) with scale $\gamma$ of size $|\Upsilon_{h-1}|\rbr{2\lceil \munorm / \gamma \rceil}^{\dlr}$ and $\ol\Fcal_h(\Upsilon_{h-1}) \subseteq (\Xcal \rightarrow \RR_{\ge 0})$.  
\end{lemma}
\begin{proof}
    The ideas of this proof are adapted from the proof of Proposition H.15 in \cite{chen2022unified}. 
    Let $\Theta_h = \{ \theta_h: \exists \mu_{h-1}\in \Upsilon_{h-1}, \text{ s.t., } \langle \mu_{h-1}, \theta_h \rangle \in \Fcal_h(\Upsilon_{h-1})\}\subseteq \{\theta_h : \theta_h\in\RR^{\dlr}, \|\theta_h\|_\infty \le 1\}$ be the set of $\theta_h$ parameters associated with $\Fcal_h(\Upsilon_{h-1})$. 
    Then any $d_h \in \Fcal_h(\Upsilon_{h-1})$ can be written as $\langle \mu_{h-1}, \theta_h\rangle$ for some $\mu_{h-1} \in \Upsilon_h$ and $\theta_h \in \Theta_h$. Define the $\gamma'$-neighborhood of $\theta_h$ to be $\Bcal(\theta_h, \gamma') \defeq \gamma' \lfloor \theta_h / \gamma' \rfloor + [0, \gamma']^{\dlr}$, and construct the optimistic covering function for each $d_h = \langle \mu_{h-1}, \theta_h\rangle$ as 
    \[
        f_{\mu_{h-1}, \theta_h}(x) = \max_{\ol\theta \in \Bcal(\theta_h, \gamma')} \langle \mu_{h-1}(x), \ol\theta \rangle \quad \forall x \in \Xcal. 
    \]
    Note that $f_{\mu_{h-1}, \theta_h} \ge d_h$ pointwise, thus $f_{\mu_{h-1}, \theta_h} \ge 0$, though it is not necessarily a valid distribution. 
    Further, 
    \begin{align*}
        \|f_{\mu_{h-1}, \theta_h} - d_h\|_1 \le&~ \int \max_{\ol\theta \in \Bcal(\theta_h, \gamma')} |\langle \ol\theta - \theta_h, \mu_{h-1}(x) \rangle| (\dd x)
        \\ 
        \le&~  \int \max_{\ol\theta \in \Bcal(\theta_h, \gamma')} \|\ol\theta - \theta_h\|_\infty \|\mu_{h-1}(x)\|_1 (\dd x)
        \\
        \le&~ \gamma' \int \|\mu_{h-1}(x)\|_1 (\dd x) 
        \\
        \le&~ \gamma' \munorm 
    \end{align*}
    using \pref{assum:realizability} in the last line. 
    Observe that there are at most $\rbr{2\lceil 1 / \gamma' \rceil}^{\dlr}$ unique $\gamma'$-neighborhoods in the set $\{\Bcal(\theta_h, \gamma')\}_{\theta_h \in \Theta_h}$. This implies that there are at most $|\Upsilon_{h-1}|\rbr{2\lceil 1 / \gamma' \rceil}^{\dlr}$ unique functions in the set $\{f_{\mu_{h-1}, \theta_h}\}_{\langle\mu_{h-1}, \theta_h\rangle \in \Fcal_h(\Upsilon_{h-1})} $, which forms an $\ell_1$-optimistic cover of $\Fcal_h(\Upsilon_{h-1})$ of scale $\gamma'$.  
    Finally, setting $\gamma' = \gamma / \munorm$ gives us an $\ell_1$-optimistic covering of $\Fcal_h(\Upsilon_{h-1})$ of scale $\gamma$ with size $|\Upsilon_{h-1}|\rbr{2\lceil \munorm / \gamma \rceil}^{\dlr}$.  
\end{proof}

\begin{lemma}
\label{lem:pdim}
    For any $h \in [H]$ and density feature $\mu_{h-1} \in \Upsilon_{h-1}$, the function class 
    \[
        \Wclip_{h}(\mu_{h-1}) = \cbr{w_{h} = \frac{\langle \mu_{h-1}, \thetaup_{h}\rangle}{\langle \mu_{h-1}, \thetadown_{h}\rangle} :\nbr{w_{h}}_\infty \le  \Bx_{h-1}\Ba_{h-1},\thetaup_{h},\thetadown_{h} \in \RR^{\dlr}}.
    \]
    has pseudo-dimension (\pref{def:pdim}) bounded as $\Pdim(\Wcal_h(\mu_{h-1})) \le 4(\dlr + 1) \log(8e)$. 
\end{lemma} 

\begin{proof}
    For any $h$ and $\mu_h$, consider the unconstrained version $\Wcal_h'(\mu_{h-1})$ of $\Wcal_h(\mu_{h-1})$: 
    \[
        \Wcal'_{h}(\mu_{h-1}) = \cbr{w = \frac{ \langle \mu_{h-1}, \thetaup_{h}\rangle }{\langle \mu_{h-1}, \thetadown_{h} \rangle} : \thetaup_{h}, \thetadown_{h} \in \RR^{\dlr}}. 
    \]
    Clearly, $\Wcal_h(\mu_{h-1}) \subseteq \Wcal_h'(\mu_{h-1})$, thus $\Pdim(\Wcal_h(\mu_{h-1})) \le \Pdim(\Wcal_h'(\mu_{h-1}))$, and $\Pdim(\Wcal_h'(\mu_{h-1})) = \VCdim(\Hcal_{\Wcal_h'(\mu_{h-1})})$, where $\Hcal_{\Wcal_h'(\mu_{h-1})} = \{h =  \sgn(w - c) : w \in \Wcal_h'(\mu_{h-1}), c \in \RR\}$. We will use \pref{lem:pseudodim} to bound $\VCdim(\Hcal_{\Wcal'_h(\mu_{h-1})})$. 
    Any $h(x) \in \Hcal_{\Wcal_h'(\mu_{h-1})}$ may be written as the following Boolean formula 
    \begin{align*}
        \Phi =&~ \one \sbr{ \frac{ \langle \mu_{h-1}(x), \thetaup_h \rangle }{\langle \mu_{h-1}(x), \thetadown_h \rangle} - c \ge 0 } 
        \\
        =&~ \rbr{\one\sbr{ \sum_{i=1}^{\dlr} \mu_{h-1}(x)[i] \thetaup_h[i] - c \sum_{i=1}^{\dlr} \mu_{h-1}(x)[i] \thetadown_h[i] \ge 0  } \one\land \sbr{\sum_{i=1}^{\dlr} \mu_{h-1}(x)[i] \thetadown_h[i] \ge 0 }} 
        \\
        &~\lor \rbr{\one\sbr{ \sum_{i=1}^{\dlr} \mu_{h-1}(x)[i] \thetaup_h[i] - c \sum_{i=1}^{\dlr} \mu_{h-1}(x)[i] \thetadown_h[i] \le 0  } \land \one\sbr{\sum_{i=1}^{\dlr} \mu_{h-1}(x)[i] \thetadown_h[i] < 0 }} 
    \end{align*}
    which involves $k = 2\dlr + 1$ real variables, a polynomial degree of at most $l = 1$ in these variables, and $s = 4$ atomic predicates. Then from \pref{lem:pseudodim}, $\Pdim(\Wcal_h(\mu_{h-1}))) \le \VCdim(\Hcal_{\Wcal_h'(\mu_{h-1})}) \le 4(\dlr + 1) \log(8e)$. 
    \end{proof}

\begin{lemma}[Theorem 2.2 of \citet{goldberg1993bounding}]\label{lem:pseudodim}
    Let $\Ccal_{k,m}$ be a concept class where concepts and instances are represented by $k$ and $m$ real values, respectively. Suppose that the membership test for any instance $c$ in any concept $C$ of $\Ccal_{k,m}$ can be expressed as a Boolean formula $\Phi_{k,m}$ containing $s$ distinct atomic predicates, each predicate being a polynomial inequality over $k+m$ variables of degree at most $l$. Then the VC dimension of $\Ccal_{k,m}$ is bounded as $\VCdim(\Ccal_{k,m}) \le 2k \log(8 els)$. 
\end{lemma}

\subsection{Probabilistic tools}

In this section, we define standard tools from statistical learning theory \citep{anthony2009neural,vapnik1998statistical} that we use in our proofs. 
We note that, for convenience, we may override some notations from the main paper, e.g., $\veps$ does not refer to the same thing as in other sections. 

\begin{definition}[VC-dimension]\label{def:vcdim}
    Let $\Fcal \subseteq \{-1, +1\}^\Xcal$ and $x_1^m = (x_1, \ldots, x_m) \in \Xcal^m$. We say $x_1^m$ is shattered by $\Fcal$ if $\forall \mathbf{b} \in \{-1, +1\}^m$, $\exists f_{\mathbf{b}} \in \Fcal$ such that $(f_{\mathbf{b}}(x_1),\ldots,f_{\mathbf{b}}(x_m)) = (b_1, \ldots,b_m) \in \RR^m$. The Vapnik-Chervonenkis (VC) dimension of $\Fcal$ is the cardinality of the largest set of points in $\Xcal$ that can be shattered by $\Fcal$, that is, $\dim(\Fcal) = \max\{ m \in \NN \mid \exists x_1^m \in \Xcal^m, \text{ s.t. $x_1^m$ is shattered by $\Fcal$}\}$. 
\end{definition}

\begin{definition}[Pseudo-dimension]\label{def:pdim}
    Let $\Fcal \subseteq \RR^\Xcal$ and $x_1^m = (x_1, \ldots, x_m) \in \Xcal^m$. We say $x_1^m$ is pseudo-shattered by $\Fcal$ if $\exists \mathbf{c} = (c_1, \ldots,c_m) \in \RR^m$ such that $\forall \mathbf{y} = (y_1, \ldots, y_m) \in \{-1, +1\}^m$, $\exists f_{\mathbf{y}} \in \Fcal$ such that $\sgn(f_{\mathbf{y}}(x_i - c_i) = y_i ~\forall i \in [m]$. The pseudo-dimension of $\Fcal$ is the cardinality of the largest set of points in $\Xcal$ that can be pseudo-shattered by $\Fcal$, that is, $\Pdim(\Fcal) = \max\{ m \in \NN \mid \exists x_1^m \in \Xcal^m, \text{ s.t. $x_1^m$ is pseudo-shattered by $\Fcal$}\}$. 
\end{definition}

\begin{definition}[Uniform covering number]\label{def:covering_number}
    For $p = 1,2,\infty$, the uniform covering number of $\Hcal$ w.r.t. the norm $\|\cdot\|_p$ is define as 
    \[
        \Ncal_p(\veps, \Hcal, m) = \max_{x_1^m \in \Xcal^m} \Ncal_p(\veps, \Hcal, x_1^m) 
    \]
    where $\Ncal_p(\veps, \Hcal, x_1^m)$ is the $\veps$-covering number of $\Hcal|_{x_1^m}$ w.r.t. $\|\cdot\|_p$, that is, the cardinality of the smallest set $S$ such that for every $h \in \Hcal|_{x_1^m}$, $\exists s \in S$ such that $\|h - s\|_p < \veps$. 
\end{definition}

\begin{lemma}[Bounding uniform covering number by pseudo-dimension, Corollary 42 of \cite{modi2021model}]\label{lem:covering_pdim} 
    Given a hypothesis class $\Hcal \subseteq (\Zcal \rightarrow [a,b])$, for any $m \in \NN$ we have 
    \[
        \Ncal_1(\veps, \Hcal, m) \le \rbr{\frac{4e^2(b-a)}{\veps}}^{\Pdim(\Hcal)}. 
    \]
\end{lemma} 

\begin{lemma}[Uniform deviation bound using covering number, adapted from Corollary 39 of \citet{modi2021model}]
\label{lem:uni_bern_conf_covering}
For $b\ge 1$, let $\Hcal\subseteq (\Zcal\rightarrow[-b,b])$ be a hypothesis class and $Z^n=(z_1,\ldots,z_n)$ be i.i.d. samples drawn from some distribution $\PP(z)$ supported on $\Zcal$. Then 
\begin{align*}
    \PP\rbr{\abr{\EE[h(z)]-\frac{1}{n}\sum_{i=1}^n h(z_i)} \ge \veps} \le&~ 36\Ncal_1\rbr{\frac{\veps^3}{640b^2},\Hcal,\frac{40nb^2}{\veps^2}}\exp\rbr{-\frac{n \veps^2}{128\VV[h(z)] + 512\veps b}} .
\end{align*}
\end{lemma}